\documentclass{article}

 \usepackage[preprint]{neurips_2026}

\usepackage[utf8]{inputenc} 
\usepackage[T1]{fontenc}    
\usepackage{hyperref}       
\usepackage{url}            
\usepackage{booktabs}       
\usepackage{amsfonts}       
\usepackage{nicefrac}       
\usepackage{microtype}      
\usepackage{xcolor}         

\setcitestyle{round}
\usepackage{amsmath}
\usepackage{amssymb}
\usepackage{mathtools}
\usepackage{amsthm}

\usepackage{algorithm}
\usepackage{algpseudocode}

\usepackage{graphicx, multirow, makecell, colortbl}

\theoremstyle{plain}
\newtheorem{theorem}{Theorem}[section]

\newtheorem{lemma}[theorem]{Lemma}
\newtheorem{corollary}[theorem]{Corollary}
\theoremstyle{definition}

\newtheorem{assumption}[theorem]{Assumption}
\theoremstyle{remark}
\newtheorem{remark}[theorem]{Remark}

\title{\textsc{FedAdaVR}: Adaptive Variance Reduction for Robust Federated Learning under Limited Client Participation}

\author{%
  S M Ruhul Kabir Howlader \\
  School of Computing and Mathematical Sciences\\ University of Leicester \\
  Leicester LE1 7RH, United Kingdom \\
  \texttt{smrkh1@leicester.ac.uk} \\
  \And
  Xiao Chen\thanks{Corresponding Author} \\  
  School of Artificial Intelligence\\ Nanchang University \\
  Nanchang 330031, China \\
  \texttt{xiao.chen@ncu.edu.cn} \\
  \And
  Yifei Xie \\
  School of Management\\ Zhejiang University \\
  Hangzhou 310027, China \\
  \texttt{xieyfzju@outlook.com} \\
  \And
  Lu Liu \\
  Department of Computer Science\\ University of Exeter \\
  Exeter EX4 4QF, United Kingdom \\
  \texttt{l.liu3@exeter.ac.uk} \\
}

\begin{document}

\maketitle

\begin{abstract}
  Federated learning (FL) encounters substantial challenges due to heterogeneity, leading to gradient noise, client drift, and partial client participation errors, the last of which is the most pervasive but remains insufficiently addressed in current literature. In this paper, we propose \textsc{FedAdaVR}, a novel FL algorithm aimed at solving heterogeneity issues caused by partial client participation by incorporating an adaptive optimiser with a variance reduction technique. This method takes advantage of the most recent stored updates from clients, even when they are absent from the current training round, thereby emulating their presence. Furthermore, we propose \textsc{FedAdaVR-Quant}, a server-side memory compression variant that stores historical client updates in quantised form (while keeping client-server communication in full precision), reducing the server memory footprint of \textsc{FedAdaVR} by 50\% (FP16), 75\% (Int8), and 87.5\% (Int4) while maintaining highly competitive, and occasionally superior model performance. We analyse the non-convex convergence behaviour of \textsc{FedAdaVR} with Adagrad and prove that our proposed framework eliminates the first-order partial client participation bias while improving monotonically with cohort size. Extensive experiments conducted on multiple vision and NLP datasets demonstrate that \textsc{FedAdaVR} consistently outperforms state-of-the-art baseline methods.
\end{abstract}

\section{Introduction}
	
	The successful deployment of machine learning models hinges upon access to large-scale, high‑quality datasets; however, assembling such datasets poses considerable challenges. In healthcare contexts, for instance, the centralised collection of patient data and the training of machine learning models are constrained by regulations such as the Health Insurance Portability and Accountability Act (HIPAA) and the General Data Protection Regulation (GDPR) \citep{warnat2021swarm}. Federated Learning (FL) has emerged as an effective paradigm to address these constraints, enabling multiple clients to collaboratively train a shared model without exposing raw data beyond the owner’s device.
	
	The cross-device FL setting \citep{Kairouz2019} encompasses vast numbers of clients, only a subset of which are active in any given round, resulting in increased heterogeneity in both data distributions and device capabilities, thereby complicating FL training. In practical deployments like mobile predictive text keyboards and healthcare wearables, this active fraction is often exceedingly small, as devices typically participate only when simultaneously idle, charged, and on unmetered Wi-Fi. Moreover, such heterogeneous (non-IID) data give rise to client‑drift errors. Convergence analyses of FedAvg \citep{McMahan2016} have shown that data heterogeneity can slow convergence and introduce bias into the global model (i.e., client drift from non-IID data) \citep{Karimireddy2019,Li2019c,Khaled2019}. To counteract these effects, FedProx incorporates a proximal term into the local objective to penalise deviation from the global model \citep{Li2018}, whereas SCAFFOLD employs control variates to correct client drift error induced by heterogeneity \citep{Karimireddy2019}. Subsequent methods, including ADACOMM \citep{Wang2018}, FedBug \citep{Kao2023}, AdaBest \citep{Varno2022} and Def‑KT \citep{Li2020c}, have further refined techniques for mitigating drift.

	While client drift has been extensively studied, FL is also affected by the error introduced by partial client participation. This challenge arising from client unavailability due to device and communication variability \citep{Gu2021,Jhunjhunwala2022,Su2023,Wang2022,Mansour2023}. To mitigate this challenge, MIFA was proposed \citep{Gu2021}, which approximates full participation by retaining each client's most recent update on the server and reusing it in subsequent rounds when that client is unavailable. FedVARP \citep{Jhunjhunwala2022} advances this approach by addressing two key limitations of MIFA and demonstrating improved empirical performance. Moreover, it provides a decomposition of the total heterogeneity error into three components: stochastic gradient error, client drift error, and partial participation error. The partial participation error, although significant, has received comparatively limited attention. Importantly, partial client participation skews the global model update towards the data distributions of the selected subset of clients in each round, thereby exacerbating overall heterogeneity. To reduce the substantial storage overhead associated with maintaining per-client updates at scale, the ClusterFedVARP variant, introduced in \citep{Jhunjhunwala2022}, groups client updates into clusters. This clustering strategy mitigates memory demands without sacrificing convergence performance. However, it may also introduce potential challenges, including incorrect cluster assignments (misgrouped updates), imbalanced cluster sizes, and temporal changes in cluster membership (cluster drift).
	
	Despite the advantages of FedVARP, our preliminary experiments reveal that it converges slowly under severely label-quantity skew conditions (e.g., when each client’s dataset contains only one (LQ-1) or two (LQ-2) classes). This limitation arises primarily from the fixed server learning rate, resulting in a total absence of adaptivity. \emph{Why adaptivity matters in variance-reduced FL.} In FedVARP, the variance-reduced update is applied with a fixed server learning rate $\eta_s$. However, when client participation is sporadic and data distributions are severely skewed (e.g., LQ-1 partitioning, the most extreme heterogeneity setting), different model parameters receive updates at drastically different frequencies. Parameters associated with frequent classes experience dense, low-variance signals, whereas those tied to rare classes are updated rarely and with high variance. A fixed $\eta_s$ cannot accommodate this imbalance: any single choice either inflates noise on active coordinates or freezes learning on stale ones. This explains our experimental observation that FedVARP converges slowly under extreme label skew, and more importantly, it reveals a fundamental \emph{missing ingredient} in current variance-reduced FL methods, i.e. \emph{coordinate-wise step adaptation}.
	
    Simply stacking an existing variance-reduced method onto an off-the-shelf adaptive server optimiser (e.g., feeding FedVARP outputs directly into FedOpt \citep{Reddi2020}) is ineffective. As we show in Section~\ref{sec:ProposedAlgorithm}, a naive composition causes the adaptive moments to absorb the variance-reduction scaling factor, which distorts the numerical stabiliser $\epsilon$. Furthermore, standard implementations lack the early-round bias correction required under extreme non-IID participation. We address this by decoupling the SAGA-style variance-reduced correction $\mathbf{r}^{(t)}$ in~\eqref{VarianceUpdateEquation} from the adaptive moment accumulators and explicitly reintegrating exact bias correction; this unified design ensures that the adaptive preconditioner cleanly tracks the debiased pseudo-gradient $\mathbf{G}^{(t)}$.
	
	These insights, alongside advances in SAGA-style variance reduction (VR) \citep{Defazio2014}, partial variance reduction and heterogeneous client sampling \citep{Gu2021,Jhunjhunwala2022,Su2023,Wang2022}, and adaptive optimisation, motivate our proposed framework. Specifically, this paper makes the following contributions:

	\begin{itemize}
		\item We propose \textsc{FedAdaVR}, a federated learning algorithm that unifies server-side adaptive optimisation with SAGA-style variance reduction by decoupling historical client-state maintenance from adaptive moment preconditioning, addressing the coordinate-wise learning-rate imbalance induced by partial client participation. We establish a formal non-convex convergence guarantee for \textsc{FedAdaVR} with Adagrad (Theorem~\ref{theo:con} and Appendix~\ref{app-convergence-proof}).
        
		\item We further introduce \textsc{FedAdaVR-Quant}, a complementary server-side memory compression variant wherein the most recent client updates are retained in quantised form to reduce server-side memory usage (by 50\% (FP16), 75\% (Int8), and 87.5\% (Int4)). Empirical evaluations show that this quantisation preserves convergence and frequently acts as a mild regulariser while substantially reducing server memory overhead.

        \item Comprehensive experiments are carried out under various IID and non-IID data partitioning schemes on multiple datasets, simulating scenarios of extreme client unavailability against several state-of-the-art baselines.
	\end{itemize}

    While the theoretical analysis (Appendix~\ref{app-convergence-proof}) provides worst-case convergence bounds under standard assumptions, the empirical results confirm that the proposed combination substantially improves training efficiency and final accuracy, especially in challenging low-participation settings where existing variance-reduced methods stagnate.

    Unlike prior research, which has leveraged quantisation primarily to address communication heterogeneity \citep{Reisizadeh2019,Mao2022,Ren2023,Chen2024} and device constraints \citep{Gupta2022,Abdelmoniem2021,Chen2024}, \textsc{FedAdaVR-Quant} is, to the best of our knowledge, the first FL strategy to explicitly mitigate server-side memory bottlenecks via the quantisation of stored client states. Uniquely, our approach tackles the challenge of partial client participation by unifying server-side adaptive optimisation, SAGA-style variance reduction, and quantised state maintenance, with a formal non-convex convergence guarantee for full-precision \textsc{FedAdaVR} under server-side Adagrad (Section~\ref{ConvergenceAnalysis}). A comprehensive review of related literature, alongside a detailed theoretical and mechanistic comparison of \textsc{FedAdaVR} against key baselines, is provided in Appendix~\ref{RelatedWork}.

    \section{Preliminaries}
    Federated learning involves a network of $N$ distributed clients (e.g., mobile devices or edge sensors), each possessing a private dataset $\mathcal{D}_n \mbox{=} \{(x_i, y_i)\}$, with no exchange of raw data. The overarching objective is to learn a shared model $w \in \mathbb{R}^d$ by minimising the aggregate empirical loss:
        \begin{equation}\label{eq:1}
            \min_{w} \; F(w) = \sum_{n=1}^{N} p_n\,F_n(w), \quad \text{where} \quad F_n(w) = \frac{1}{|\mathcal{D}_n|} \sum_{(x_i,y_i)\in \mathcal{D}_n} \ell\bigl(w; x_i, y_i\bigr).
        \end{equation}
    where $p_n \mbox{=} \frac{|\mathcal{D}_n|}{\sum_{j=1}^N |\mathcal{D}_j|}$ represents the weight assigned to each client's contribution, and $\ell(\cdot)$ denotes the loss function. Owing to communication and privacy constraints, clients perform several local updates (typically stochastic gradient steps) on $F_n(w)$, and transmit only model updates (rather than raw data) to a central server. The server then aggregates these updates (e.g., through weighted averaging) to produce an improved global model. This federated framework aims to balance statistical heterogeneity (i.e., non-IID data across clients), limited communication bandwidth, and stringent data privacy requirements.
    
    Consider the FedAvg algorithm, in which the server randomly selects a subset $\mathcal{S}^{(t)}$ of clients and transmits the current global model $\mathbf{w}^{(t)}$ to them. Each selected client $i$ performs $K$ steps of local SGD initialised from the current global model $\mathbf{w}^{(t)}$, returning the gradient
        \begin{equation}
            \mathbf{g}_i^{(t)} = \frac{1}{\eta_c}\left(\mathbf{w}^{(t)} - \mathbf{w}_i^{(t,K)}\right),
        \end{equation}
    where $\eta_c$ denotes the client learning rate, to the server (see Appendix~\ref{Apd:DeviceUpdate} for the full procedure). The server aggregates these local updates to form the new global model according to:

        \begin{equation}\label{FedAvgUpdateEquation}
            \mathbf{w}^{(t+1)} = \mathbf{w}^{(t)} - \sum_{i \in \mathcal{S}^{(t)}} p_i \mathbf{g}_i^{(t)}.
        \end{equation}
    
    Variance arises in this process due to the random selection of clients, particularly when their local datasets are heterogeneous. Classical federated learning algorithms like FedAvg~\citep{McMahan2016} struggle to manage this variance in scenarios where only a small number of clients participate per round and the data distribution is highly skewed.

    \section{\textsc{FedAdaVR} and \textsc{FedAdaVR-Quant}} \label{sec:ProposedAlgorithm}
	This section introduces the proposed algorithms, \textsc{FedAdaVR} and \textsc{FedAdaVR-Quant}, both designed to address variance resulting from limited client participation and highly skewed data distributions across clients. The proposed method combines the variance reduction capability of SAGA with the adaptive characteristics of optimisers such as Adam, Adagrad, Adabelief, Yogi, and Lamb. Notably, the proposed \textsc{FedAdaVR}, as defined in Algorithm~\ref{alg:FedAdaVRAlgorithm}, operates using only the gradients exchanged between the server and the clients, imposing no additional computational or communication burden on client devices. 
	
	During each training round, a subset $S^{(t)}$ of clients is randomly selected under the \textsc{FedAdaVR} framework. The selected clients receive the global model $\mathbf{w}^{(t)}$ from the server and carry out local stochastic gradient descent (SGD). Following this, each participating client transmits its local update  $\mathbf{g}_i^{(t)}$ to the server. Up to this stage, the procedure mirrors that of FedAvg. However, whereas FedAvg applies direct averaging of client updates (see eq.~(\ref{FedAvgUpdateEquation})), it is susceptible to high variance under limited client participation when data distributions are heterogeneous. This source of variance constitutes the dominant error component. To mitigate this, the proposed method incorporates a variance reduction mechanism in conjunction with an adaptive optimiser. The server retains the most recent update received from each client, thereby simulating full client presence in each round. In the quantised variant, referred to as \textsc{FedAdaVR-Quant}, the stored updates are quantised to reduce memory usage, while all other components remain consistent with those of \textsc{FedAdaVR}.
	
	The server maintains the most recent updates received from clients. Initially, each client $j \in [N]$ is associated with an update denoted by $\mathbf{y}_j^{(0)}$. Following each training round, the updates are revised according to the following rule:

		\begin{equation}\label{MemoryUpdateEquation}
			\mathbf{y}_j^{(t+1)} \mbox{=} 
			\begin{cases}
				\mathbf{g}_j^{(t)} & \text{if } j \in S^{(t)} \\
				\mathbf{y}_j^{(t)} & \text{if } j \notin S^{(t)}
			\end{cases}, \ \text{for all } j \in [N].
		\end{equation}
	
	For the quantised variant, equation~\eqref{MemoryUpdateEquation} is modified as follows:
		\begin{equation}\label{MemoryUpdateEquationQuantised}
			\mathbf{y}_j^{(t+1)} \mbox{=} 
			\begin{cases}
				\textsc{Quant}(\mathbf{g}_j^{(t)}) & \text{if } j \in S^{(t)} \\
				\mathbf{y}_j^{(t)} & \text{if } j \notin S^{(t)}
			\end{cases}, \ \text{for all } j \in [N].
		\end{equation}
	
    Equations~\eqref{MemoryUpdateEquation} and \eqref{MemoryUpdateEquationQuantised}, corresponding to \textsc{FedAdaVR} and \textsc{FedAdaVR-Quant}, respectively, ensure that $\mathbf{y}_j^{(t)}$ consistently stores the most recent update for each client $j$, and require $\mathcal{O}(Nd)$ memory on the server.
	
	To compute the variance-reduced model update $\mathbf{r}^{(t)}$, the received updates $\mathbf{g}_i^{(t)}$ from the $M = |S^{(t)}|$ participating clients are jointly utilised with the stored updates $\mathbf{y}_j^{(t)}$ (under uniform client weighting $p_j = \frac{1}{N}$):
    	\begin{equation}\label{VarianceUpdateEquation}
    		\mathbf{r}^{(t)} = \frac{1}{M}\sum_{i \in S^{(t)}} \left( \mathbf{g}_i^{(t)} - \mathbf{y}_i^{(t)} \right) + \frac{1}{N}\sum_{j=1}^{N} \mathbf{y}_j^{(t)}.
    	\end{equation}
    
	For the quantised variant, both the active correction term and the historical average dequantise the stored states $\mathbf{y}_j^{(t)}$, modifying equation~\eqref{VarianceUpdateEquation} as follows:
        \begin{equation}\label{VarianceUpdateEquationQuant}
    		\mathbf{r}^{(t)} = \frac{1}{M}\sum_{i \in S^{(t)}} \left( \mathbf{g}_i^{(t)} - \textsc{Dequant}(\mathbf{y}_i^{(t)}) \right) + \frac{1}{N}\sum_{j=1}^{N} \textsc{Dequant}(\mathbf{y}_j^{(t)}).
    	\end{equation}
    
	Subsequently, $\mathbf{r}^{(t)}$ is employed to compute the pseudo-gradient for the adaptive optimiser using the client learning rate $\eta_c$:
	\begin{equation}\label{GradientUpdateEquation}
		\mathbf{G}^{(t)} = \mathbf{r}^{(t)} \, \eta_c.
	\end{equation}
	The weight decay parameter $\lambda$ is incorporated into the pseudo-gradient when $\lambda \neq 0$:
        \begin{equation}\label{GradientUpdateEquationLambda}
            \mathbf{G}^{(t)} = \mathbf{G}^{(t)} + \lambda \mathbf{w}^{(t)}.
        \end{equation}

    Finally, $\mathbf{G}^{(t)}$ is passed to the selected adaptive optimiser (one of Adagrad, Adam, Adabelief, Yogi, or Lamb) to produce the updated global model $\mathbf{w}^{(t+1)}$. The full update rules and pseudocode for each optimiser are provided in Appendix~\ref{Apd:AdpOpt}. For instance, under Adagrad the accumulator is updated as $z_t = z_{t-1} + \mathbf{G}^{(t)} \odot \mathbf{G}^{(t)}$, and the global model is updated as $\mathbf{w}^{(t+1)} = \mathbf{w}^{(t)} - \eta_s \frac{\mathbf{G}^{(t)}}{\sqrt{z_t} + \epsilon}$, where $\epsilon$ is a small constant.

	\begin{algorithm}[!t]
		\caption{\colorbox{gray!15}{\textsc{FedAdaVR}} and \colorbox{gray!40}{\textsc{FedAdaVR-Quant}}}
		\label{alg:FedAdaVRAlgorithm}
		\begin{algorithmic}[1]
				\State \textbf{Initialisation:} Set initial model parameters \(\mathbf{w}^{(0)}\), server learning rate \(\eta_s\), client learning rate \(\eta_c\), rounds \(T\), and initial states \(\mathbf{y}_j^{(0)} = 0\) for all \(j \in [N]\), with \(m_{-1} = 0\), \(v_{-1} = 0\), and \(z_{-1} = 0\)
				\For{\(t = 0, 1, \ldots, T-1\)}
				\State Send \(\mathbf{w}^{(t)}\) to all active devices \(i \in S^{(t)}\)
				\For{\(i \in S^{(t)}\)}
				\State \(\mathbf{g}_i^{(t)} \gets \textsc{DeviceUpdate}(i, \mathbf{w}^{(t)}, \eta_c)\)  \Comment{Compute local gradient}
				\EndFor
				\State \(\begin{aligned}[t]
					\mathbf{r}^{(t)} \gets & \begin{cases}
						\colorbox{gray!15}{\(\displaystyle\frac{1}{M}\sum_{i \in S^{(t)}} \Bigl(\mathbf{g}_i^{(t)} - \mathbf{y}_i^{(t)}\Bigr) + \frac{1}{N}\sum_{j=1}^{N} \mathbf{y}_j^{(t)}\)} \\[10pt]
						\colorbox{gray!40}{\(\displaystyle\frac{1}{M}\sum_{i \in S^{(t)}} \Bigl(\mathbf{g}_i^{(t)} - \textsc{Dequant}(\mathbf{y}_i^{(t)})\Bigr) + \frac{1}{N}\sum_{j=1}^{N} \textsc{Dequant}(\mathbf{y}_j^{(t)})\)} \\[4pt]
					\end{cases}
				\end{aligned}\) \Comment{VR update}
				\State \(\mathbf{G}^{(t)} \gets \mathbf{r}^{(t)}\, \eta_c\) \Comment{Form server pseudo-gradient}
   %
				\State \(  \mathbf{w}^{(t+1)} \gets
				\begin{cases}
					\textcolor{violet!70}{\textsc{AdagradOptimiser}(\mathbf{w}^{(t)}, \mathbf{G}^{(t)}, \eta_s, z_t)} \\
					\textcolor{gray}{\textsc{AdamOptimiser}(\mathbf{w}^{(t)}, \mathbf{G}^{(t)}, \eta_s, m_t, v_t)} \\
					\textcolor{purple}{\textsc{AdabeliefOptimiser}(\mathbf{w}^{(t)}, \mathbf{G}^{(t)}, \eta_s, m_t, z_t)} \\
					\textcolor{brown}{\textsc{YogiOptimiser}(\mathbf{w}^{(t)}, \mathbf{G}^{(t)}, \eta_s, m_t, v_t)} \\
					\textcolor{red!30}{\textsc{LambOptimiser}(\mathbf{w}^{(t)}, \mathbf{G}^{(t)}, \eta_s, m_t, v_t)} \\
				\end{cases} \) \Comment{\text{Select one optimiser}}
				\For{\(j \in [N]\)} 
				
				\State \colorbox{gray!15}{
					\(
					\mathbf{y}_j^{(t+1)} \gets 
					\begin{cases}
						\mathbf{g}_j^{(t)} & \text{if } j \in S^{(t)} \\
						\mathbf{y}_j^{(t)} & \text{if } j \notin S^{(t)}
					\end{cases}
					\)} \Comment{Update stored client states}
				\State \colorbox{gray!40}{
					\(
					\mathbf{y}_j^{(t+1)} \gets 
					\begin{cases}
						\textsc{Quant}(\mathbf{g}_j^{(t)}) & \text{if } j \in S^{(t)} \\
						\mathbf{y}_j^{(t)} & \text{if } j \notin S^{(t)}
					\end{cases}
					\)} \Comment{Update quantised stored client states}
				\EndFor
				\EndFor
		\end{algorithmic}
	\end{algorithm}
	
    To summarise Algorithm~\ref{alg:FedAdaVRAlgorithm}, the server initially transmits the current global model $\mathbf{w}^{(t)}$ to the selected clients $S^{(t)}$, who subsequently perform parallel local updates using their private datasets $\mathcal{D}_n$. Each participating client then returns an update $\mathbf{g}_i^{(t)}$ to the server. Upon receiving the updates, the server first computes the variance-reduced update $\mathbf{r}^{(t)}$, which is then employed to determine the server-side pseudo-gradient $\mathbf{G}^{(t)}$. This $\mathbf{G}^{(t)}$ is then processed by an adaptive optimiser together with its associated parameters. The optimiser updates the global model to $\mathbf{w}^{(t+1)}$, and the server stores the most recent client updates $\mathbf{y}_j^{(t+1)}$ in memory for use in the next round. This procedure is iterated for $T$ communication rounds.

    The complete execution of Algorithm~\ref{alg:FedAdaVRAlgorithm} relies on several subroutines: the client update routine (Appendix~\ref{Apd:DeviceUpdate}), the adaptive optimiser implementations (Appendix~\ref{Apd:AdpOpt}), and the quantisation procedure for \textsc{FedAdaVR-Quant} (Appendix~\ref{Apd:ModQuant}). For further details on the specific adaptive optimisers, the reader is referred to~\citep{Adagrad,Adam,Adabelief,Yogi,Lamb}.

    \textbf{Synergy between variance reduction and adaptive optimisers.}
	The proposed algorithm combines two mechanisms whose roles are complementary. Let $\mathbf{G}^{(t)} = \eta_c\mathbf{r}^{(t)}$ be the server-side pseudo-gradient obtained after SAGA-style correction. Under partial participation, $\mathbf{G}^{(t)}$ faces two distinct challenges: (i) sampling bias and variance relative to the full active average due to unvisited clients, and (ii) severe coordinate-wise variance heterogeneity stemming from skewed label distributions across clients. The SAGA-style step~\eqref{VarianceUpdateEquation} controls the partial-participation sampling error (Lemmas~\ref{lem:6} and~\ref{lem:7}), while the adaptive preconditioner (e.g., $\operatorname{diag}(\sqrt{z_t}+\epsilon)^{-1}$ in Adagrad) normalises the heterogeneous coordinate-wise scales. Neither component alone suffices: without VR, partial-participation variance causes late-stage instability; without coordinate-wise adaptivity, the server step size is bottlenecked by the most volatile coordinates, stalling progress on rare-class parameters. As our ablation study (Section~\ref{sec:AblationMain}, Fig.~\ref{fig:ablation_and_time} left) and full benchmark comparisons against fixed-step FedVARP across Adagrad, Adam, AdaBelief, Yogi, and Lamb (Appendix~\ref{Apd:Ablation}, Tables~\ref{tab:shakespeare_accuracy_com}--\ref{tab:CIFAR10_accuracy_com}) confirm, the unified approach consistently outperforms using either technique in isolation.

	\textbf{Where a naive stack of FedVARP + FedOpt differs from \textsc{FedAdaVR}.}\label{diff-naive-stack}
	Because \textsc{FedAdaVR} combines the SAGA-style correction direction $\mathbf{r}^{(t)}$ from FedVARP \citep{Jhunjhunwala2022} with server-side adaptive preconditioning from FedOpt \citep{Reddi2020}, it represents the principled formulation of variance-reduced adaptive federated optimisation, whereas naively piping the output of FedVARP ($\mathbf{g}_t = \eta_{\text{varp}}\eta_c\mathbf{r}^{(t)}$) into standard FedOpt implementations differs in key mechanics, as summarised in Table~\ref{tab:naive_stack_comparison} (Appendix~\ref{sec:ProposedAlgorithmComparison}). First, because adaptive accumulators are scale-covariant ($m_t \propto c$ and $v_t, z_t \propto c^2$ for $c = \eta_{\text{varp}}\eta_c > 0$), the scaling factor cancels in $m_t / \sqrt{v_t}$ and affects only the effective numerical stabiliser $\tau / (\eta_{\text{varp}}\eta_c)$. Second, standard FedOpt implementations omit exact moment bias correction: standard federated adaptive implementations omit exact bias correction ($\hat{m}_t, \hat{v}_t$). \textsc{FedAdaVR} resolves this by feeding the unscaled pseudo-gradient $\mathbf{G}^{(t)} = \eta_c\mathbf{r}^{(t)}$ directly into the accumulators, applying a single server learning rate $\eta_s$ after normalisation, and applying exact bias correction $\hat{m}_t, \hat{v}_t$. Further comparison against MIFA, FedVARP, and SCAFFOLD is provided in Appendix~\ref{sec:ProposedAlgorithmComparison}.

    \section{Convergence Analysis}\label{ConvergenceAnalysis}
	This section presents the formal non-convex convergence analysis of \textsc{FedAdaVR} when instantiated with the server-side Adagrad optimiser (Algorithm~\ref{alg:AdagOpt}). While our empirical evaluation in Section~\ref{Experiments} demonstrates that the decoupled variance-reduced update extends effectively to momentum-based adaptive optimisers (Adam, AdaBelief, Yogi, and Lamb), establishing formal bounds for those variants requires tracking coupled heavy-ball momentum dynamics alongside historical state staleness; we therefore restrict our formal theoretical guarantee strictly to Adagrad. The analysis further assumes full-precision stored states (\textsc{FedAdaVR-Quant} is not covered), vanilla local SGD without client momentum, and no weight decay ($\lambda = 0$). Consistent with the uniform weighting in~\eqref{VarianceUpdateEquation}, the analysis considers $f := \frac{1}{N}\sum_{i=1}^N f_i$ with $f_i \equiv F_i$, i.e., objective~\eqref{eq:1} with $p_n = \frac{1}{N}$ (e.g., equally sized local datasets); non-uniform weights $p_n$ are not covered.

	\begin{assumption}[$L$-smoothness]\label{ass:lsmooth}
		Each local objective function $f_i : \mathbb{R}^d \to \mathbb{R}$ is differentiable and $L$-smooth: for all $i \in [N]$ and all $\mathbf{x}, \mathbf{y} \in \mathbb{R}^d$,
		$\|\nabla f_i(\mathbf{x}) - \nabla f_i(\mathbf{y})\| \leq L\, \|\mathbf{x} - \mathbf{y}\|$.
	\end{assumption}
	\begin{assumption}[Bounded local and global variance]\label{ass:boundvar}
		The stochastic gradient computed at each client is an unbiased estimator of the local gradient, $\mathbb{E}_{\xi_i \sim \mathcal{D}_i}[\nabla f_i(\mathbf{w},\xi_i)] = \nabla f_i(\mathbf{w})$, with bounded local variance:
		$\mathbb{E}_{\xi_i \sim \mathcal{D}_i}\| \nabla f_i(\mathbf{w},\xi_i) - \nabla f_i(\mathbf{w})\|^2 \leq \sigma^2$.
		Additionally, there exists a constant $\sigma_g \geq 0$ bounding the global gradient dissimilarity across clients:
		$\|\nabla f_i(\mathbf{w}) - \nabla f(\mathbf{w})\|^2 \leq \sigma_g^2$, for all $i \in [N]$.
	\end{assumption}
	\begin{assumption}[Bounded stochastic gradients]\label{ass:bg}
		The stochastic gradients at each client are uniformly bounded: there exists a constant $G > 0$ such that for all $i \in [N]$ and $\mathbf{w} \in \mathbb{R}^d$,
		$\|\nabla f_i(\mathbf{w}, \xi_i)\| \leq G$ almost surely (which, by Jensen's inequality, implies $\|\nabla f_i(\mathbf{w})\| \leq G$, $\|\nabla f(\mathbf{w})\| \leq G$, and coordinate-wise $|[\nabla f_i(\mathbf{w}, \xi_i)]_j| \leq G$ for all $j \in [d]$).
	\end{assumption}

	\textbf{Role and necessity of Assumption~\ref{ass:bg}.} Assumptions~\ref{ass:lsmooth} and~\ref{ass:boundvar} are standard in federated optimisation \citep{Karimireddy2019,Jhunjhunwala2022}, and suffice on their own to establish our client-drift and variance-reduction lemmas (Lemmas~\ref{lem:5}--\ref{lem:7}). Assumption~\ref{ass:bg} is invoked exclusively on the adaptive server side, mirroring the stochastic gradient bound required by the adaptive FedOpt analyses of \citet{Reddi2020} (Assumption~3 therein). Specifically, it is required in three places: (i) bounding the monotone Adagrad accumulator $z_{t-1,j} \le 9T\eta_c^2 K^2 G^2$ to guarantee a deterministic lower bound $1/(\sqrt{z_{t-1,j}}+\epsilon) \ge 1/D_T$ on the effective coordinate-wise step size; (ii) bounding $|[\nabla f(\mathbf{w}^{(t)})]_j \tilde{G}_j^{(t)}| \le 3G^2$ to telescope the preconditioner-drift term across rounds; and (iii) bounding the one-step iterate drift $\|\tilde{\mathbf{G}}^{(t-1)}\|^2 \le 9G^2$ in the staleness recursion. Without Assumption~\ref{ass:bg}, a single unbounded stochastic gradient early in training could permanently inflate the monotone accumulator $z_{t-1,j}$, preventing a deterministic lower bound on $\inf_{t \le T} 1/(\sqrt{z_{t-1,j}}+\epsilon)$ and obstructing the telescoping of the descent recursion. We note that the proof only invokes Assumption~\ref{ass:bg} at the global iterates $\mathbf{w}^{(t)}$ and the local iterates $\mathbf{w}_i^{(t,k)}$ visited by the algorithm, so the uniform statement over $\mathbb{R}^d$ above is a sufficient condition; relaxing it via affine-variance Adagrad techniques is a promising direction for future work.

	\begin{theorem}[Convergence of \textsc{FedAdaVR} with Adagrad] \label{theo:con}
		Suppose the local functions $\{f_i\}_{i=1}^N$ ($N \ge 2$) satisfy Assumptions~\ref{ass:lsmooth}, \ref{ass:boundvar}, and~\ref{ass:bg} with uniform client weights $p_i = \frac{1}{N}$. In each round $t \in \{0, \dots, T-1\}$, the server selects a subset $\mathcal{S}^{(t)}$ of $|\mathcal{S}^{(t)}| = M$ clients uniformly at random without replacement to perform $K \ge 1$ local SGD steps. Define the participation factors $A := 1 + \frac{13(N-M)}{2M(N-1)}$ and $B := 1 + \frac{4(N-M)}{N-1}$, the effective server step $\tilde{\eta}_s := \eta_s \eta_c K$, the accumulator upper bound $D_T := 3\eta_c K G \sqrt{T} + \epsilon$, and $\Delta_0 := f(\mathbf{w}^{(0)}) - f^*$ where $f^* = \min_{\mathbf{w}} f(\mathbf{w})$. If the client and server learning rates satisfy
		\begin{align*}
			\eta_c &\le \min\left\{ \frac{1}{2LK},\; \frac{1}{8\sqrt{2A}\,L\sqrt{K(K-1)}},\; \left(\frac{\epsilon}{384 A L^2 K^2 (K-1) G \sqrt{T}}\right)^{1/3} \right\}, \\
			\tilde{\eta}_s = \eta_s \eta_c K &\le \min\left\{ \frac{\epsilon}{4L},\; \frac{\epsilon^2}{16 L D_T} \right\},
		\end{align*}
		(for $K = 1$, the second and third bounds on $\eta_c$ are void, i.e., read as $+\infty$), then the iterates $\{\mathbf{w}^{(t)}\}_{t=0}^{T-1}$ generated by \textsc{FedAdaVR} satisfy
		\begin{align*}
			\min_{0 \leq t \leq T-1} \mathbb{E}\bigl\|\nabla f(\mathbf{w}^{(t)})\bigr\|^2
			&\le \frac{4 D_T (\Delta_0 + \Psi_0)}{\tilde{\eta}_s T}
			+ \frac{12 D_T d G^2}{\epsilon T}
			+ \frac{4 D_T}{\tilde{\eta}_s}\bigl(\bar{A}_1 \sigma_g^2 + \bar{A}_2 \sigma^2 + \bar{A}_3\bigr),
		\end{align*}
		where $d$ is the model dimension, $\Psi_0 := \frac{4\tilde{\eta}_s^2 L N(N-M)}{\epsilon^2 M^2(N-1)} G^2$ bounds the initial historical-state mismatch, and the per-round error constants are given by
        \begin{align*}
            \bar{A}_1 &:= \frac{8 \tilde{\eta}_s A \eta_c^2 L^2 K(K-1)}{\epsilon}, \qquad 
            \bar{A}_2 := \frac{2 \tilde{\eta}_s A \eta_c^2 L^2 (K-1)}{\epsilon} + \frac{2 \tilde{\eta}_s^2 L B}{\epsilon^2 M K}, \\
            \bar{A}_3 &:= \frac{36 \tilde{\eta}_s^4 L^3 (N-M)}{\epsilon^4 M(N-1)}\left(1 + \frac{2N(N-M)}{M^2}\right) G^2.
        \end{align*}

	\end{theorem}

	\begin{remark}[Necessity of local steps and non-vanishing learning rates]
		The upper bounds on $\eta_c$ and $\tilde{\eta}_s$ in Theorem~\ref{theo:con} represent sufficient worst-case conditions for the non-convex smoothness analysis. They do not imply that $\eta_c \to 0$ is optimal: setting $\eta_c = 0$ yields $\tilde{\eta}_s = 0$, causing the leading $\Delta_0 / (\tilde{\eta}_s T)$ term to diverge. In practice, a strictly positive $\eta_c > 0$ with $K > 1$ balances local progress against communication overhead, as confirmed by our hyperparameter studies (Tables~\ref{tab:shakespeare_accuracy_com}--\ref{tab:CIFAR10_accuracy_com}).
	\end{remark}

	The complete proof of Theorem~\ref{theo:con}, along with the supporting lemmas, is provided in Appendix~\ref{app-convergence-proof}. Comparing Theorem~\ref{theo:con} against the FedAvg error decomposition of \citet{Jhunjhunwala2022}, which decomposes federated optimisation error into stochastic gradient variance, client drift, and partial participation error, shows that \textsc{FedAdaVR} eliminates the $\mathcal{O}(\sigma_g^2)$ partial participation bias: $\sigma_g^2$ enters solely through the higher-order client-drift term $\bar{A}_1 \sigma_g^2$. Moreover, every partial-participation overhead ($A - 1$, $B - 1$, $\Psi_0$, and $\bar{A}_3$) decreases monotonically as the cohort size $M$ increases and vanishes under full participation ($M = N$), where $A = B = 1$; the local stochastic-gradient noise and client-drift terms remain.

	To obtain an explicit asymptotic rate in terms of $T$, $K$, and $M$, we specialise Theorem~\ref{theo:con} under standard parameter choices.
	\begin{corollary}\label{cor:rate}
		Suppose Assumptions~\ref{ass:lsmooth}--\ref{ass:bg} hold and we set $\eta_c = \Theta\bigl(\frac{1}{KL\sqrt{T}}\bigr)$, $\epsilon = \Theta\bigl(\frac{G}{L}\bigr)$, and $\eta_s = \Theta(1)$. Then for fixed $N$, $M$, and $K$ and sufficiently large $T$, the iterates of \textsc{FedAdaVR} satisfy
        \begin{align}\label{eq:cor_rate}
            \min_{0 \leq t \leq T-1} \mathbb{E}\bigl\|\nabla f(\mathbf{w}^{(t)})\bigr\|^2 
            &= \mathcal{O}\left(\frac{G \Delta_0}{\eta_s \sqrt{T}}\right) 
            + \mathcal{O}\left(\frac{L \eta_s B \sigma^2}{G M K \sqrt{T}}\right) \nonumber \\
            &\quad + \mathcal{O}\left(\frac{d G^2}{T}\right) 
            + \mathcal{O}\left(\frac{A (K-1) \sigma^2}{K^2 T} + \frac{A (K-1) \sigma_g^2}{K T}\right).
        \end{align}

	\end{corollary}
	Corollary~\ref{cor:rate} establishes an $\mathcal{O}(1/\sqrt{T})$ convergence rate matching the canonical non-convex rate of federated Adagrad \citep{Reddi2020}, while ensuring that the dominant $\mathcal{O}(1/\sqrt{T})$ terms are completely free of the global heterogeneity variance $\sigma_g^2$ and that the sampling variance scales inversely with both $M$ and $K$ (since $1 \le B \le 5$). The omitted $\Psi_0$ and $\bar{A}_3$ contributions are of order $T^{-3/2}$ for fixed $N$ and $M$, but carry participation factors of order $N/M^2$ and $N^2/M^3$ respectively; the $1/(MK)$ scaling of the sampling-variance term is therefore visible once $T \gtrsim K N^2/M^2$.

    \section{Experiments}\label{Experiments}
    
    The experimental evaluation comprises four primary components: an overview of the experimental setup across diverse configurations (further details in Appendix~\ref{Apd:ExpSetup}); a comparative analysis of the proposed algorithm against state-of-the-art baselines (extended analysis in Appendix~\ref{Apd:ExpResComSOTA}); an evaluation of \textsc{FedAdaVR} and \textsc{FedAdaVR-Quant} utilising various adaptive optimisers (additional results in Appendix~\ref{Apd:ComOwn}); and ablation and System Efficiency. Due to space constraints, detailed ablation study of \textsc{FedAdaVR} (Appendix~\ref{Apd:Ablation}), details and results comparison of the quantisation methods (Appendix~\ref{Apd:ModQuant} and~\ref{Apd:ComQuantAccuracy}), and the computational and communication comparisons against baselines (Appendix~\ref{Apd:ComLatComSOTA}) are provided exclusively in the Appendix. Unless otherwise specified, all experiments utilised FP16 quantisation for \textsc{FedAdaVR-Quant}; due to computational constraints, a comprehensive ablation comparing FP32, FP16, Int8, and Int4 precisions was conducted exclusively on the CIFAR-10 dataset (LQ-1 partition) and is detailed in Appendix~\ref{Apd:ComQuantAccuracy}.
	
	\subsection{Experimental Configuration}\label{ExperimentalSetup}
    
    All experiments used the Flower framework~\citep{FlowerPaper} with PyTorch~\citep{PyTorchPaper} for local model training. We evaluate on three vision datasets (MNIST, FMNIST, and CIFAR-10) and one NLP dataset (Shakespeare), across six data partitioning strategies: IID, IID-NonIID, Dirichlet, LQ-1, LQ-2, and LQ-3~\citep{LiPartitioning} for vision, and the Natural ID strategy from the LEAF benchmark~\citep{Caldas2018} for Shakespeare. Client participation rates are deliberately extreme: 1\% for MNIST and FMNIST, 2\% for CIFAR-10, and below 1\% for Shakespeare, with five clients selected per round in all cases. For vision, we adopt LeNet-5~\citep{Lecun1998} for MNIST and FMNIST, and ResNet-18~\citep{resnet} with group normalisation~\citep{Hsieh2019} for CIFAR-10. For Shakespeare we use a GRU-based RNN~\citep{Chung2014}. Hyperparameters were selected via grid search for all methods to ensure a fair comparison, with most experiments conducted using one random seed, some using two, and a small subset using three seeds to support reproducibility. Full details of datasets, architectures, partitioning strategies, hyperparameter grids, and evaluation protocols are provided in Appendix~\ref{Apd:ExpSetup}.

	\subsection{Comparison with State-of-the-Art Methods}\label{ExperimentalResultsComparisonSOTA}
    
	To evaluate the proposed algorithms against state-of-the-art methods, we compared them with the most relevant baselines: FedAvg~\citep{McMahan2016}, FedProx~\citep{Li2018}, SCAFFOLD~\citep{Karimireddy2019}, FedNova~\citep{Wang2020}, FedAdam~\citep{Reddi2020}, FedAdagrad~\citep{Reddi2020}, FedYogi~\citep{Reddi2020}, MIFA~\citep{Gu2021}, and FedVARP~\citep{Jhunjhunwala2022}. Accuracy and training loss served as the primary metrics for performance assessment. Figure~\ref{fig:final_evaluation_accuracy_all} illustrates only the accuracy trajectories, while Figures~\ref{fig:final_evaluation_all_iid}--\ref{fig:final_evaluation_all_lq3} in Appendix~\ref{Apd:ExpResComSOTA} present the accuracy and loss curves across varying data partitioning strategies and vision datasets. Figure \ref{fig:Shakespeare} in Appendix \ref{Apd:Shakespeare} illustrates the results on the Shakespeare dataset.

	\textsc{FedAdaVR} consistently outperforms all baseline algorithms in nearly all scenarios, despite incurring no additional communication or computational overhead at the clients. The experimental setup posed significant challenges (variance and data heterogeneity), particularly under the LQ-1 partitioning scheme (where each client holds data from only one class) and with very limited client participation rates (1\% for MNIST and FMNIST, 2\% for CIFAR-10 and, $< 1\%$ for Shakespeare). This setting renders learning especially difficult for all algorithms. In LQ-1, data are sorted by label and partitioned into a single chunk (two chunks for LQ-2 and three for LQ-3), such that each partition contains data from only one (or two or three) class(es).
	
	FedAdam failed to cope with this exceptionally demanding scenario on MNIST and FMNIST datasets and also underperformed in IID data distributions combined with extreme partial client participation. While the other baseline methods remained functional, both \textsc{FedAdaVR} and its quantised variant, \textsc{FedAdaVR-Quant}, demonstrated superior performance. FedYogi exhibited strong results throughout the experiments, which was somewhat unexpected given the low client participation rates, indicating its capacity to handle partial client participation errors. Although FedVARP outperformed other baselines under the highly challenging LQ-1 data partitioning on the CIFAR-10 dataset (exhibiting improved robustness to partial client participation error), it still falls significantly short of the performance achieved by our proposed algorithms. MIFA also achieved reasonable performance but suffered from slower convergence. Similarly, FedAvg, FedAdagrad, FedProx, SCAFFOLD, and FedNova displayed sluggish convergence rates. While FedAdagrad remains a strong competitor in the NLP domain (Shakespeare), our method matches its performance while strictly outperforming it in vision, thereby establishing a broader cross-domain superiority.

    A detailed analysis of system-side metrics, including communication cost, server memory, client computation overhead, per-round wall-clock time, end-to-end time to reach accuracy thresholds, and communication round efficiency, is provided in Appendix~\ref{Apd:ComLatComSOTA}. In summary, whilst \textsc{FedAdaVR} incurs a minor per-round latency increase over FedAvg (178.31s vs.\ 166.43s) due to server-side variance reduction, its end-to-end wall-clock cost is substantially lower: \textsc{FedAdaVR} reaches 45\% accuracy on CIFAR-10 (LQ-1) in 163,689 seconds compared to 232,336 seconds for FedAvg, a reduction of approximately 30\%. Furthermore, \textsc{FedAdaVR} requires significantly fewer communication rounds than all baselines to reach equivalent accuracy thresholds, confirming that the per-round overhead is more than compensated by faster convergence.

    \begin{figure}[tb]
        \centering
        \includegraphics[scale=0.065]{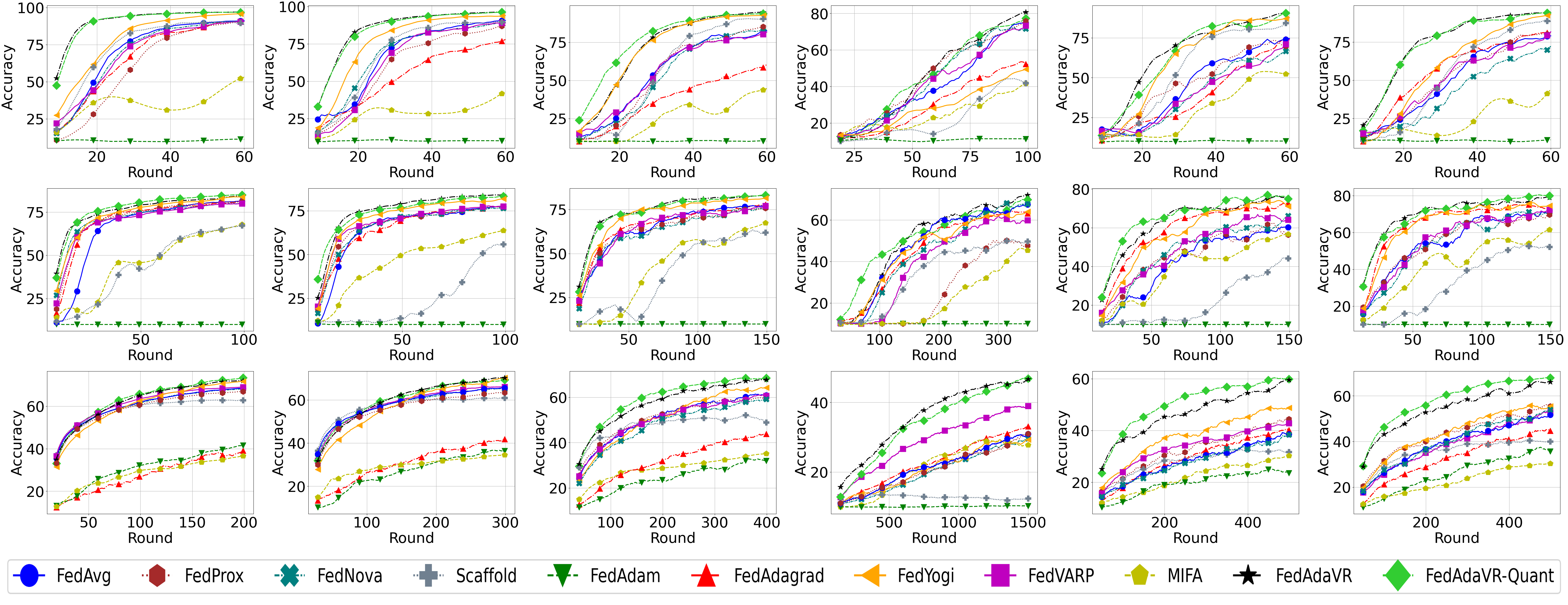}
        \caption{Comparison of accuracy across various data partitioning methods—IID, IID-NonIID, Dirichlet, LQ-1, LQ-2, and LQ-3 (left to right)—for the MNIST (top), FMNIST (middle), and CIFAR-10 (bottom) datasets. For larger versions of these plots alongside their corresponding loss metrics, please refer to Figures~\ref{fig:final_evaluation_all_iid}--\ref{fig:final_evaluation_all_lq3} in Appendix~\ref{Apd:ExperimentalResults}.}
        \label{fig:final_evaluation_accuracy_all}
    \end{figure}
	
	\subsection{Comparison Between \textsc{FedAdaVR} and \textsc{FedAdaVR-Quant} Across Adaptive Optimisers}\label{ExperimentalResultsComparisonOurs}

    Both algorithms exhibit rapid convergence and comparable performance. Despite storing previous local model weights in a quantised format, \textsc{FedAdaVR-Quant} maintains comparable accuracy while offering enhanced memory efficiency. This observation suggests that FL algorithms employing memory to retain historical updates for global aggregation are only marginally impacted by quantisation. This minimal change arises from the negligible difference between client weights pre- and post-quantisation. When computing $\mathbf{y}^{(t)}$, quantisation errors are randomly distributed, leading to an inconsequential deviation that does not materially affect the performance of the algorithm (may act as a mild regulariser \citep{Tang2026}). Therefore, our quantised version sometimes performs better. 
	
	Both \textsc{FedAdaVR} and \textsc{FedAdaVR-Quant} employ adaptive optimisers. Tables~\ref{tab:shakespeare_accuracy_com}, \ref{tab:MNIST_accuracy_com} \ref{tab:FashiontMNIST_accuracy_com}, and \ref{tab:CIFAR10_accuracy_com} detail the accuracy achieved under various optimiser and server learning rate. These tables present average accuracy over the final training rounds, as detailed in Table \ref{tab:dataset_partition}. The best accuracy values and configurations for both algorithms, as recorded in Tables~\ref{tab:shakespeare_accuracy_com}, \ref{tab:MNIST_accuracy_com}, \ref{tab:FashiontMNIST_accuracy_com}, and \ref{tab:CIFAR10_accuracy_com}, are illustrated in Figures~\ref{fig:final_evaluation_accuracy_all}, and \ref{fig:final_evaluation_all_iid}--\ref{fig:Shakespeare}. A full breakdown of accuracy under each optimiser and server learning rate configuration is provided in Appendix~\ref{Apd:ComOwn}, with discussion of optimiser selection in Appendix~\ref{Apd:DisOwnOpt}.

    \subsection{Ablation and System Efficiency}\label{sec:AblationMain}
	
	To evaluate the contributions of our components, we conduct an ablation study on the CIFAR-10 dataset with LQ-1 partitioning. To assess the practical overhead of the algorithms, we perform a system-side efficiency analysis. As shown in Fig.~\ref{fig:ablation_and_time} (left), removing the SAGA-style correction (\textsc{FedAdaVR-NoVR}) causes severe late-stage training instability and loss spikes, providing a direct empirical manifestation of uncorrected partial-participation estimator variance. Conversely, removing the adaptive preconditioner (\textsc{FedAdaVR-NoOpt}) eliminates these oscillations via global variance reduction, yet convergence severely stagnates; this confirms that under extreme label skew (LQ-1), resolving global participation bias alone is insufficient without coordinate-wise step-size adaptation to balance sparse, high-variance rare-class updates against dense frequent-class signals. Furthermore, while server-side variance reduction incurs a minor per-round latency overhead, it drastically reduces total communication rounds: as shown in Fig.~\ref{fig:ablation_and_time} (right), \textsc{FedAdaVR} reduces end-to-end wall-clock time to reach 45\% accuracy by $\sim 30\%$ compared to FedAvg (full details in Appendices~\ref{Apd:Ablation} and~\ref{Apd:ComLatComSOTA}).

    \begin{figure}[ht]
        \centering
        \begin{minipage}[c]{0.5\textwidth}
            \centering
            \includegraphics[width=\textwidth]{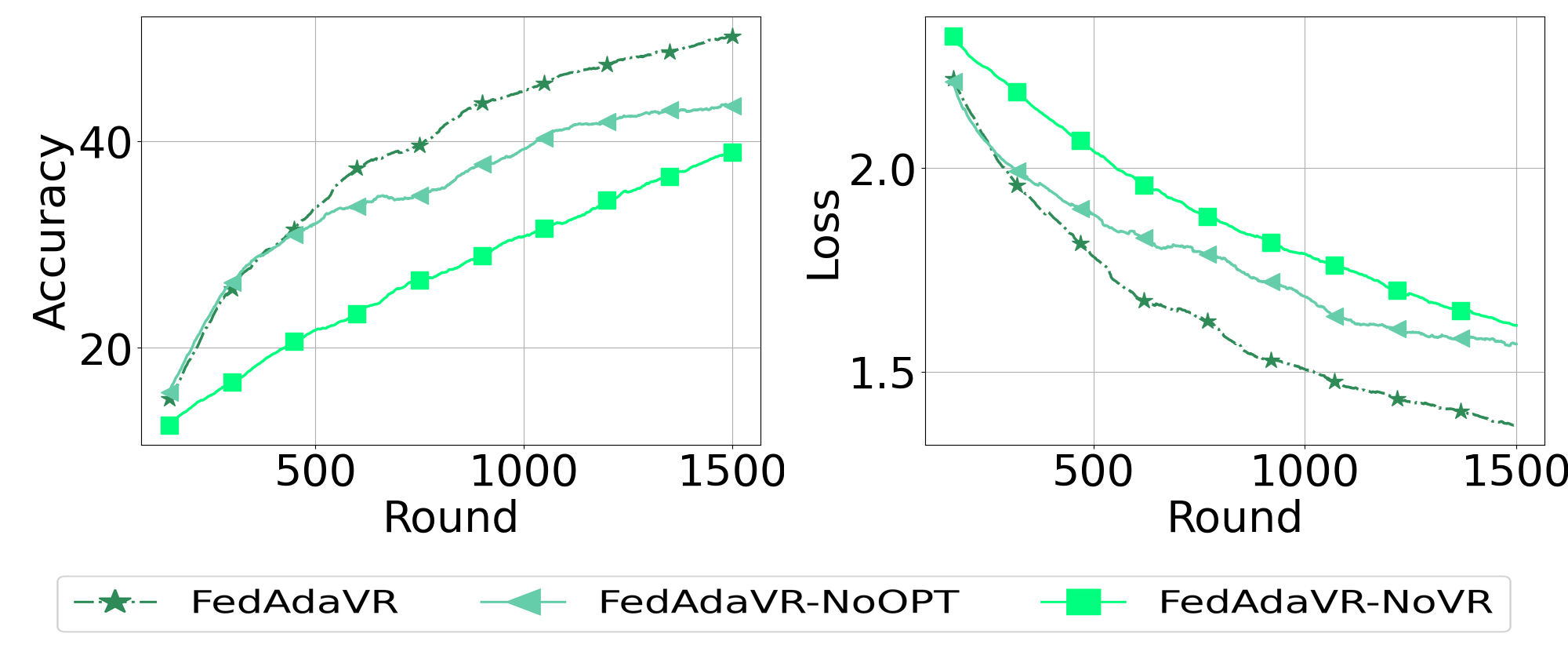} 
        \end{minipage}
        \hfill
        \begin{minipage}[c]{0.47\textwidth}
            \centering
            \begin{tabular}{lrr}
                \toprule
                \textbf{Algorithm} & \textbf{Time/Rnd} & \textbf{Time to 45\%} \\
                \midrule
                FedAvg & 166s & 232,336s \\
                FedNova & 168s & 251,457s \\
                FedVARP & 175s & 242,749s \\
                \midrule
                \textbf{\textsc{FedAdaVR}} & 178s & \textbf{163,689s} \\
                \bottomrule
            \end{tabular}
        \end{minipage}
        \caption{\textbf{Left:} Ablation study on CIFAR-10 (LQ-1) demonstrating the necessity of both variance reduction and adaptive optimisation. \textbf{Right:} End-to-end wall-clock time required to reach 45\% accuracy on CIFAR-10, highlighting \textsc{FedAdaVR}'s 30\% speedup.}
        \label{fig:ablation_and_time}
    \end{figure}

    \section{Conclusion and Future Work}\label{conclusion}
	This study introduces \textsc{FedAdaVR}, a novel FL algorithm integrating an adaptive optimiser with variance reduction mechanisms by decoupling historical client-state maintenance from adaptive moment preconditioning. Extensive experiments across vision and NLP benchmarks show this approach enhances convergence speed while mitigating variance from extreme partial client participation ($1\text{--}2\%$) and severe label skew. The algorithm leverages server-side memory to retain the most recent client updates, thereby emulating continuous client availability during aggregation. We also present non-convex convergence results of \textsc{FedAdaVR} and prove that our proposed algorithm can asymptotically eliminate first-order partial client participation error while improving monotonically with participation. To further optimise memory efficiency, we propose a quantised variant, \textsc{FedAdaVR-Quant}, wherein updates are stored in FP16, Int8, or Int4 formats, reducing server memory requirements by up to 87.5\% while preserving generalisation accuracy.
	
	\textbf{Limitations and Future Work.} While \textsc{FedAdaVR} and \textsc{FedAdaVR-Quant} consistently outperform state-of-the-art baselines, our current study focuses on three deterministic linear quantisation schemes. Future work will explore advanced compression and hybrid caching to scale server memory to foundation-model regimes, alongside staleness-aware decay for temporal distribution shifts. Furthermore, while our formal analysis focuses on Adagrad and leaves the coupled quantisation error bounds of \textsc{FedAdaVR-Quant} for future work, our experimental results confirm strong practical efficacy across adaptive optimisers and low-precision states \citep{Tang2026}. For an extended discussion of these limitations, we refer the reader to Appendix~\ref{Apd:Limitations}.

    \begin{ack}
        The work has been supported by the SLAIDER project funded by the UK Research and Innovation Grant EP/Y018281/1. Xiao Chen is supported by the National Natural Science Foundation of China (NSFC) under Grant No. 62672224. This research used the ALICE High Performance Computing facility at the University of Leicester.
    \end{ack}

\bibliography{references_file}
\bibliographystyle{plainnat}


\appendix

\clearpage

\addcontentsline{toc}{section}{Appendix}
\tableofcontents
\newpage

\section{Extended Related Work and Algorithmic Differences}\label{RelatedWork}

\subsection{Background and Related Literature}
	\paragraph{Federated Learning.} FedAvg \citep{McMahan2016} is the foundational algorithm in federated learning, wherein a server initialises a global model and dispatches it to clients, who perform several steps of local training using their private data. Clients subsequently return their local updates, which the server aggregates to form an updated global model. This process iterates until convergence. Following FedAvg, numerous FL algorithms have been introduced to address its inherent limitations and the broader challenges within FL. For an overview of these challenges, we refer the reader to \citep{Li2020a,Rahman2021,Wen2022,Guendouzi2023}.
	\paragraph{Adaptive Optimisers.} Prior to the advent of federated learning, adaptive optimisers such as Adam \citep{Adam} and Adagrad \citep{Adagrad} were widely recognised for their effectiveness in guiding convergence behaviour. Building upon this success, the authors of FedOpt \citep{Reddi2020} pioneered the integration of adaptive optimisers (e.g., Adagrad \citep{Adagrad}, Adam \citep{Adam}, and Yogi \citep{Yogi}) into federated settings. In the FedOpt family of algorithms, stochastic gradient descent (SGD) is retained on the client side to preserve communication parity with FedAvg, while three distinct adaptive methods (i.e. FedAdagrad, FedAdam, and FedYogi) are employed on the server side \citep{Reddi2020}. The Lamb optimiser \citep{Lamb} was subsequently utilised to develop FedLamb \citep{Karimi2021}, whereas the AmsGrad optimiser \citep{AmsGrad} was adopted in a separate FL algorithm \citep{Chen2021}.
	\paragraph{Variance Reduction Techniques.} Variance reduction techniques have emerged as potent tools for accelerating the convergence of stochastic optimisation by mitigating gradient noise. In SAGA \citep{Defazio2014}, for instance, a table of past gradients is maintained and updated to construct an unbiased estimator at each iteration, yielding linear convergence for strongly convex functions without requiring full-batch gradients. Related methods include SAG \citep{Schmidt2014}, which averages stored gradients to reduce variance. SVRG \citep{Rie2013} avoids storage overhead by alternating between full-gradient snapshots and inner stochastic loops. SARAH \citep{Nguyen2017} further refines this approach by employing recursive gradient estimators without necessitating full-batch gradients. Numerous other VR methods have been proposed for centralised stochastic problems without requiring additional memory. In the context of FL, these methods are generally classified into two categories: \textit{SVRG-style Variance Reduction}, which necessitates client participation in specific rounds, and \textit{Momentum-based Variance Reduction}, which incurs both double communication and double computation costs \citep{Jhunjhunwala2022}. AdaLVR combines variance reduction with Adagrad optimiser to enhance finite-sum optimisation performance \citep{Batardière2024}. Our proposed algorithms are influenced by AdaLVR but are designed to avoid the aforementioned drawbacks.
	\paragraph{Quantisation.} Quantisation refers to the process of reducing the numerical precision of model parameters to produce a more compact version of the same neural network. Various quantisation techniques have previously been employed to address diverse FL challenges, particularly communication heterogeneity \citep{Reisizadeh2019,Mao2022,Ren2023,Chen2024}, and device heterogeneity \citep{Gupta2022,Abdelmoniem2021,Chen2024}. The \textsc{FedAdaVR-Quant} variant of our algorithm employs quantisation to store the most recent client updates, enabling their reuse in subsequent training rounds. To demonstrate the practical feasibility of our approach, we evaluated three fundamental and computationally efficient formats: FP16, Int8, and Int4. While diverse integer quantisation methods exist in the literature, \textsc{FedAdaVR-Quant} specifically adopts the 8-bit symmetric per-tensor uniform quantiser proposed by \citep{Krishnamoorthi2018} and the 4-bit methodology of \citep{Zhou2016}, which packs two 4-bit values per byte for storage.

\subsection{Detailed Differences from Key Baselines} \label{sec:ProposedAlgorithmComparison}
	This work relates closely to several state-of-the-art strategies. A brief comparison of these methods with the proposed algorithm is provided below:
	
	\textbf{MIFA.} Authors of MIFA were among the first to utilise memory for storing the latest observed client updates in FL to address the issue of unpredictable client unavailability \citep{Gu2021}. Their aggregation procedure follows a SAG-like approach \citep{Schmidt2014}, which assigns equal weight to both the most recent and previous updates, thereby impeding convergence speed. Moreover, MIFA requires all clients to participate in the initial round, a condition rarely met in practical federated learning scenarios. In contrast, neither of the proposed algorithms relies on uniform weighting, nor do they require complete client participation at the outset.

    \begin{table}[t]
		\caption{Mechanistic comparison of optimiser dynamics across FedVARP, FedOpt, a naive composition (FedVARP + FedOpt), and \textsc{FedAdaVR}. Here $\tilde{m}_t, \tilde{v}_t, \tilde{z}_t$ denote accumulators formed on the unscaled base signal $\mathbf{r}^{(t)}$, $\mathcal{T}$ is the round counter, and $\eta_{\text{opt}}$ is the server rate of the stacked optimiser.}
		\label{tab:naive_stack_comparison}
		\centering
		\setlength{\tabcolsep}{8pt}
		\renewcommand{\arraystretch}{3.4}
		\begin{tabular}{l l l l l}
			\toprule
			\textbf{Step} 
			& \textbf{\makecell[l]{Standard\\FedVARP}} 
			& \textbf{\makecell[l]{Standard\\FedOpt}} 
			& \textbf{\makecell[l]{Naive Stack\\(FedVARP \\+ FedOpt)}} 
			& \textbf{\textsc{FedAdaVR}} \\
			\midrule
			\textbf{Signal} 
			& \makecell[l]{$\mathbf{r}^{(t)} = $\\ $\text{VR}_{\text{Correction}}$} 
			& \makecell[l]{$\mathbf{g}_t = $\\ $\text{Client}_{\text{avg}}$} 
			& \makecell[l]{$\mathbf{g}_t = \eta_{\text{varp}} \cdot$ \\ $\eta_c \cdot \mathbf{r}^{(t)}$} 
			& \makecell[l]{$\mathbf{G}^{(t)} = $ \\ $\eta_c \cdot \mathbf{r}^{(t)}$} \\
			
			\textbf{\makecell[l]{$m_t$\\ (Adam, Yogi)}} 
			& \emph{(None)} 
			& \makecell[l]{$\cdots + $ \\ $(1-\beta_1)\mathbf{g}_t$} 
			& \makecell[l]{$\cdots + (1-\beta_1) \cdot$ \\ $(\eta_{\text{varp}}\eta_c \mathbf{r}^{(t)})$} 
			& \makecell[l]{$\cdots + $ \\ $(1-\beta_1)\mathbf{G}^{(t)}$} \\
			
			\textbf{\makecell[l]{$v_t$\\ (Adam)}} 
			& \emph{(None)} 
			& \makecell[l]{$\cdots + $ \\ $(1-\beta_2)\mathbf{g}_t^2$} 
			& \makecell[l]{$\cdots + (1-\beta_2) \cdot$ \\ $(\eta_{\text{varp}}\eta_c \mathbf{r}^{(t)})^2$} 
			& \makecell[l]{$\cdots + $ \\ $(1-\beta_2)(\mathbf{G}^{(t)})^2$} \\
			
			\textbf{\makecell[l]{$v_t$\\ (Yogi)}} 
			& \emph{(None)} 
			& \makecell[l]{$v_{t-1} - (1-\beta_2)\mathbf{g}_t^2$ \\ $\mathrm{sgn}(v_{t-1}-\mathbf{g}_t^2)$}  
			& \makecell[l]{\emph{same on} \\ $\eta_{\text{varp}}\eta_c\mathbf{r}^{(t)}$} 
			& \makecell[l]{\emph{same on} \\ $\mathbf{G}^{(t)}$} \\
			
			\textbf{\makecell[l]{$z_t$\\ (Adagrad)}} 
			& \emph{(None)} 
			& \makecell[l]{$z_{t-1} + $ \\ $\mathbf{g}_t^2$} 
			& \makecell[l]{$z_{t-1} + $\\ $(\eta_{\text{varp}}\eta_c\mathbf{r}^{(t)})^2$} 
			& \makecell[l]{$z_{t-1} + $ \\ $(\mathbf{G}^{(t)})^2$} \\
			
			\textbf{\makecell[l]{Bias\\ correction}} 
			& --- 
			& \makecell[l]{none \\ (all three)} 
			& none 
			& \makecell[l]{$\hat{m}_t = \frac{m_t}{1-\beta_1^{\mathcal{T}}}$, \\ $\hat{v}_t = \frac{v_t}{1-\beta_2^{\mathcal{T}}}$\\ (Adam, Yogi); \\ none (Adagrad)} \\
			
			\textbf{Update} 
			& \makecell[l]{$\mathbf{w}^{(t)} - $\\ $\eta_{\text{varp}}(\eta_c\mathbf{r}^{(t)})$} 
			& \makecell[l]{$\mathbf{w}^{(t)} - $\\ $\eta_{\text{opt}}\frac{m_t}{\sqrt{v_t}+\tau}$} 
			& \makecell[l]{$\mathbf{w}^{(t)} - $\\ $\eta_{\text{opt}}\frac{m_t}{\sqrt{v_t}+\tau}$} 
			& \makecell[l]{$\mathbf{w}^{(t)} - $\\ $\eta_s\frac{\hat{m}_t}{\sqrt{\hat{v}_t}+\epsilon}$} \\
			
			\textbf{\makecell[l]{Final ---\\ Adagrad}} 
			& \makecell[l]{\emph{(Direct} \\ \emph{Step)}} 
			& \makecell[l]{\emph{(Std.} \\ \emph{Adagrad)}} 
			& \makecell[l]{$\approx \mathbf{w}^{(t)} - $ \\ $\eta_{\text{opt}}\frac{\mathbf{r}^{(t)}}{\sqrt{\tilde{z}_t} + \frac{\tau}{\eta_{\text{varp}}\eta_c}}$} 
			& \makecell[l]{$\approx \mathbf{w}^{(t)} - $ \\ $\eta_s\frac{\mathbf{r}^{(t)}}{\sqrt{\tilde{z}_t} + \frac{\epsilon}{\eta_c}}$} \\
			
			\textbf{\makecell[l]{Final ---\\ Adam / Yogi}} 
			& \makecell[l]{\emph{(Direct} \\ \emph{Step)}} 
			& \makecell[l]{\emph{(Std.} \\ \emph{Adam/Yogi)}} 
			& \makecell[l]{$\approx \mathbf{w}^{(t)} - $ \\ $\eta_{\text{opt}}\frac{\tilde{m}_t}{\sqrt{\tilde{v}_t} + \frac{\tau}{\eta_{\text{varp}}\eta_c}}$} 
			& \makecell[l]{$\approx \mathbf{w}^{(t)} - $ \\ $\eta_s\frac{\hat{\tilde{m}}_t}{\sqrt{\hat{\tilde{v}}_t} + \frac{\epsilon}{\eta_c}}$} \\
			\bottomrule
		\end{tabular}
	\end{table}
	
	\textbf{FedVARP and ClusterFedVARP.} 		Variance due to partial client participation is a major source of error in FL, first addressed by the authors of FedVARP \citep{Jhunjhunwala2022}. FedVARP applies a SAGA-like variance reduction technique on the server to mitigate errors from partial participation. Its variant, ClusterFedVARP, groups clients into clusters to improve memory efficiency by storing a single cluster update rather than individual updates. Although both variants outperform earlier state-of-the-art methods in highly non-IID settings, experiments show relatively slow convergence, potentially due to the fixed server learning rate (absence of server-side adaptivity). The proposed \textsc{FedAdaVR} adopts a similar SAGA-based reduction but replaces the fixed learning rate with an adaptive optimiser. According to our experiments, these refinements lead to improved performance in both IID and non-IID cases. To further enhance memory efficiency, \textsc{FedAdaVR-Quant} is proposed, where the latest client updates are stored in a quantised representation.

	\textbf{SCAFFOLD.} Designed to address client drift error, SCAFFOLD is among the earliest works to highlight this challenge \citep{Karimireddy2019}. It employs a SAGA-like technique \citep{Defazio2014} enhanced with control variates, which necessitate communication between server and clients. While effective, these control variates increase communication overhead and slow down the overall federated learning process. In contrast, \textsc{FedAdaVR} requires clients to perform local SGD and transmit only local weight updates, without additional information for variance reduction. Consequently, \textsc{FedAdaVR} achieves lower communication costs.
	
	\textbf{FedOpt.} FedOpt utilises adaptive optimisers such as Adam, Adagrad, and Yogi in federated settings \citep{Reddi2020}. Inspired by the success of adaptive optimisers in non-federated contexts, these algorithms have demonstrated substantial performance improvements. The current work also employs these optimisers following the application of SAGA-like variance reduction. Experimental findings indicate that integrating adaptive optimisation after variance reduction further accelerates convergence. Although Adabelief and Lamb optimisers are additionally considered here, other research applying these specific optimisers directly has not been included in this study. A detailed mechanistic comparison of how these baseline (FedOpt and FedVARP) update rules and naive compositions differ from our unified \textsc{FedAdaVR} framework is provided in Section~\ref{diff-naive-stack} and illustrated in Table~\ref{tab:naive_stack_comparison}.

    \section{Convergence Analysis for \textsc{FedAdaVR}} \label{app-convergence-proof}
	
	This appendix presents the complete, self-contained convergence proofs for Lemmas~\ref{lem:5}--\ref{lem:8}, Theorem~\ref{theo:con}, and Corollary~\ref{cor:rate}.
	
	\subsection{Basic Notations and Normalised Quantities}\label{Apd:BasicNotations}
	Following standard federated optimisation analyses \citep{Jhunjhunwala2022,Reddi2020}, we consider uniform client weights $p_i = \frac{1}{N}$ (yielding an active average of $\frac{1}{M}\sum_{i \in \mathcal{S}^{(t)}}$) for the theoretical convergence analysis. Note that while our code implementation weights clients proportionally to their local dataset sizes, this proportional weighting is mathematically identical to uniform weighting for balanced partitions (where local dataset sizes are equal) and serves as a standard practical approximation for theoretical analyses in unbalanced non-IID settings. Let $\mathcal{F}_t$ denote the $\sigma$-algebra generated by all stochasticity prior to communication round $t$. Within round $t$, for each local step $k \in \{0, \dots, K-1\}$, let $\mathcal{F}_{t,k}$ denote the $\sigma$-algebra generated by $\mathcal{F}_t$ and the local mini-batch draws $\{\xi_i^{(t,\tau)} : i \in [N],\, 0 \le \tau < k\}$, where $\xi_i^{(t,k)} \sim \mathcal{D}_i$ are sampled independently across clients $i \in [N]$ and steps $k$, and independently of the active client subset $\mathcal{S}^{(t)}$ (of size $|\mathcal{S}^{(t)}| = M$). We write $\mathbb{E}_{\xi^{(t)}}[\cdot] := \mathbb{E}_{\xi^{(t)}}[\cdot \mid \mathcal{F}_t]$ for the conditional expectation over all local mini-batch randomness $\xi^{(t)} = \{\xi_i^{(t,k)}\}$ at round $t$, and $\mathbb{E}_t[\cdot] := \mathbb{E}_{\mathcal{S}^{(t)}, \xi^{(t)}}[\cdot \mid \mathcal{F}_t]$ for the joint conditional expectation over $\mathcal{S}^{(t)}$ and $\xi^{(t)}$ given $\mathcal{F}_t$.

	Recall from Algorithm~\ref{DeviceUpdateAlgorithm} that each participating client $i \in \mathcal{S}^{(t)}$ returns $\mathbf{g}_i^{(t)} = \frac{1}{\eta_c}(\mathbf{w}^{(t)} - \mathbf{w}_i^{(t,K)}) = \sum_{k=0}^{K-1} \nabla f_i(\mathbf{w}_i^{(t,k)}, \xi_i^{(t,k)})$, and the server forms $\mathbf{G}^{(t)} = \eta_c \mathbf{r}^{(t)}$. To align the scale of the server update with the average per-step gradient, we define the step-normalised quantities:
    
	\begin{align*}
		\Delta_i^{(t)}
		&:= \frac{1}{K}\mathbf{g}_i^{(t)} = \frac{1}{K}\sum_{k=0}^{K-1}
		\nabla f_i\bigl(\mathbf{w}_i^{(t,k)},\,\xi_i^{(t,k)}\bigr),
		\\
		\mathbf{h}_i^{(t)}
		&:= \frac{1}{K}\sum_{k=0}^{K-1}
		\nabla f_i\bigl(\mathbf{w}_i^{(t,k)}\bigr),
		\qquad
		\bar{\mathbf{h}}^{(t)}
		:= \frac{1}{N}\sum_{i=1}^N \mathbf{h}_i^{(t)},
		\\
		\tilde{\mathbf{y}}_j^{(t)}
		&:= \frac{1}{K}\mathbf{y}_j^{(t)},
		\qquad
		\tilde{\mathbf{G}}^{(t)}
		:= \frac{\mathbf{G}^{(t)}}{\eta_c K} = \frac{\mathbf{r}^{(t)}}{K}
		= \frac{1}{M}\sum_{i \in \mathcal{S}^{(t)}}\bigl(\Delta_i^{(t)} - \tilde{\mathbf{y}}_i^{(t)}\bigr) + \frac{1}{N}\sum_{j=1}^N \tilde{\mathbf{y}}_j^{(t)},
		\\
		\tilde{\eta}_s
		&:= \eta_s\,\eta_c\,K.
	\end{align*}
    
	Under the server-side Adagrad optimiser (Algorithm~\ref{alg:AdagOpt}, with $\lambda = 0$), the accumulator at round $t \in \{0, \dots, T-1\}$ is updated via $z_t = z_{t-1} + \mathbf{G}^{(t)} \odot \mathbf{G}^{(t)}$ from initial state $z_{-1} := \mathbf{0} \in \mathbb{R}^d$, and the global model update is given equivalently by
    
	\begin{align}\label{eq:server_step_norm}
		\mathbf{w}^{(t+1)}
		= \mathbf{w}^{(t)} - \eta_s \frac{\mathbf{G}^{(t)}}{\sqrt{z_t} + \epsilon}
		= \mathbf{w}^{(t)} - \tilde{\eta}_s \frac{\tilde{\mathbf{G}}^{(t)}}{\sqrt{z_t} + \epsilon},
	\end{align}
    
	where division and square roots are understood coordinate-wise. Because $\mathbf{w}_i^{(t,k)}$ is $\mathcal{F}_{t,k}$-measurable, Assumption~\ref{ass:boundvar} implies that $\delta_i^{(t,k)} := \nabla f_i(\mathbf{w}_i^{(t,k)}, \xi_i^{(t,k)}) - \nabla f_i(\mathbf{w}_i^{(t,k)})$ forms a martingale difference sequence with respect to $\{\mathcal{F}_{t,k}\}_{k=0}^{K-1}$, satisfying $\mathbb{E}[\delta_i^{(t,k)} \mid \mathcal{F}_{t,k}] = \mathbf{0}$ and $\mathbb{E}[\|\delta_i^{(t,k)}\|^2 \mid \mathcal{F}_{t,k}] \le \sigma^2$. By the tower property of conditional expectation, $\mathbb{E}_{\xi^{(t)}}[\Delta_i^{(t)} - \mathbf{h}_i^{(t)}] = \mathbf{0}$ and, by pairwise orthogonality of martingale differences, $\mathbb{E}_{\xi^{(t)}}\|\Delta_i^{(t)} - \mathbf{h}_i^{(t)}\|^2 = \frac{1}{K^2}\sum_{k=0}^{K-1}\mathbb{E}_{\xi^{(t)}}\|\delta_i^{(t,k)}\|^2 \le \frac{\sigma^2}{K}$. Crucially, because $\mathcal{S}^{(t)}$ is sampled uniformly at random independently of $\xi^{(t)}$ and $\mathcal{F}_t$, the normalised variance-reduced direction $\tilde{\mathbf{G}}^{(t)}$ is conditionally unbiased for $\mathbb{E}_{\xi^{(t)}}[\bar{\mathbf{h}}^{(t)}]$:
	\begin{align}\label{eq:unbiased_Gtilde}
		\mathbb{E}_t\bigl[\tilde{\mathbf{G}}^{(t)}\bigr]
		= \frac{1}{N}\sum_{i=1}^N \bigl(\mathbb{E}_{\xi^{(t)}}[\Delta_i^{(t)}] - \tilde{\mathbf{y}}_i^{(t)}\bigr) + \frac{1}{N}\sum_{j=1}^N \tilde{\mathbf{y}}_j^{(t)}
		= \mathbb{E}_{\xi^{(t)}}\bigl[\bar{\mathbf{h}}^{(t)}\bigr].
	\end{align}
	
	\subsection{Auxiliary Lemmas and Full Proofs}\label{Apd:AuxiliaryLemmas}
	
	\begin{lemma}[Young's inequality]\label{lem1}
		For any $\mathbf{u},\mathbf{v}\in\mathbb{R}^d$ and any scalar $\gamma > 0$,
		$\langle \mathbf{u},\mathbf{v}\rangle \le \frac{\|\mathbf{u}\|^2}{2\gamma} + \frac{\gamma\,\|\mathbf{v}\|^2}{2}$ and $\|\mathbf{u} + \mathbf{v}\|^2 \le (1 + \frac{1}{\gamma})\|\mathbf{u}\|^2 + (1 + \gamma)\|\mathbf{v}\|^2$.
	\end{lemma}
	
	\begin{lemma}[Jensen's inequality]\label{lem2}
		For any convex function $\phi : \mathbb{R}^d \to \mathbb{R}$ and integrable random vector $X$,
		$\phi\bigl(\mathbb{E}[X]\bigr) \le \mathbb{E}\bigl[\phi(X)\bigr]$.
	\end{lemma}
	
	\begin{lemma}[Sum of squares]\label{lem3}
		For any positive integer $K$ and vectors $\mathbf{x}_1,\dots,\mathbf{x}_K\in\mathbb{R}^d$,
		$\bigl\|\sum_{k=1}^K \mathbf{x}_k\bigr\|^2 \le K \sum_{k=1}^K \|\mathbf{x}_k\|^2$.
	\end{lemma}
	
	\begin{lemma}[Variance under uniform without-replacement sampling]\label{lem4}
		Let $\bar{\mathbf{x}} = \frac{1}{N} \sum_{i=1}^{N} \mathbf{x}_i$. If $\mathcal{S}$ is a subset of size $M$ sampled uniformly at random without replacement from $[N]$, then $\mathbb{E}_{\mathcal{S}}\bigl[\frac{1}{M}\sum_{i\in\mathcal{S}}\mathbf{x}_i\bigr] = \bar{\mathbf{x}}$ and
		\begin{align*}
			\mathbb{E}_{\mathcal{S}}\Bigl\|\frac{1}{M}\sum_{i\in\mathcal{S}}\mathbf{x}_i - \bar{\mathbf{x}}\Bigr\|^2
			= \frac{N-M}{M(N-1)}\,\frac{1}{N}\sum_{i=1}^N\|\mathbf{x}_i - \bar{\mathbf{x}}\|^2
			\le \frac{N-M}{M(N-1)}\,\frac{1}{N}\sum_{i=1}^N\|\mathbf{x}_i\|^2.
		\end{align*}
	\end{lemma}
	
	\begin{lemma}[Client drift bound]\label{lem:5}
		Suppose the local functions $\{f_i\}$ satisfy Assumptions~\ref{ass:lsmooth} and~\ref{ass:boundvar}, and let the client learning rate satisfy $\eta_c \le \frac{1}{2LK}$. Then the iterates $\{\mathbf{w}^{(t)}\}$ generated by \textsc{FedAdaVR} satisfy
		\begin{align*}
			\bigl\|\nabla f(\mathbf{w}^{(t)}) - \mathbb{E}_{\xi^{(t)}}[\bar{\mathbf{h}}^{(t)}]\bigr\|^2\\[4pt]
			&\le \mathbb{E}_{\xi^{(t)}}\bigl\|\nabla f(\mathbf{w}^{(t)}) - \bar{\mathbf{h}}^{(t)}\bigr\|^2\\[4pt]
			&\leq C^{(t)} := \frac{1}{N}\sum_{i=1}^N \mathbb{E}_{\xi^{(t)}}\bigl\|\nabla f_i(\mathbf{w}^{(t)}) - \mathbf{h}_i^{(t)}\bigr\|^2 \\[4pt]
			&\leq 2\eta_c^2L^2(K-1)\sigma^2 + 8\eta_c^2L^2K(K-1)\Bigl[\sigma_g^2 + \|\nabla f(\mathbf{w}^{(t)})\|^2\Bigr].
		\end{align*}
	\end{lemma}
	\begin{proof}
		By Lemma~\ref{lem2} (convexity of $\|\cdot\|^2$) and $\nabla f(\mathbf{w}^{(t)}) = \frac{1}{N}\sum_{i=1}^N \nabla f_i(\mathbf{w}^{(t)})$, the first two inequalities hold immediately. For each client $i \in [N]$, by Lemma~\ref{lem3} and Assumption~\ref{ass:lsmooth},
        \begin{align}\label{eq:lem5_step1}
            \mathbb{E}_{\xi^{(t)}}\bigl\|\nabla f_i(\mathbf{w}^{(t)}) - \mathbf{h}_i^{(t)}\bigr\|^2 
            &= \mathbb{E}_{\xi^{(t)}}\Bigl\|\frac{1}{K}\sum_{k=0}^{K-1}\bigl(\nabla f_i(\mathbf{w}^{(t)}) - \nabla f_i(\mathbf{w}_i^{(t,k)})\bigr)\Bigr\|^2 \nonumber \\
            &\le \frac{L^2}{K}\sum_{k=0}^{K-1}\mathbb{E}_{\xi^{(t)}}\bigl\|\mathbf{w}^{(t)} - \mathbf{w}_i^{(t,k)}\bigr\|^2.
        \end{align}

		For $k = 0$, $\mathbf{w}_i^{(t,0)} = \mathbf{w}^{(t)}$. For $k \in \{1, \dots, K-1\}$, unrolling the local SGD update $\mathbf{w}_i^{(t,k)} = \mathbf{w}^{(t)} - \eta_c \sum_{\tau=0}^{k-1} \nabla f_i(\mathbf{w}_i^{(t,\tau)}, \xi_i^{(t,\tau)})$, separating the zero-mean martingale difference noise $\delta_i^{(t,\tau)}$ from the full local gradient, and using conditional independence across steps $\tau$ together with Assumption~\ref{ass:boundvar} and Lemma~\ref{lem3} gives
        \begin{align*}
            \mathbb{E}_{\xi^{(t)}}\bigl\|\mathbf{w}_i^{(t,k)} - \mathbf{w}^{(t)}\bigr\|^2
            &\le 2\eta_c^2 \mathbb{E}_{\xi^{(t)}}\Bigl\|\sum_{\tau=0}^{k-1}\bigl(\nabla f_i(\mathbf{w}_i^{(t,\tau)}, \xi_i^{(t,\tau)}) - \nabla f_i(\mathbf{w}_i^{(t,\tau)})\bigr)\Bigr\|^2 \\
            &\quad + 2\eta_c^2 \mathbb{E}_{\xi^{(t)}}\Bigl\|\sum_{\tau=0}^{k-1}\nabla f_i(\mathbf{w}_i^{(t,\tau)})\Bigr\|^2 \\
            &\le 2\eta_c^2 k \sigma^2 + 2\eta_c^2 k \sum_{\tau=0}^{k-1}\mathbb{E}_{\xi^{(t)}}\bigl\|\nabla f_i(\mathbf{w}_i^{(t,\tau)})\bigr\|^2.
        \end{align*}

		By Lemma~\ref{lem3} and Assumptions~\ref{ass:lsmooth} and~\ref{ass:boundvar}, for every $\tau \ge 0$,
		\begin{align*}
			\bigl\|\nabla f_i(\mathbf{w}_i^{(t,\tau)})\bigr\|^2
			&\le 2\bigl\|\nabla f_i(\mathbf{w}_i^{(t,\tau)}) - \nabla f_i(\mathbf{w}^{(t)})\bigr\|^2 + 2\bigl\|\nabla f_i(\mathbf{w}^{(t)})\bigr\|^2 \\
			&\le 2L^2\bigl\|\mathbf{w}_i^{(t,\tau)} - \mathbf{w}^{(t)}\bigr\|^2 + 4\bigl(\sigma_g^2 + \|\nabla f(\mathbf{w}^{(t)})\|^2\bigr).
		\end{align*}
		Substituting this bound back and averaging over $k = 0, \dots, K-1$ (using $\sum_{k=1}^{K-1} k = \frac{K(K-1)}{2}$ and $\sum_{k=1}^{K-1} k^2 \le \frac{K^2(K-1)}{2}$) yields
		\begin{align*}
			\frac{1}{K}\sum_{k=0}^{K-1}\mathbb{E}_{\xi^{(t)}}\bigl\|\mathbf{w}_i^{(t,k)} - \mathbf{w}^{(t)}\bigr\|^2
			&\le \eta_c^2(K-1)\sigma^2 \\
			&\quad + 2\eta_c^2 L^2 K(K-1)\Biggl(\frac{1}{K}\sum_{\tau=0}^{K-1}\mathbb{E}_{\xi^{(t)}}\bigl\|\mathbf{w}_i^{(t,\tau)} - \mathbf{w}^{(t)}\bigr\|^2\Biggr) \\
			&\quad + 4\eta_c^2 K(K-1)\bigl(\sigma_g^2 + \|\nabla f(\mathbf{w}^{(t)})\|^2\bigr).
		\end{align*}
		Since $\eta_c \le \frac{1}{2LK}$, we have $2\eta_c^2 L^2 K(K-1) \le \frac{1}{2}$. Rearranging and multiplying both sides by $2L^2$, then substituting into~\eqref{eq:lem5_step1} and averaging over $i \in [N]$ completes the proof.
	\end{proof}
	
	\begin{lemma}[Variance of the SAGA-style server estimator]\label{lem:6}
		Suppose the local functions $\{f_i\}$ satisfy Assumptions~\ref{ass:lsmooth} and~\ref{ass:boundvar}. With the boundary convention $\mathbf{w}^{(-1)} := \mathbf{w}^{(0)}$ and $\tilde{\mathbf{G}}^{(-1)} := \mathbf{0}$, for every round $t \ge 0$, the iterates of \textsc{FedAdaVR} satisfy
		\begin{align*}
			\mathbb{E}_t\bigl\|\tilde{\mathbf{G}}^{(t)} - \bar{\mathbf{h}}^{(t)}\bigr\|^2
			&\le \frac{2\sigma^2}{MK}
			+ \frac{5(N-M)}{M(N-1)}\,C^{(t)}
			+ \frac{4(N-M)}{M(N-1)}\frac{\tilde{\eta}_s^2 L^2}{\epsilon^2}\bigl\|\tilde{\mathbf{G}}^{(t-1)}\bigr\|^2 \\
			&\quad + \frac{2(N-M)}{M(N-1)}\,\frac{1}{N}\sum_{i=1}^N \bigl\|\nabla f_i(\mathbf{w}^{(t-1)}) - \tilde{\mathbf{y}}_i^{(t)}\bigr\|^2.
		\end{align*}
	\end{lemma}
	\begin{proof}
		By definition of $\tilde{\mathbf{G}}^{(t)}$ and $\bar{\mathbf{h}}^{(t)}$, we decompose the estimation error as
		\begin{align*}
			\tilde{\mathbf{G}}^{(t)} - \bar{\mathbf{h}}^{(t)}
			= \underbrace{\frac{1}{M}\sum_{i \in \mathcal{S}^{(t)}}\bigl(\Delta_i^{(t)} - \mathbf{h}_i^{(t)}\bigr)}_{U_1}
			+ \underbrace{\frac{1}{M}\sum_{i \in \mathcal{S}^{(t)}}\bigl(\mathbf{h}_i^{(t)} - \tilde{\mathbf{y}}_i^{(t)}\bigr) - \frac{1}{N}\sum_{j=1}^N\bigl(\mathbf{h}_j^{(t)} - \tilde{\mathbf{y}}_j^{(t)}\bigr)}_{U_2}.
		\end{align*}
		To decouple the local-trajectory randomness in $U_2$ from the $\mathcal{F}_t$-measurable historical states, we further split $U_2 = U_{2,a} + U_{2,b}$, where
		\begin{align*}
			U_{2,a} &:= \frac{1}{M}\sum_{i \in \mathcal{S}^{(t)}}\bigl(\mathbf{h}_i^{(t)} - \nabla f_i(\mathbf{w}^{(t)})\bigr) - \frac{1}{N}\sum_{j=1}^N\bigl(\mathbf{h}_j^{(t)} - \nabla f_j(\mathbf{w}^{(t)})\bigr), \\
			U_{2,b} &:= \frac{1}{M}\sum_{i \in \mathcal{S}^{(t)}}\bigl(\nabla f_i(\mathbf{w}^{(t)}) - \tilde{\mathbf{y}}_i^{(t)}\bigr) - \frac{1}{N}\sum_{j=1}^N\bigl(\nabla f_j(\mathbf{w}^{(t)}) - \tilde{\mathbf{y}}_j^{(t)}\bigr).
		\end{align*}
		Conditioned on $\mathcal{F}_t$ and $\mathcal{S}^{(t)}$, $U_{2,b}$ is deterministic and $\mathbb{E}_{\xi^{(t)}}[U_1 \mid \mathcal{F}_t, \mathcal{S}^{(t)}] = \mathbf{0}$, so $\mathbb{E}_t[\langle U_1, U_{2,b} \rangle] = 0$. For the remaining cross-term, Lemma~\ref{lem1} ($\gamma = 1$) gives $2\mathbb{E}_t[\langle U_1, U_{2,a} \rangle] \le \mathbb{E}_t\|U_1\|^2 + \mathbb{E}_t\|U_{2,a}\|^2$, and thus
		\begin{align*}
			\mathbb{E}_t\bigl\|\tilde{\mathbf{G}}^{(t)} - \bar{\mathbf{h}}^{(t)}\bigr\|^2
			&= \mathbb{E}_t\|U_1\|^2 + 2\mathbb{E}_t[\langle U_1, U_{2,b} \rangle] + 2\mathbb{E}_t[\langle U_1, U_{2,a} \rangle] + \mathbb{E}_t\|U_2\|^2 \\
			&\le 2\mathbb{E}_t\|U_1\|^2 + \mathbb{E}_t\|U_{2,a}\|^2 + \mathbb{E}_t\|U_2\|^2.
		\end{align*}
		By conditional independence of mini-batch sampling across clients and steps together with Assumption~\ref{ass:boundvar}, $\mathbb{E}_t\|U_1\|^2 \le \frac{\sigma^2}{MK}$. Applying Lemma~\ref{lem4} to $U_{2,a}$ and $U_2$ yields $\mathbb{E}_t\|U_{2,a}\|^2 \le \frac{N-M}{M(N-1)}\,C^{(t)}$ and
		\begin{align*}
			\mathbb{E}_t\|U_2\|^2
			\le \frac{N-M}{M(N-1)}\,\frac{1}{N}\sum_{i=1}^N \mathbb{E}_{\xi^{(t)}}\bigl\|\mathbf{h}_i^{(t)} - \tilde{\mathbf{y}}_i^{(t)}\bigr\|^2.
		\end{align*}
		Writing $\mathbf{h}_i^{(t)} - \tilde{\mathbf{y}}_i^{(t)} = (\mathbf{h}_i^{(t)} - \nabla f_i(\mathbf{w}^{(t)})) + (\nabla f_i(\mathbf{w}^{(t)}) - \nabla f_i(\mathbf{w}^{(t-1)})) + (\nabla f_i(\mathbf{w}^{(t-1)}) - \tilde{\mathbf{y}}_i^{(t)})$, applying Lemma~\ref{lem1} ($\gamma = 1$) twice ($\|\mathbf{a}+\mathbf{b}+\mathbf{c}\|^2 \le 4\|\mathbf{a}\|^2 + 4\|\mathbf{b}\|^2 + 2\|\mathbf{c}\|^2$), and using Assumption~\ref{ass:lsmooth} together with~\eqref{eq:server_step_norm} and $\sqrt{z_{t-1,j}} + \epsilon \ge \epsilon$ (or $\mathbf{w}^{(0)} - \mathbf{w}^{(-1)} = \mathbf{0}$ at $t = 0$) gives
		\begin{align*}
			\bigl\|\nabla f_i(\mathbf{w}^{(t)}) - \nabla f_i(\mathbf{w}^{(t-1)})\bigr\|^2
			\le L^2\bigl\|\mathbf{w}^{(t)} - \mathbf{w}^{(t-1)}\bigr\|^2
			= \tilde{\eta}_s^2 L^2\Bigl\|\frac{\tilde{\mathbf{G}}^{(t-1)}}{\sqrt{z_{t-1}} + \epsilon}\Bigr\|^2
			\le \frac{\tilde{\eta}_s^2 L^2}{\epsilon^2}\bigl\|\tilde{\mathbf{G}}^{(t-1)}\bigr\|^2.
		\end{align*}
		Combining the bounds on $2\mathbb{E}_t\|U_1\|^2$, $\mathbb{E}_t\|U_{2,a}\|^2$, and $\mathbb{E}_t\|U_2\|^2$ establishes the claim.
	\end{proof}
	
	\begin{lemma}[Staleness recursion for stored client states]\label{lem:7}
		Suppose the local functions $\{f_i\}$ satisfy Assumptions~\ref{ass:lsmooth} and~\ref{ass:boundvar}. With the boundary convention $\mathbf{w}^{(-1)} := \mathbf{w}^{(0)}$ and $\tilde{\mathbf{G}}^{(-1)} := \mathbf{0}$, for any $\beta > 0$ and $t \ge 0$, the iterates of \textsc{FedAdaVR} satisfy
		\begin{align*}
			\mathbb{E}_t\Biggl[\frac{1}{N}\sum_{j=1}^N \bigl\|\nabla f_j(\mathbf{w}^{(t)}) - \tilde{\mathbf{y}}_j^{(t+1)}\bigr\|^2\Biggr]
			&\le \frac{2M}{N}\Biggl[\frac{\sigma^2}{K} + C^{(t)}\Biggr] \\
			&\quad + \Bigl(1 - \frac{M}{N}\Bigr)(1+\beta)\,\frac{1}{N}\sum_{j=1}^N \bigl\|\nabla f_j(\mathbf{w}^{(t-1)}) - \tilde{\mathbf{y}}_j^{(t)}\bigr\|^2 \\
			&\quad + \Bigl(1 - \frac{M}{N}\Bigr)\Bigl(1 + \frac{1}{\beta}\Bigr)\frac{\tilde{\eta}_s^2 L^2}{\epsilon^2}\bigl\|\tilde{\mathbf{G}}^{(t-1)}\bigr\|^2.
		\end{align*}
	\end{lemma}
	\begin{proof}
		By the memory update rule~\eqref{MemoryUpdateEquation}, each client $j \in [N]$ is included in $\mathcal{S}^{(t)}$ with probability $\frac{M}{N}$ (setting $\tilde{\mathbf{y}}_j^{(t+1)} = \Delta_j^{(t)}$) and excluded with probability $1 - \frac{M}{N}$ (retaining $\tilde{\mathbf{y}}_j^{(t+1)} = \tilde{\mathbf{y}}_j^{(t)}$). Hence,
		\begin{align*}
			\mathbb{E}_t\Biggl[\frac{1}{N}\sum_{j=1}^N \bigl\|\nabla f_j(\mathbf{w}^{(t)}) - \tilde{\mathbf{y}}_j^{(t+1)}\bigr\|^2\Biggr]
			&= \frac{M}{N}\,\frac{1}{N}\sum_{j=1}^N \mathbb{E}_{\xi^{(t)}}\bigl\|\nabla f_j(\mathbf{w}^{(t)}) - \Delta_j^{(t)}\bigr\|^2 \\
			&\quad + \Bigl(1 - \frac{M}{N}\Bigr)\frac{1}{N}\sum_{j=1}^N \bigl\|\nabla f_j(\mathbf{w}^{(t)}) - \tilde{\mathbf{y}}_j^{(t)}\bigr\|^2.
		\end{align*}
		For the first sum, writing $\nabla f_j(\mathbf{w}^{(t)}) - \Delta_j^{(t)} = (\nabla f_j(\mathbf{w}^{(t)}) - \mathbf{h}_j^{(t)}) + (\mathbf{h}_j^{(t)} - \Delta_j^{(t)})$ and applying Lemma~\ref{lem1} ($\gamma = 1$) together with $\mathbb{E}_{\xi^{(t)}}\|\Delta_j^{(t)} - \mathbf{h}_j^{(t)}\|^2 \le \frac{\sigma^2}{K}$ gives $\mathbb{E}_{\xi^{(t)}}\|\nabla f_j(\mathbf{w}^{(t)}) - \Delta_j^{(t)}\|^2 \le 2\mathbb{E}_{\xi^{(t)}}\|\Delta_j^{(t)} - \mathbf{h}_j^{(t)}\|^2 + 2\mathbb{E}_{\xi^{(t)}}\|\nabla f_j(\mathbf{w}^{(t)}) - \mathbf{h}_j^{(t)}\|^2 \le \frac{2\sigma^2}{K} + 2\mathbb{E}_{\xi^{(t)}}\|\nabla f_j(\mathbf{w}^{(t)}) - \mathbf{h}_j^{(t)}\|^2$. For the second sum, applying Lemma~\ref{lem1} with parameter $\beta > 0$ to $\nabla f_j(\mathbf{w}^{(t)}) - \tilde{\mathbf{y}}_j^{(t)} = (\nabla f_j(\mathbf{w}^{(t)}) - \nabla f_j(\mathbf{w}^{(t-1)})) + (\nabla f_j(\mathbf{w}^{(t-1)}) - \tilde{\mathbf{y}}_j^{(t)})$ and using $\|\nabla f_j(\mathbf{w}^{(t)}) - \nabla f_j(\mathbf{w}^{(t-1)})\|^2 \le \frac{\tilde{\eta}_s^2 L^2}{\epsilon^2}\|\tilde{\mathbf{G}}^{(t-1)}\|^2$ yields the result.
	\end{proof}

	\begin{lemma}[Uniform bound on the Adagrad accumulator]\label{lem:8}
		Suppose Assumption~\ref{ass:bg} holds. Then for all $t \in \{0, \dots, T-1\}$ and all coordinates $j \in [d]$,
		\begin{align*}
			\bigl\|\tilde{\mathbf{G}}^{(t)}\bigr\| \le 3G, \qquad
			\bigl\|\mathbf{G}^{(t)}\bigr\| \le 3\eta_c K G, \qquad
			\frac{1}{D_T} \le \frac{1}{\sqrt{z_{t-1, j}} + \epsilon} \le \frac{1}{\epsilon},
		\end{align*}
		where $D_T := 3\eta_c K G \sqrt{T} + \epsilon$.
	\end{lemma}
	\begin{proof}
		Under Assumption~\ref{ass:bg}, $\|\nabla f_i(\mathbf{w}_i^{(t,k)}, \xi_i^{(t,k)})\| \le G$, so by Lemma~\ref{lem2} (convexity of $\|\cdot\|$), $\|\Delta_i^{(t)}\| \le \frac{1}{K}\sum_{k=0}^{K-1}\|\nabla f_i(\mathbf{w}_i^{(t,k)}, \xi_i^{(t,k)})\| \le G$. Since $\tilde{\mathbf{y}}_j^{(0)} = \mathbf{0}$ and $\tilde{\mathbf{y}}_j^{(t)}$ stores the most recent $\Delta_j^{(s)}$ ($s < t$), induction on $t$ gives $\|\tilde{\mathbf{y}}_j^{(t)}\| \le G$ for all $j \in [N]$ and $t \ge 0$. Because $\frac{1}{M}\sum_{i \in \mathcal{S}^{(t)}}$ and $\frac{1}{N}\sum_{j=1}^N$ are convex combinations, the triangle inequality yields
		\begin{align*}
			\bigl\|\tilde{\mathbf{G}}^{(t)}\bigr\|
			\le \frac{1}{M}\sum_{i \in \mathcal{S}^{(t)}}\bigl(\|\Delta_i^{(t)}\| + \|\tilde{\mathbf{y}}_i^{(t)}\|\bigr) + \frac{1}{N}\sum_{j=1}^N \|\tilde{\mathbf{y}}_j^{(t)}\|
			\le 2G + G = 3G.
		\end{align*}
		Consequently, $\|\mathbf{G}^{(t)}\| = \eta_c K \|\tilde{\mathbf{G}}^{(t)}\| \le 3\eta_c K G$, which implies $|G_j^{(s)}| \le \|\mathbf{G}^{(s)}\| \le 3\eta_c K G$ for every coordinate $j \in [d]$ and round $s$. Summing over $s = 0, \dots, t-1$ (with $t \le T$), the non-negative accumulator satisfies $0 \le z_{t-1,j} = \sum_{s=0}^{t-1} (G_j^{(s)})^2 \le 9t\eta_c^2 K^2 G^2 \le 9T\eta_c^2 K^2 G^2$. Taking square roots and adding $\epsilon > 0$ gives $\epsilon \le \sqrt{z_{t-1,j}} + \epsilon \le 3\eta_c K G \sqrt{T} + \epsilon = D_T$, which inverts to the stated inequality.
	\end{proof}
	
	\subsection{Proof of Theorem~\ref{theo:con} and Corollary~\ref{cor:rate}}
	
	\begin{proof}[Proof of Theorem~\ref{theo:con}]
		By Assumption~\ref{ass:lsmooth} ($L$-smoothness of $f = \frac{1}{N}\sum_{i=1}^N f_i$) and the server update~\eqref{eq:server_step_norm}, taking conditional expectation $\mathbb{E}_t[\cdot]$ gives
		\begin{align}\label{eq:descent_start}
			\mathbb{E}_t\bigl[f(\mathbf{w}^{(t+1)})\bigr]
			&\le f(\mathbf{w}^{(t)})
			+ \underbrace{\tilde{\eta}_s \Bigl\langle \nabla f(\mathbf{w}^{(t)}),\, \mathbb{E}_t\Bigl[-\frac{\tilde{\mathbf{G}}^{(t)}}{\sqrt{z_{t-1}} + \epsilon}\Bigr]\Bigr\rangle}_{T_1^{(t)}} \nonumber\\
			&\quad + \underbrace{\tilde{\eta}_s \Bigl\langle \nabla f(\mathbf{w}^{(t)}),\, \mathbb{E}_t\Bigl[\frac{\tilde{\mathbf{G}}^{(t)}}{\sqrt{z_{t-1}} + \epsilon} - \frac{\tilde{\mathbf{G}}^{(t)}}{\sqrt{z_t} + \epsilon}\Bigr]\Bigr\rangle}_{T_2^{(t)}}
			+ \underbrace{\frac{\tilde{\eta}_s^2 L}{2}\,\mathbb{E}_t\Bigl\|\frac{\tilde{\mathbf{G}}^{(t)}}{\sqrt{z_t} + \epsilon}\Bigr\|^2}_{T_3^{(t)}}.
		\end{align}

		\paragraph{Step 1: Bounding the first-order term $T_1^{(t)}$.}
		Because $z_{t-1}$ is $\mathcal{F}_t$-measurable, property~\eqref{eq:unbiased_Gtilde} yields $\mathbb{E}_t\bigl[\frac{\tilde{\mathbf{G}}^{(t)}}{\sqrt{z_{t-1}} + \epsilon}\bigr] = \frac{\mathbb{E}_{\xi^{(t)}}[\bar{\mathbf{h}}^{(t)}]}{\sqrt{z_{t-1}} + \epsilon}$. Writing $-\mathbb{E}_{\xi^{(t)}}[\bar{\mathbf{h}}^{(t)}] = -\nabla f(\mathbf{w}^{(t)}) + (\nabla f(\mathbf{w}^{(t)}) - \mathbb{E}_{\xi^{(t)}}[\bar{\mathbf{h}}^{(t)}])$ and applying Lemma~\ref{lem1} (Young's inequality with $\gamma = 1$) coordinate-wise while retaining the positive weight $\frac{1}{\sqrt{z_{t-1,j}} + \epsilon}$ gives
		\begin{align}
			T_1^{(t)}
			&= -\tilde{\eta}_s \sum_{j=1}^d \frac{([\nabla f(\mathbf{w}^{(t)})]_j)^2}{\sqrt{z_{t-1,j}} + \epsilon}
			+ \tilde{\eta}_s \sum_{j=1}^d \frac{[\nabla f(\mathbf{w}^{(t)})]_j\,\bigl[\nabla f(\mathbf{w}^{(t)}) - \mathbb{E}_{\xi^{(t)}}[\bar{\mathbf{h}}^{(t)}]\bigr]_j}{\sqrt{z_{t-1,j}} + \epsilon} \nonumber\\
			&\le -\frac{\tilde{\eta}_s}{2}\sum_{j=1}^d \frac{([\nabla f(\mathbf{w}^{(t)})]_j)^2}{\sqrt{z_{t-1,j}} + \epsilon}
			+ \frac{\tilde{\eta}_s}{2}\sum_{j=1}^d \frac{\bigl([\nabla f(\mathbf{w}^{(t)}) - \mathbb{E}_{\xi^{(t)}}[\bar{\mathbf{h}}^{(t)}]]_j\bigr)^2}{\sqrt{z_{t-1,j}} + \epsilon} \nonumber\\
			&\le -\frac{\tilde{\eta}_s}{2 D_T}\bigl\|\nabla f(\mathbf{w}^{(t)})\bigr\|^2
			+ \frac{\tilde{\eta}_s}{2\epsilon}\,C^{(t)}, \label{eq:T1_bound}
		\end{align}
		where the last step uses Lemma~\ref{lem:8} ($\frac{1}{D_T} \le \frac{1}{\sqrt{z_{t-1,j}}+\epsilon} \le \frac{1}{\epsilon}$) and Lemma~\ref{lem:5} ($\|\nabla f(\mathbf{w}^{(t)}) - \mathbb{E}_{\xi^{(t)}}[\bar{\mathbf{h}}^{(t)}]\|^2 \le C^{(t)}$).

		\paragraph{Step 2: Telescoping the preconditioner-drift term $T_2^{(t)}$.}
		Since $z_{t,j} = z_{t-1,j} + (G_j^{(t)})^2 \ge z_{t-1,j}$, the difference $\frac{1}{\sqrt{z_{t-1,j}}+\epsilon} - \frac{1}{\sqrt{z_{t,j}}+\epsilon}$ is non-negative. By Assumption~\ref{ass:bg} and Lemma~\ref{lem:8}, $|[\nabla f(\mathbf{w}^{(t)})]_j| \le G$ and $|\tilde{G}_j^{(t)}| \le 3G$, so
		\begin{align*}
			T_2^{(t)}
			&= \tilde{\eta}_s \sum_{j=1}^d \mathbb{E}_t\Biggl[[\nabla f(\mathbf{w}^{(t)})]_j \tilde{G}_j^{(t)} \Bigl(\frac{1}{\sqrt{z_{t-1,j}} + \epsilon} - \frac{1}{\sqrt{z_{t,j}} + \epsilon}\Bigr)\Biggr] \\
			&\le 3\tilde{\eta}_s G^2 \sum_{j=1}^d \mathbb{E}_t\Bigl[\frac{1}{\sqrt{z_{t-1,j}} + \epsilon} - \frac{1}{\sqrt{z_{t,j}} + \epsilon}\Bigr].
		\end{align*}
		Taking full expectation and summing over $t = 0, \dots, T-1$ telescopes the series (since $z_{-1,j} = 0$):
		\begin{align}\label{eq:T2_telescope}
			\sum_{t=0}^{T-1}\mathbb{E}\bigl[T_2^{(t)}\bigr]
			\le 3\tilde{\eta}_s G^2 \sum_{j=1}^d \mathbb{E}\Bigl[\frac{1}{\epsilon} - \frac{1}{\sqrt{z_{T-1,j}} + \epsilon}\Bigr]
			\le \frac{3\tilde{\eta}_s d G^2}{\epsilon}.
		\end{align}

		\paragraph{Step 3: Bounding the second-order term $T_3^{(t)}$ via Lyapunov analysis.}
		Since $\sqrt{z_{t,j}} + \epsilon \ge \epsilon$, using Lemma~\ref{lem1} ($\gamma = 1$) twice to expand $\tilde{\mathbf{G}}^{(t)} = (\tilde{\mathbf{G}}^{(t)} - \bar{\mathbf{h}}^{(t)}) + (\bar{\mathbf{h}}^{(t)} - \nabla f(\mathbf{w}^{(t)})) + \nabla f(\mathbf{w}^{(t)})$ gives
		\begin{align}\label{eq:T3_split}
			T_3^{(t)}
			\le \frac{\tilde{\eta}_s^2 L}{2\epsilon^2}\mathbb{E}_t\bigl\|\tilde{\mathbf{G}}^{(t)}\bigr\|^2
			\le \frac{\tilde{\eta}_s^2 L}{\epsilon^2}\mathbb{E}_t\bigl\|\tilde{\mathbf{G}}^{(t)} - \bar{\mathbf{h}}^{(t)}\bigr\|^2
			+ \frac{2\tilde{\eta}_s^2 L}{\epsilon^2}\,C^{(t)}
			+ \frac{2\tilde{\eta}_s^2 L}{\epsilon^2}\bigl\|\nabla f(\mathbf{w}^{(t)})\bigr\|^2.
		\end{align}
		Define the historical staleness functional $Y^{(t)} := \frac{1}{N}\sum_{j=1}^N \|\nabla f_j(\mathbf{w}^{(t-1)}) - \tilde{\mathbf{y}}_j^{(t)}\|^2$ (with $\mathbf{w}^{(-1)} := \mathbf{w}^{(0)}$ and $\tilde{\mathbf{G}}^{(-1)} := \mathbf{0}$, so $Y^{(0)} \le G^2$). In Lemma~\ref{lem:7}, choose $\beta = \frac{M}{2(N-M)}$ (for $M < N$; if $M = N$, $Y^{(t)}$ vanishes from Lemma~\ref{lem:6} and we set $c_y := 0$), which gives $(1 - \frac{M}{N})(1+\beta) = 1 - \frac{M}{2N}$ and $(1 - \frac{M}{N})(1 + \frac{1}{\beta}) = \frac{N-M}{N}\frac{2N-M}{M} \le \frac{2(N-M)}{M}$. Defining the Lyapunov weight $c_y := \frac{\tilde{\eta}_s^2 L}{\epsilon^2}\frac{4N(N-M)}{M^2(N-1)}$ so that $c_y \frac{M}{2N} = \frac{\tilde{\eta}_s^2 L}{\epsilon^2}\frac{2(N-M)}{M(N-1)}$, combining~\eqref{eq:T3_split} with Lemma~\ref{lem:6} and $c_y(\mathbb{E}_t[Y^{(t+1)}] - Y^{(t)})$ from Lemma~\ref{lem:7} cancels the $Y^{(t)}$ term and yields
		\begin{align*}
			T_3^{(t)} + c_y\bigl(\mathbb{E}_t[Y^{(t+1)}] - Y^{(t)}\bigr)
			&\le \frac{\tilde{\eta}_s^2 L}{\epsilon^2}\Bigl(2 + \frac{13(N-M)}{M(N-1)}\Bigr)C^{(t)}
			+ \frac{2\tilde{\eta}_s^2 L}{\epsilon^2}\bigl\|\nabla f(\mathbf{w}^{(t)})\bigr\|^2
			+ \frac{2\tilde{\eta}_s^2 L B}{\epsilon^2 M K}\sigma^2 \\
			&\quad + \frac{4\tilde{\eta}_s^4 L^3 (N-M)}{\epsilon^4 M(N-1)}\Bigl(1 + \frac{2N(N-M)}{M^2}\Bigr)\bigl\|\tilde{\mathbf{G}}^{(t-1)}\bigr\|^2,
		\end{align*}
		where $A = 1 + \frac{13(N-M)}{2M(N-1)}$ collects the $C^{(t)}$ coefficients ($2 + \frac{5(N-M)}{M(N-1)} + c_y\frac{\epsilon^2}{\tilde{\eta}_s^2 L}\frac{2M}{N} = 2 + \frac{13(N-M)}{M(N-1)} = 2A$), and $B = 1 + \frac{4(N-M)}{N-1}$ collects the direct noise terms, since $\frac{\tilde{\eta}_s^2 L}{\epsilon^2}\frac{2\sigma^2}{MK} + c_y\frac{2M}{N}\frac{\sigma^2}{K} = \frac{2\tilde{\eta}_s^2 L}{\epsilon^2 MK}\bigl(1 + \frac{4(N-M)}{N-1}\bigr)\sigma^2$. We also used $\frac{4N(N-M)}{M^2(N-1)}\cdot\frac{2(N-M)}{M} = \frac{4(N-M)}{M(N-1)}\cdot\frac{2N(N-M)}{M^2}$. By Lemma~\ref{lem:8}, $\|\tilde{\mathbf{G}}^{(t-1)}\|^2 \le 9G^2$, so the last term is bounded by $\bar{A}_3$. Moreover, since $\tilde{\eta}_s \le \frac{\epsilon}{4L}$, we have $\frac{2\tilde{\eta}_s^2 L}{\epsilon^2}(1 + \frac{13(N-M)}{2M(N-1)}) \le \frac{\tilde{\eta}_s}{2\epsilon} A$.

		\paragraph{Step 4: Combining terms and telescoping over $t = 0, \dots, T-1$.}
		Defining the Lyapunov potential $\Phi^{(t)} := f(\mathbf{w}^{(t)}) + c_y Y^{(t)}$ and combining~\eqref{eq:descent_start}, \eqref{eq:T1_bound} (using $\frac{\tilde{\eta}_s}{2\epsilon}C^{(t)} \le \frac{\tilde{\eta}_s A}{2\epsilon}C^{(t)}$), and Step~3 gives
        \begin{align*}
            \mathbb{E}_t\bigl[\Phi^{(t+1)}\bigr] - \Phi^{(t)}
            &\le -\frac{\tilde{\eta}_s}{2 D_T}\bigl\|\nabla f(\mathbf{w}^{(t)})\bigr\|^2
            + \frac{2\tilde{\eta}_s^2 L}{\epsilon^2}\bigl\|\nabla f(\mathbf{w}^{(t)})\bigr\|^2 \\
            &\quad + \frac{\tilde{\eta}_s A}{\epsilon}\,C^{(t)}
            + \frac{2\tilde{\eta}_s^2 L B}{\epsilon^2 M K}\sigma^2 + \bar{A}_3 + T_2^{(t)}.
        \end{align*}

		Substituting the bound on $C^{(t)}$ from Lemma~\ref{lem:5} produces
		\begin{align*}
			\mathbb{E}_t\bigl[\Phi^{(t+1)}\bigr] - \Phi^{(t)}
			&\le -\tilde{\eta}_s\Biggl[\frac{1}{2 D_T} - \frac{2\tilde{\eta}_s L}{\epsilon^2} - \frac{8 A \eta_c^2 L^2 K(K-1)}{\epsilon}\Biggr]\bigl\|\nabla f(\mathbf{w}^{(t)})\bigr\|^2 \\
			&\quad + \bar{A}_1 \sigma_g^2 + \bar{A}_2 \sigma^2 + \bar{A}_3 + T_2^{(t)}.
		\end{align*}
		Under the learning-rate conditions of Theorem~\ref{theo:con}, $\frac{2\tilde{\eta}_s L}{\epsilon^2} \le \frac{1}{8 D_T}$ and $\frac{8 A \eta_c^2 L^2 K(K-1)}{\epsilon} \le \frac{1}{8 D_T}$ (since $64 A \eta_c^2 L^2 K(K-1) \epsilon \le \frac{\epsilon}{2}$ and $64 A \eta_c^2 L^2 K(K-1) \cdot 3\eta_c K G \sqrt{T} \le \frac{\epsilon}{2}$). Hence the bracketed coefficient is at least $\frac{1}{4 D_T}$. Setting $\Gamma := \frac{\tilde{\eta}_s}{4 D_T}$, taking full expectations, summing from $t = 0$ to $T-1$, and using~\eqref{eq:T2_telescope} along with $\Phi^{(0)} - \mathbb{E}[\Phi^{(T)}] \le \Delta_0 + c_y Y^{(0)} \le \Delta_0 + \Psi_0$ gives
		\begin{align*}
			\Gamma \sum_{t=0}^{T-1}\mathbb{E}\bigl\|\nabla f(\mathbf{w}^{(t)})\bigr\|^2
			\le \Delta_0 + \Psi_0 + \frac{3\tilde{\eta}_s d G^2}{\epsilon} + T\bigl(\bar{A}_1 \sigma_g^2 + \bar{A}_2 \sigma^2 + \bar{A}_3\bigr).
		\end{align*}
		Dividing both sides by $\Gamma T = \frac{\tilde{\eta}_s T}{4 D_T}$ completes the proof of Theorem~\ref{theo:con}.
	\end{proof}

	\begin{proof}[Proof of Corollary~\ref{cor:rate}]
		With $\eta_c = \Theta(\frac{1}{KL\sqrt{T}})$, $\epsilon = \Theta(\frac{G}{L})$, and $\eta_s = \Theta(1)$, we have $\tilde{\eta}_s = \eta_s \eta_c K = \Theta(\frac{\eta_s}{L\sqrt{T}})$ and $D_T = 3\eta_c K G \sqrt{T} + \epsilon = \Theta(\frac{G}{L}) = \Theta(\epsilon)$. Thus $\frac{D_T}{\tilde{\eta}_s} = \Theta(\frac{G\sqrt{T}}{\eta_s})$. Substituting these scalings into each term of Theorem~\ref{theo:con}:
		(i) $\frac{4 D_T \Delta_0}{\tilde{\eta}_s T} = \mathcal{O}(\frac{G\Delta_0}{\eta_s\sqrt{T}})$ and $\frac{4 D_T \Psi_0}{\tilde{\eta}_s T} = \mathcal{O}\bigl(\frac{\eta_s G L}{T^{3/2}}\frac{N(N-M)}{M^2(N-1)}\bigr) = o(\frac{1}{T})$ for fixed $N, M$;
		(ii) $\frac{12 D_T d G^2}{\epsilon T} = \mathcal{O}(\frac{d G^2}{T})$;
		(iii) $\frac{4 D_T}{\tilde{\eta}_s}\bar{A}_1 \sigma_g^2 = \Theta(A \eta_c^2 L^2 K(K-1)\sigma_g^2) = \mathcal{O}(\frac{A(K-1)\sigma_g^2}{KT})$;
		(iv) $\frac{4 D_T}{\tilde{\eta}_s}\bar{A}_2 \sigma^2 = \mathcal{O}(\frac{A(K-1)\sigma^2}{K^2 T}) + \mathcal{O}(\frac{L\eta_s B\sigma^2}{G M K \sqrt{T}})$; and
		(v) $\frac{4 D_T}{\tilde{\eta}_s}\bar{A}_3 = \mathcal{O}\bigl(\frac{\eta_s^3 L^3}{G\,T^{3/2}}\frac{N-M}{M(N-1)}\bigl(1 + \frac{2N(N-M)}{M^2}\bigr)\bigr) = o(\frac{1}{T})$ for fixed $N, M$.
		Combining (i)--(v) yields~\eqref{eq:cor_rate}.
	\end{proof}

    \section{Algorithms}\label{Apd:Algorithms}
    
	\subsection{Client Device Update Procedure}\label{Apd:DeviceUpdate}
    Each participating client $i$ executes the \textsc{DeviceUpdate} procedure described in Algorithm~\ref{DeviceUpdateAlgorithm}. The client initialises its local model from the current global model $\mathbf{w}^{(t)}$ and performs $K$ stochastic gradient descent steps using its private dataset $\mathcal{D}_i$ at learning rate $\eta_c$. The resulting local update, computed as the scaled difference between the initial and final model states, is then returned to the server.

	\begin{algorithm}
		\caption{\textsc{DeviceUpdate}(\(i,\,\mathbf{w}^{(t)},\,\eta_c\))}
		\label{DeviceUpdateAlgorithm}
    	\begin{algorithmic}[1]
			\State \(\mathbf{w}_i^{(t,0)} \gets \mathbf{w}^{(t)}\)
			\For{local step \(k = 0, 1, \ldots, K - 1\)}
			\State Compute stochastic gradient \(\nabla f_i\left(\mathbf{w}_i^{(t,k)}\right)\)
			\State \(\mathbf{w}_i^{(t,k+1)} \gets \mathbf{w}_i^{(t,k)} - \eta_c\, \nabla f_i\left(\mathbf{w}_i^{(t,k)}\right)\)
			\EndFor
			\State \Return \(\frac{1}{\eta_c}\left(\mathbf{w}^{(t)} - \mathbf{w}_i^{(t,K)}\right)\)
		\end{algorithmic}
	\end{algorithm}
    
	\subsection{Adaptive Optimisers}\label{Apd:AdpOpt}

    This appendix provides the full pseudocode and update rules for each of the five server-side adaptive optimisers supported by \textsc{FedAdaVR}, as referenced in Section~\ref{sec:ProposedAlgorithm}. In each case, the optimiser receives the current global model $\mathbf{w}^{(t)}$, the pseudo-gradient $\mathbf{G}^{(t)}$ (computed via equations~\eqref{GradientUpdateEquation} and~\eqref{GradientUpdateEquationLambda}), the server learning rate $\eta_s$, and optimiser-specific state variables. Only one optimiser is selected per experimental run; the five options are mutually exclusive.

	\subsubsection{Adagrad (Algorithm \ref{alg:AdagOpt}).}
	\begin{algorithm}[H]
		\caption{\textsc{AdagradOptimiser} (Latest model parameters \( \mathbf{w}^{(t)} \), pseudo-gradient \( \mathbf{G}^{(t)} \), server learning rate \( \eta_s \), accumulator \( z_t \))}
		\label{alg:AdagOpt}
		\begin{algorithmic}[1]
			\Require \( \epsilon \) (small constant), \( \lambda \) (weight decay)
			\If {\(\lambda \neq 0\)}
			\State \(\mathbf{G}^{(t)} \gets \mathbf{G}^{(t)} + \lambda \mathbf{w}^{(t)}\);
			\EndIf
			\State \(z_t \gets z_{t-1} + \mathbf{G}^{(t)} \odot \mathbf{G}^{(t)}\);
			\State \(\mathbf{w}^{(t+1)} \gets \mathbf{w}^{(t)} - \eta_s \frac{\mathbf{G}^{(t)}}{\sqrt{z_t} + \epsilon}\);
			\State \Return \(\mathbf{w}^{(t+1)}\).
		\end{algorithmic}
	\end{algorithm}
	
	The value $\mathbf{G}^{(t)}$, as computed in equation \eqref{GradientUpdateEquationLambda}, is utilised in the calculation of $z_t$ (accumulator):
	\begin{equation}\label{GradientUpdateEquationAdagradAccumulator}
		z_t \gets z_{t-1} + \mathbf{G}^{(t)} \odot \mathbf{G}^{(t)}.
	\end{equation}	
	Then \(z_t\) is used for the global update:
	\begin{equation}\label{GradientUpdateEquationAdagrad}
		\mathbf{w}^{(t+1)} \gets \mathbf{w}^{(t)} - \eta_s \frac{\mathbf{G}^{(t)}}{\sqrt{z_t} + \epsilon},
	\end{equation}
	where \(\epsilon\) is a small constant.

	\subsubsection{Adam (Algorithm \ref{alg:AdamOpt}).}
	
	\begin{algorithm}[H]
		\caption{\textsc{AdamOptimiser} (Latest model parameters \( \mathbf{w}^{(t)} \), pseudo-gradient \( \mathbf{G}^{(t)} \), server learning rate \( \eta_s \), moments \(m_t, v_t\))}
		\label{alg:AdamOpt}
		\begin{algorithmic}[1]
			\Require Decay rates \( \beta_1 \), \( \beta_2 \), small constant \( \epsilon \), weight decay \( \lambda \)
			\If{\(\lambda \neq 0\)}
			\State \(\mathbf{G}^{(t)} \gets \mathbf{G}^{(t)} + \lambda\, \mathbf{w}^{(t)}\);
			\EndIf
			\State \(m_t \gets \beta_1\, m_{t-1} + (1 - \beta_1)\, \mathbf{G}^{(t)}\);
			\State \(v_t \gets \beta_2\, v_{t-1} + (1 - \beta_2)\, \bigl(\mathbf{G}^{(t)} \odot \mathbf{G}^{(t)}\bigr)\);
			\State \(\hat{m}_t \gets \frac{m_t}{1 - \beta_1^\mathcal{T}}\);
			\State \(\hat{v}_t \gets \frac{v_t}{1 - \beta_2^\mathcal{T}}\);
			\State \(\mathbf{w}^{(t+1)} \gets \mathbf{w}^{(t)} - \eta_s\, \frac{\hat{m}_t}{\sqrt{\hat{v}_t} + \epsilon}\);
			\State \Return \(\mathbf{w}^{(t+1)}\).
		\end{algorithmic}
	\end{algorithm}
	
	The $\mathbf{G}^{(t)}$, obtained from the equation \eqref{GradientUpdateEquationLambda}, is employed to compute the moments $m_t$ and $v_t$, both of which are initially set to 0. The hyperparameters $\beta_1$ and $\beta_2$ are also required for the computation of $m_t$ and $v_t$:
	\begin{equation}\label{GradientUpdateEquationAdamMt}
		m_t \gets \beta_1\, m_{t-1} + (1 - \beta_1)\, \mathbf{G}^{(t)};
	\end{equation}
	\begin{equation}\label{GradientUpdateEquationAdamVt}
		v_t \gets \beta_2\, v_{t-1} + (1 - \beta_2)\, \bigl(\mathbf{G}^{(t)} \odot \mathbf{G}^{(t)}\bigr).
	\end{equation}
	The moments $m_t$ and $v_t$ are subsequently utilised to compute $\hat{m}_t$ and $\hat{v}_t$:
	\begin{equation}\label{GradientUpdateEquationAdamMht}
		\hat{m}_t \gets \frac{m_t}{1 - \beta_1^\mathcal{T}};
	\end{equation}
	\begin{equation}\label{GradientUpdateEquationAdamVht}
		\hat{v}_t \gets \frac{v_t}{1 - \beta_2^\mathcal{T}}.
	\end{equation}
	Here, $\mathcal{T}$ denotes the iteration counter (the server round, with $\mathcal{T} = t+1$ at round $t$ in the proposed algorithms). The values $\hat{m}_t$ and $\hat{v}_t$ are then employed to compute the global update:
	\begin{equation}\label{GradientUpdateEquationAdam}
		\mathbf{w}^{(t+1)} \gets \mathbf{w}^{(t)} - \eta_s\, \frac{\hat{m}_t}{\sqrt{\hat{v}_t} + \epsilon},
	\end{equation}
	where \(\epsilon\) is a small constant.
	
	\subsubsection{Adabelief (Algorithm \ref{alg:AdabOpt}).}
	
	\begin{algorithm}[H]
		\caption{\textsc{AdabeliefOptimiser} (Latest model parameters \( \mathbf{w}^{(t)} \), pseudo-gradient \( \mathbf{G}^{(t)} \), server learning rate \( \eta_s \), moments \( m_t, z_t \))}
		\label{alg:AdabOpt}
		\begin{algorithmic}[1]
			\Require Decay rates \( \beta_1 \), \( \beta_2 \), small constant \( \epsilon \), weight decay \( \lambda \)
			\If{\(\lambda \neq 0\)}
			\State \(\mathbf{G}^{(t)} \gets \mathbf{G}^{(t)} + \lambda\, \mathbf{w}^{(t)}\);
			\EndIf
			\State \(m_t \gets \beta_1\, m_{t-1} + (1 - \beta_1)\, \mathbf{G}^{(t)}\);
			\State \(z_t \gets \beta_2\, z_{t-1} + (1 - \beta_2)\, \bigl(\mathbf{G}^{(t)} - m_t\bigr)^2\);
			\State \(\hat{m}_t \gets \frac{m_t}{1 - \beta_1^\mathcal{T}}\);
			\State \(\hat{z}_t \gets \frac{z_t}{1 - \beta_2^\mathcal{T}}\);
			\State \(\mathbf{w}^{(t+1)} \gets \mathbf{w}^{(t)} - \eta_s\, \frac{\hat{m}_t}{\sqrt{\hat{z}_t} + \epsilon}\);
			\State \Return \(\mathbf{w}^{(t+1)}\).
		\end{algorithmic}
	\end{algorithm}
	
	The $\mathbf{G}^{(t)}$, obtained in the equation \eqref{GradientUpdateEquationLambda}, is used to compute the moments $m_t$ and $z_t$, both of which are initially set to 0. The hyperparameters $\beta_1$ and $\beta_2$ are also required for the computation of $m_t$ and $z_t$.
	\begin{equation}\label{GradientUpdateEquationAdaBeliefMt}
		m_t \gets \beta_1\, m_{t-1} + (1 - \beta_1)\, \mathbf{G}^{(t)};
	\end{equation}
	\begin{equation}\label{GradientUpdateEquationAdaBeliefSt}
		z_t \gets \beta_2\, z_{t-1} + (1 - \beta_2)\, \bigl(\mathbf{G}^{(t)} - m_t\bigr)^2.
	\end{equation}
	The moments $m_t$ and $z_t$ are subsequently used to compute $\hat{m}_t$ and $\hat{z}_t$:
	\begin{equation}\label{GradientUpdateEquationAdaBeliefMht}
		\hat{m}_t \gets \frac{m_t}{1 - \beta_1^\mathcal{T}};
	\end{equation}
	\begin{equation}\label{GradientUpdateEquationAdaBeliefSht}
		\hat{z}_t \gets \frac{z_t}{1 - \beta_2^\mathcal{T}}.
	\end{equation}
	Here, $\mathcal{T}$ denotes the iteration counter (the server round, with $\mathcal{T} = t+1$ at round $t$ in the proposed algorithms). The values $\hat{m}_t$ and $\hat{z}_t$ are then employed to compute the global update:
	\begin{equation}\label{GradientUpdateEquationAdaBelief}
		\mathbf{w}^{(t+1)} \gets \mathbf{w}^{(t)} - \eta_s\, \frac{\hat{m}_t}{\sqrt{\hat{z}_t} + \epsilon},
	\end{equation}
	where \(\epsilon\) is a small constant.
	
	\subsubsection{Yogi (Algorithm \ref{alg:YogiOpt}).}
	
	\begin{algorithm}[H]
		\caption{\textsc{YogiOptimiser} (Latest model parameters \( \mathbf{w}^{(t)} \), pseudo-gradient \( \mathbf{G}^{(t)} \), server learning rate \( \eta_s \), moments \(m_t, v_t\))}
		\label{alg:YogiOpt}
		\begin{algorithmic}[1]
			\Require Decay rates \( \beta_1 \), \( \beta_2 \), small constant \( \epsilon \), and weight decay \( \lambda \)
			\If{\(\lambda \neq 0\)}
			\State \(\mathbf{G}^{(t)} \gets \mathbf{G}^{(t)} + \lambda\, \mathbf{w}^{(t)}\);
			\EndIf
			\State \(m_t \gets \beta_1\, m_{t-1} + (1 - \beta_1)\, \mathbf{G}^{(t)}\);
			\State \(v_t \gets v_{t-1} - (1 - \beta_2)\, \mathbf{G}^{(t)} \odot \mathbf{G}^{(t)}\, \mathrm{sgn}(v_{t-1} - \mathbf{G}^{(t)} \odot \mathbf{G}^{(t)})\);
			\State \(\hat{m}_t \gets \frac{m_t}{1 - \beta_1^\mathcal{T}}\);
			\State \(\hat{v}_t \gets \frac{v_t}{1 - \beta_2^\mathcal{T}}\);
			\State \(\mathbf{w}^{(t+1)} \gets \mathbf{w}^{(t)} - \eta_s\, \frac{\hat{m}_t}{\sqrt{\hat{v}_t} + \epsilon}\);
			\State \Return \(\mathbf{w}^{(t+1)}\).
		\end{algorithmic}
	\end{algorithm}
	
	The $\mathbf{G}^{(t)}$ calculated in the equation \eqref{GradientUpdateEquationLambda} is used to compute the moments $m_t$ and $v_t$, which are initially set to 0. The hyperparameters $\beta_1$ and $\beta_2$ are also required for calculating $m_t$ and $v_t$:   \begin{equation}\label{GradientUpdateEquationYogiMt}
		m_t \gets \beta_1\, m_{t-1} + (1 - \beta_1)\, \mathbf{G}^{(t)};
	\end{equation}
	\begin{equation}\label{GradientUpdateEquationYogiVt}
		v_t \gets v_{t-1} - (1 - \beta_2)\, \mathbf{G}^{(t)} \odot \mathbf{G}^{(t)}\, \mathrm{sgn}(v_{t-1} - \mathbf{G}^{(t)} \odot \mathbf{G}^{(t)}).
	\end{equation}
	The $m_t$ and $v_t$ are then used to compute $\hat{m}_t$ and $\hat{v}_t$:
	\begin{equation}\label{GradientUpdateEquationYogiMht}
		\hat{m}_t \gets \frac{m_t}{1 - \beta_1^\mathcal{T}};
	\end{equation}
	\begin{equation}\label{GradientUpdateEquationYogiVht}
		\hat{v}_t \gets \frac{v_t}{1 - \beta_2^\mathcal{T}}.
	\end{equation}
	Here, $\mathcal{T}$ denotes the iteration counter (the server round, with $\mathcal{T} = t+1$ at round $t$ in the proposed algorithms). Subsequently, $\hat{m}_t$ and $\hat{v}_t$ are utilised for the global update:
	\begin{equation}\label{GradientUpdateEquationYogi}
		\mathbf{w}^{(t+1)} \gets \mathbf{w}^{(t)} - \eta_s\, \frac{\hat{m}_t}{\sqrt{\hat{v}_t} + \epsilon},
	\end{equation}
	where \(\epsilon\) is a small constant.
	
	\subsubsection{Lamb (Algorithm \ref{alg:LambOpt}).}
	
	\begin{algorithm}[H]
		\caption{\textsc{LambOptimiser} (Latest model parameters \( \mathbf{w}^{(t)} \), pseudo-gradient \( \mathbf{G}^{(t)} \), server learning rate \( \eta_s \), moments \(m_t, v_t\))}
		\label{alg:LambOpt}
		\begin{algorithmic}[1]
			\Require Decay rates \( \beta_1 \), \( \beta_2 \), small constant \( \epsilon \), and weight decay \( \lambda \)
			\If{\(\lambda \neq 0\)}
			\State \(\mathbf{G}^{(t)} \gets \mathbf{G}^{(t)} + \lambda\, \mathbf{w}^{(t)}\);
			\EndIf
			\State \(m_t \gets \beta_1\, m_{t-1} + (1 - \beta_1)\, \mathbf{G}^{(t)}\);
			\State \(v_t \gets \beta_2\, v_{t-1} + (1 - \beta_2)\, \bigl(\mathbf{G}^{(t)} \odot \mathbf{G}^{(t)}\bigr)\);
			\State \(\hat{m}_t \gets \frac{m_t}{1 - \beta_1^\mathcal{T}}\);
			\State \(\hat{v}_t \gets \frac{v_t}{1 - \beta_2^\mathcal{T}}\);
			\State \(\hat{r}_t \gets \frac{\hat{m}_t}{\sqrt{\hat{v}_t} + \epsilon}\);
			\State \(weight\_norm \gets \|\mathbf{w}^{(t)}\|\);
			\State \(update\_norm \gets \|\hat{r}_t\|\);
			\If{\(weight\_norm > 0\) \textbf{and} \(update\_norm > 0\)}
			\State \(r_t \gets \frac{weight\_norm}{update\_norm}\);
			\Else
			\State \(r_t \gets 1.0\);
			\EndIf
			\State \(\mathbf{w}^{(t+1)} \gets \mathbf{w}^{(t)} - \eta_s\, r_t\, \hat{r}_t\);
			\State \Return \(\mathbf{w}^{(t+1)}\).
		\end{algorithmic}
	\end{algorithm}
	
	$\mathbf{G}^{(t)}$ calculated in the equation \eqref{GradientUpdateEquationLambda} is used to compute the moments $m_t$ and $v_t$, both initially set to 0. The hyperparameters $\beta_1$ and $\beta_2$ are also required to calculate $m_t$ and $v_t$:
	\begin{equation}\label{GradientUpdateEquationLambMt}
		m_t \gets \beta_1\, m_{t-1} + (1 - \beta_1)\, \mathbf{G}^{(t)};
	\end{equation}
	\begin{equation}\label{GradientUpdateEquationLambVt}
		v_t \gets \beta_2\, v_{t-1} + (1 - \beta_2)\, \bigl(\mathbf{G}^{(t)} \odot \mathbf{G}^{(t)}\bigr).
	\end{equation}
	The $m_t$ and $v_t$ are then used to compute $\hat{m}_t$ and $\hat{v}_t$:
	\begin{equation}\label{GradientUpdateEquationLambMht}
		\hat{m}_t \gets \frac{m_t}{1 - \beta_1^\mathcal{T}};
	\end{equation}
	\begin{equation}\label{GradientUpdateEquationLambVht}
		\hat{r}_t \gets \frac{\hat{m}_t}{\sqrt{\hat{v}_t} + \epsilon};
	\end{equation}
	\begin{equation}\label{GradientUpdateEquationLambRht}
		\hat{v}_t \gets \frac{v_t}{1 - \beta_2^\mathcal{T}};
	\end{equation}
	Here, $\mathcal{T}$ denotes the iteration counter (the server round, with $\mathcal{T} = t+1$ at round $t$ in the proposed algorithms). Then, $\mathbf{w}^{(t)}$ and $\hat{r}_t$ are used to calculate:
	\begin{equation}\label{GradientUpdateEquationLambWN}
		weight\_norm \gets \|\mathbf{w}^{(t)}\|;
	\end{equation}
	\begin{equation}\label{GradientUpdateEquationLambUN}
		update\_norm \gets \|\hat{r}_t\|;
	\end{equation}
	\begin{equation}\label{GradientUpdateEquationLambRt}
		r_t = 
		\begin{cases}
			\frac{weight\_norm}{update\_norm} & \text{\scriptsize \(weight\_norm\) > 0 and \(update\_norm\) > 0} \\
			r_t \gets 1.0 & 
		\end{cases}.
	\end{equation}
	Then, $r_t$ and $\hat{r}_t$ are used for the global update:
	\begin{equation}\label{GradientUpdateEquationLamb}
		\mathbf{w}^{(t+1)} \gets \mathbf{w}^{(t)} - \eta_s\, r_t\, \hat{r}_t.
	\end{equation}
	
	\section{Model Quantisation}\label{Apd:ModQuant}
	
	While various quantisation paradigms exist, we focus on three distinct strategies for compressing stored client updates. Algorithm~\ref{alg:quantised} details the use of FP16 (half-precision), Int8, and Int4 quantisation, complemented by the corresponding dequantisation procedures in Algorithm~\ref{alg:dequant}. These methods achieve server-side memory reductions of 50\%, 75\%, and 87.5\%, respectively. We provide experimental validation of these integer-based formats; however, a comprehensive evaluation of alternative quantisation schemes is beyond the scope of this work due to computational constraints. For 8-bit quantisation we use the symmetric per-tensor uniform quantiser of \citep{Krishnamoorthi2018}, and for 4-bit quantisation we follow \citep{Zhou2016}, packing two 4-bit values per byte for storage. Values are scaled by the tensor maximum; per-channel/asymmetric alternatives exist but are out of scope.

	\begin{algorithm}
		\caption{\textsc{Quant}(\( \mathbf{g}_j^{(t)}\))}
		\label{alg:quantised}
		\begin{algorithmic}[1]
			\State \textbf{Initialisation:} Set \(mode\) \Comment{Initialise quantisation \(mode\) (eg: FP16, Int8, Int4). }
			\State \(\mathcal{Q} \gets \varnothing\); \Comment{Initialise list to store compressed layers}
			\ForAll{\(\mathbf{W}\) in \( \mathbf{g}_j^{(t)} \)}
			
			\If{\(mode = \text{FP16}\)}
			\State \(\mathbf{W}_{\text{q}} \gets \text{cast}(\mathbf{W}, \text{float16})\);
			\State Append \(\mathbf{W}_{\text{q}}\) to \(\mathcal{Q}\);
			
			\ElsIf{\(mode = \text{Int8}\)}
			\State \(\alpha \gets \max(|\mathbf{W}|) / 127\); \Comment{Scale for signed 8-bit}
			\State \(\alpha \gets 1.0\) \textbf{if} \(\alpha = 0\);
			\State \(\mathbf{W}_{\text{int}} \gets \text{cast}(\text{clip}(\text{round}(\mathbf{W}/\alpha), -127, 127), \text{int8})\);
			\State Append \((\mathbf{W}_{\text{int}}, \alpha)\) to \(\mathcal{Q}\); \Comment{Store Int8 tensor and scale}
			
			\ElsIf{\(mode = \text{Int4}\)}
			\State \(S \gets \text{shape}(\mathbf{W})\);
			\State \(\alpha \gets \max(|\mathbf{W}|) / 7\); \Comment{Scale for signed 4-bit}
			\State \(\alpha \gets 1.0\) \textbf{if} \(\alpha = 0\);
			\State \(\mathbf{W}_{\text{int}} \gets \text{clip}(\text{round}(\mathbf{W}/\alpha), -7, 7)\);
			\State \(\mathbf{W}_{\text{flat}} \gets \text{flatten}(\mathbf{W}_{\text{int}} + 8)\); \Comment{Shift to \([1, 15]\) and flatten}
			
			\State \(\mathbf{P} \gets \varnothing\);
			\For{\(k = 0, 2, \dots, \text{length}(\mathbf{W}_{\text{flat}}) - 1\)}
			\State \(h \gets \mathbf{W}_{\text{flat}}[k] \ll 4\); \Comment{High 4 bits}
			\State \(l \gets \mathbf{W}_{\text{flat}}[k+1]\); \Comment{Low 4 bits}
			\State Append \((h \lor l)\) to \(\mathbf{P}\); \Comment{Pack into Uint8}
			\EndFor
			\State Append \((\mathbf{P}, \alpha, S)\) to \(\mathcal{Q}\); \Comment{Store packed data, scale, and shape}
			\EndIf
			
			\EndFor
			\State \Return \(\mathcal{Q}\).
		\end{algorithmic}
	\end{algorithm}
	
	\begin{algorithm}
		\caption{\textsc{Dequant}(\( \mathcal{Q}\))}
		\label{alg:dequant}
		\begin{algorithmic}[1]
			\State \textbf{Initialisation:} Set \(mode\) \Comment{Initialise dequantisation \(mode\) (eg: FP16, Int8, Int4). }
			\State \(\mathcal{W}_{\text{restored}} \gets \varnothing\);
			
			\ForAll{\text{item} \textbf{in} \(\mathcal{Q}\)}
			
			\If{\(\text{mode} = \text{FP16}\)}
			\State \(\mathbf{W}_{\text{float}} \gets \text{cast}(\text{item}, \text{float32})\);
			\State Append \(\mathbf{W}_{\text{float}}\) to \(\mathcal{W}_{\text{restored}}\);
			
			\ElsIf{\(\text{mode} = \text{Int8}\)}
			\State \((\mathbf{W}_{\text{int}}, \alpha) \gets \text{item}\);
			\State \(\mathbf{W}_{\text{float}} \gets \text{cast}(\mathbf{W}_{\text{int}}, \text{float32}) \times \alpha\); \Comment{Restore scale}
			\State Append \(\mathbf{W}_{\text{float}}\) to \(\mathcal{W}_{\text{restored}}\);
			
			\ElsIf{\(\text{mode} = \text{Int4}\)}
			\State \((\mathbf{P}, \alpha, S) \gets \text{item}\); \Comment{Unpack tuple: Packed data, scale, shape}
			\State \(N \gets \prod_{d \in S} d\); \Comment{Total elements}
			
			\State \(\mathbf{U}_{\text{even}} \gets (\mathbf{P} \gg 4)\); \Comment{Extract high 4 bits}
			\State \(\mathbf{U}_{\text{odd}} \gets (\mathbf{P} \ \& \ 0\text{x}0\text{F})\); \Comment{Extract low 4 bits}
			
			\State \(\mathbf{U} \gets \text{Interleave}(\mathbf{U}_{\text{even}}, \mathbf{U}_{\text{odd}})\); \Comment{Vectorised merge}
			\State \(\mathbf{U} \gets \mathbf{U}[0 : N]\); \Comment{Remove padding if any}
			
			\State \(\mathbf{W}_{\text{signed}} \gets \text{cast}(\mathbf{U}, \text{int8}) - 8\); \Comment{Reverse shift offset}
			\State \(\mathbf{W}_{\text{float}} \gets \text{reshape}(\text{cast}(\mathbf{W}_{\text{signed}}, \text{float32}) \times \alpha, S)\);
			\State Append \(\mathbf{W}_{\text{float}}\) to \(\mathcal{W}_{\text{restored}}\);
			\EndIf
			
			\EndFor
			\State \Return \(\mathcal{W}_{\text{restored}}\).
		\end{algorithmic}
	\end{algorithm}

	\section{Experimental Setup}\label{Apd:ExpSetup}
    
    This appendix provides a complete description of all experimental configurations used in Section~\ref{Experiments} of the main paper, supplementing the concise summary provided there.

    \subsection{Datasets}\label{dataset}
    To conduct the experiments, we utilised three publicly available vision classification datasets and one next-character prediction text dataset. Brief descriptions of these datasets are provided below.

    \begin{itemize}
        \item \textbf{MNIST \citep{deng2012mnist}.} The MNIST dataset contains 60,000 handwritten digit images (50,000 for training and 10,000 for testing) across 10 classes. Each image is a 28$\times$28 grayscale image.
        
        \item \textbf{FMNIST \citep{xiao2017fmnist}.} The FMNIST dataset consists of 60,000 training and 10,000 test images of fashion items, each as a 28$\times$28 grayscale image across 10 classes.
        
        \item \textbf{CIFAR-10 \citep{Krizhevsky09cifar10}.} The CIFAR-10 dataset comprises 60,000 colour images (50,000 for training and 10,000 for testing) of size 32$\times$32 across 10 classes.
        
        \item \textbf{Shakespeare \citep{Caldas2018}.} This dataset is built from The Complete Works of William Shakespeare. We used the Flower~\citep{FlowerPaper} Hugging Face repository to download this dataset, which contains a total of 4,226,158 samples. Each speaking role in each play is treated as a distinct device, yielding 1,129 clients in 
        total.
    \end{itemize}

    \subsection{Dataset Partitioning Methods}\label{DataPartitioningMethod}
    
    In our experiments, data were partitioned into an Independent and Identically Distributed (IID) setting and five distinct Non-IID settings. Vision datasets were partitioned based on label skewness, whereas the Shakespeare dataset utilised a Natural ID partitioning strategy. These settings include:
    
    \begin{itemize}
        \item \textbf{IID:} Data are uniformly distributed across labels with no skewness. Each client receives an equal amount of data from every class.
        
        \item \textbf{IID-NonIID:} A hybrid setting in which 50\% of the data follows an IID distribution, while the remainder follows a non-IID distribution. The notation Non\_IID is used instead of IID-NonIID for clarity in the tables.
        
        \item \textbf{Dirichlet:} A non-IID scenario where data points are assigned to clients according to a Dirichlet distribution. The parameter $\beta$ governs the degree of skewness; smaller $\beta$ values indicate greater skew, whereas larger values approximate 
        an IID distribution. In all experiments, $\beta = 0.5$ was employed.
        
        \item \textbf{Sort and Partition:} Data are initially sorted by label and then partitioned into $C$ chunks, where $C = 1$ implies each partition contains data from only one class, representing the most extreme non-IID setting. LQ-1, LQ-2, and LQ-3 refer to $C = 1, 2, 
        3$ chunks respectively.
        
        \item \textbf{Natural ID Partitioner:} Unlike synthetic splitting strategies, this method preserves the inherent user-based segmentation present in the raw data, capturing realistic non-IID statistical heterogeneity. This strategy is applied exclusively to the Shakespeare 
        dataset.
    \end{itemize}

    Further details regarding these partitioning methods are provided in~\citep{LiPartitioning} and the official Flower documentation~\citep{FlowerPaper}.

    \subsection{Model Architectures}
    For MNIST and FMNIST, the LeNet-5 architecture~\citep{Lecun1998} was adopted, comprising seven layers. For CIFAR-10, the ResNet-18 architecture~\citep{resnet} was employed, with batch normalisation replaced by group normalisation~\citep{Hsieh2019}. For the Shakespeare dataset, we utilise a recurrent neural network (RNN) based on Gated Recurrent Units (GRU)~\citep{Chung2014}. The GRU architecture consists of an 8-dimensional character embedding layer, followed by a single GRU layer with 128 hidden units, and concludes with a fully connected layer mapping the final hidden state to the character vocabulary size.

    \subsection{Client Configuration}
    The MNIST and FMNIST datasets were partitioned across 500 clients, CIFAR-10 across 250 clients, and Shakespeare into 904 training clients. In each federated learning round, five clients were selected uniformly at random without replacement from the complete client pool, yielding participation rates of 1\% for MNIST and FMNIST, 2\% for CIFAR-10, and below 1\% for Shakespeare. Test datasets were partitioned identically to the training distributions. For the vision datasets, 250 clients were sampled for evaluation after each round, whereas for Shakespeare the model was evaluated on a held-out set of 225 clients from the full 1,129 (904 used for training).

    \subsection{Hyperparameter Configuration}\label{ExperimentalDetails}
    
    Hyperparameter configurations were tailored to each specific algorithm to ensure fair benchmarking. Most experiments were conducted using a single fixed random seed, while some were repeated with two seeds and a small subset with three seeds to assess reproducibility. For client-side optimisation, we performed a grid search over the learning rate $\eta_c \in \{0.001, 0.01, 0.1\}$ and report the results of the best-performing configuration. For methods incorporating server-side optimisation (FedVARP, \textsc{FedAdaVR}, and \textsc{FedAdaVR-Quant}), we conducted an additional grid search over $\eta_s \in \{0.01, 0.1, 1.0\}$ for FedVARP and $\eta_s \in \{0.001, 0.005, 0.01\}$ for the \textsc{FedAdaVR} variants. The algorithms MIFA, FedVARP, \textsc{FedAdaVR}, and \textsc{FedAdaVR-Quant} were implemented directly using Flower and PyTorch. For baseline algorithms already available within Flower (FedAvg, FedAdagrad, FedAdam, FedYogi, FedProx, SCAFFOLD, and FedNova), default hyperparameter values provided by the official Flower repositories were used. All experiments were executed on high-performance computing systems equipped with CPUs and GPUs. These default server hyperparameters in Flower (for FedAdagrad, FedAdam, and FedYogi) have been extensively tuned and validated by the framework maintainers to reproduce the peak performance of their original publications. Conversely, because FedVARP, \textsc{FedAdaVR}, and \textsc{FedAdaVR-Quant} do not have pre-existing framework defaults, we restricted their server-side tuning to a minimal three-value grid over $\eta_s$ while keeping all other optimiser hyperparameters ($\beta_1, \beta_2, \epsilon$) completely untuned at standard defaults, ensuring a fair comparison across all methods.

    Table~\ref{tab:ExperimentalDetails} provides a complete summary of all parameters and hyperparameters employed across datasets. The best-performing configurations for FedVARP, \textsc{FedAdaVR}, and \textsc{FedAdaVR-Quant} across datasets and partitioning strategies are summarised in Tables~\ref{tab:shakespeare_accuracy_com}, \ref{tab:MNIST_accuracy_com}, \ref{tab:FashiontMNIST_accuracy_com}, and \ref{tab:CIFAR10_accuracy_com} in Appendix~\ref{Apd:ComOwn}.

    \begin{table}[ht]
        \caption{Experimental and hyperparameter setup. Client momentum is used in all experiments, whereas the convergence analysis in Section~\ref{ConvergenceAnalysis} considers vanilla local SGD without momentum.}
        \label{tab:ExperimentalDetails}
        \centering
        \begin{tabular}{l cccc}
            \toprule
            & \multicolumn{4}{c}{\textbf{Dataset}} \\
            \cmidrule(lr){2-5}
            \textbf{Parameter} 
            & \textbf{MNIST} 
            & \textbf{F-MNIST}  
            & \textbf{CIFAR-10} 
            & \textbf{Shakespeare} \\
            \midrule
            \# of clients                & 500     & 500     & 250       & 904 \\
            \# of participating clients  & 5       & 5       & 5         & 5 \\
            \% of participating clients  & 1\%     & 1\%     & 2\%       & $<1\%$ \\
            \# of clients for evaluation & 250     & 250     & 250       & 225 \\
            Local model                  & LeNet-5 & LeNet-5 & ResNet-18 & GRU \\
            Batch size                   & 20      & 20      & 20        & 64 \\
            Local epochs                 & 1       & 3       & 5         & 5 \\
                Client learning rate                 & 0.1       & 0.01       & 0.01         & 0.1, 0.01 \\
            Client momentum              & \multicolumn{4}{c}{0.9} \\
            Decay rate $\beta_1$         & \multicolumn{4}{c}{0.9} \\
            Decay rate $\beta_2$         & \multicolumn{4}{c}{0.999} \\
            Small constant $\epsilon$    & \multicolumn{4}{c}{$1\times10^{-8}$} \\
            Weight decay $\lambda$       & \multicolumn{4}{c}{0} \\
            \bottomrule
        \end{tabular}
    \end{table}

    \subsection{Training Rounds and Evaluation Protocol}
    Convergence speed varies according to the dataset and the complexity of the data partitioning method. Consequently, the number of federated learning rounds was adjusted per setting, with more rounds allocated to more challenging scenarios. Due to notable performance fluctuations between rounds, average accuracy is reported over the final rounds (typically the last 10\% of total rounds) in Tables~\ref{tab:shakespeare_accuracy_com}, \ref{tab:MNIST_accuracy_com}, \ref{tab:FashiontMNIST_accuracy_com}, and \ref{tab:CIFAR10_accuracy_com}. Table~\ref{tab:dataset_partition} summarises the total rounds and the subset of final rounds used for averaging per partitioning strategy. For instance, in the LQ-1 partitioning of CIFAR-10, 1,500 FL rounds were conducted with results from the last 150 rounds averaged.

    \begin{table}[ht]
        \caption{Number of training rounds and selected rounds for average accuracy reporting, by dataset and partition strategy.}
        \label{tab:dataset_partition}
        \centering
        \begin{tabular}{ll ccccccc}
            \toprule
            & & \multicolumn{7}{c}{\textbf{Partition}} \\
            \cmidrule(lr){3-9}
            \textbf{Type} & \textbf{Dataset} 
            & \textbf{IID} 
            & \textbf{Non-IID} 
            & \textbf{Dirichlet} 
            & \textbf{LQ-1} 
            & \textbf{LQ-2} 
            & \textbf{LQ-3} 
            & \begin{tabular}[c]{@{}c@{}}\textbf{Natural} \\ \textbf{ID}\end{tabular} \\
            \midrule
            \multirow{4}{*}{\begin{tabular}[c]{@{}l@{}}\textbf{Training} \\ \textbf{rounds}\end{tabular}} 
            & MNIST       & 60  & 60  & 60  & 100  & 60  & 60  & x \\
            & F-MNIST     & 100 & 100 & 150 & 350  & 150 & 150 & x \\
            & CIFAR-10    & 200 & 300 & 400 & 1500 & 500 & 500 & x \\
            & Shakespeare & x   & x   & x   & x    & x   & x   & 200 \\
            \midrule
            \multirow{4}{*}{\begin{tabular}[c]{@{}l@{}}\textbf{Selected} \\ \textbf{rounds}\end{tabular}} 
            & MNIST       & 10 & 10 & 10 & 10  & 10 & 10 & x \\
            & F-MNIST     & 10 & 10 & 15 & 35  & 25 & 15 & x \\
            & CIFAR-10    & 20 & 30 & 40 & 150 & 50 & 50 & x \\
            & Shakespeare & x  & x  & x  & x   & x  & x  & 20 \\
            \bottomrule
        \end{tabular}
    \end{table}

    \subsection{Multi-Seed Evaluation under Computational Constraints}
    Due to the computational cost of simulating thousands/hundreds of federated communication rounds, the number of random seeds varied across experiments. Most experiments were conducted using a single random seed, while some were repeated with two seeds and a small subset with three seeds to assess reproducibility. Additional runs were prioritised for the more challenging data partitioning strategies (Dirichlet, LQ-1, LQ-2, and LQ-3) on CIFAR-10, where variations in client data distributions and partial client participation can affect training outcomes. Accordingly, mean $\pm$ standard deviation is reported for experiments conducted with multiple seeds, while single-seed results are reported without estimates of variability across runs.

    \section{Impact of Variance Reduction and Adaptive Optimisation}\label{Apd:Ablation}

    To rigorously evaluate the individual contributions of our proposed components, we conduct an ablation study comparing the full \textsc{FedAdaVR} algorithm against two dismantled variants: (i) \textsc{FedAdaVR-NoVR}, which removes the variance reduction mechanism while retaining the optimiser (Adagrad for this case); and (ii) \textsc{FedAdaVR-NoOPT}, which removes the server-side adaptive optimiser while retaining variance reduction. All ablation experiments are performed on the CIFAR-10 dataset under the LQ-1 partitioning scheme, strictly adhering to the experimental protocols and hyperparameter settings detailed in Table~\ref{tab:ExperimentalDetails} (with $\eta_s = \eta_c = 0.01$).

    Figure~\ref{fig:ablation_and_time} (left) in the main text illustrates the convergence trajectories of these variants. \textsc{FedAdaVR} demonstrates clear superiority, achieving the lowest loss and highest accuracy. Neither component strictly dominates; rather, they serve complementary and necessary roles. Removing variance reduction (\textsc{FedAdaVR-NoVR}) leads to severe instability and high variance in later rounds, confirming that the SAGA-style variance reduction technique is crucial for mitigating the sampling variance induced by partial client participation. Conversely, removing the adaptive optimiser (\textsc{FedAdaVR-NoOPT}) yields stable but exceedingly slow convergence, confirming that coordinate-wise step adaptation is essential for handling the severe parameter update imbalance caused by heterogeneous data distributions. In summary, variance reduction provides the necessary stability whilst the adaptive optimiser drives convergence speed, and the synergistic combination of both is required to achieve the performance gains reported in Section~\ref{ExperimentalResultsComparisonSOTA}.

	\section{Experimental Results}\label{Apd:ExperimentalResults}
	
	\subsection{Accuracy and Loss Comparison with State-of-the-Art Methods}\label{Apd:ExpResComSOTA}

    This section presents a comprehensive comparative evaluation of \textsc{FedAdaVR} and \textsc{FedAdaVR-Quant} against state-of-the-art methods.

    \subsubsection{Accuracy and loss comparison with the IID data partitioning method (Vision Datasets).}
    
    Figure~\ref{fig:final_evaluation_all_iid} illustrate the accuracy and loss of various state-of-the-art methods alongside the two variants of \textsc{FedAdaVR}. In every instance, \textsc{FedAdaVR} surpasses the existing state-of-the-art algorithms. Although FedYogi performs comparatively well, it remains inferior to \textsc{FedAdaVR} under this straightforward data partitioning.

    \begin{figure}[H]
         \centering
        \includegraphics[scale=0.12]{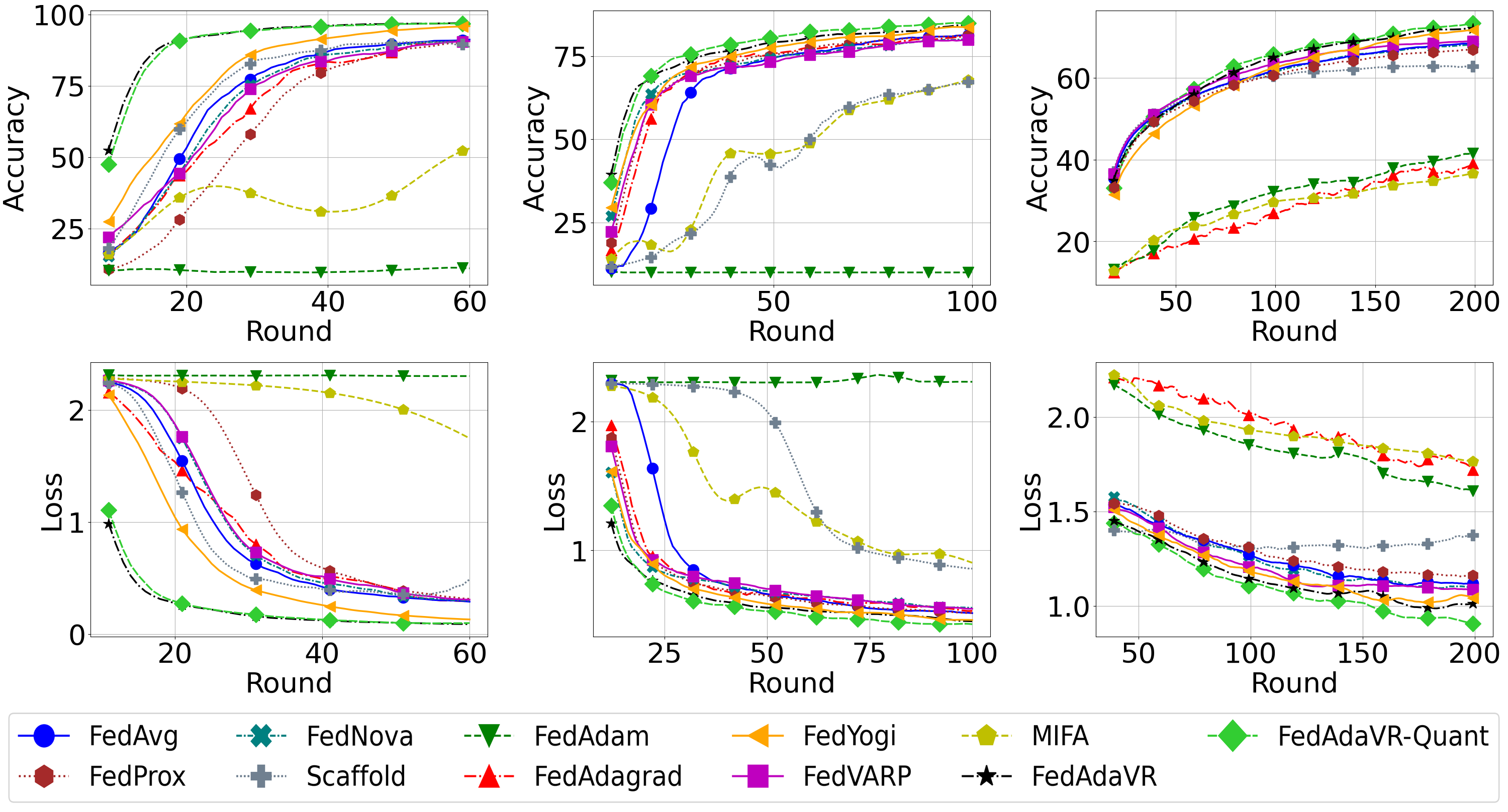}
        \caption{Accuracy and loss metrics under the IID data partitioning setup are presented in the top and bottom rows, respectively, for the MNIST (left), FMNIST (center), and CIFAR-10 (right) datasets.}
        \label{fig:final_evaluation_all_iid}
    \end{figure}
    
    \subsubsection{Accuracy and loss comparison with the IID-Non\_IID data partitioning method (Vision Datasets).} 
    
    From Figure~\ref{fig:final_evaluation_all_iid_noniid}, the accuracy and loss of various state-of-the-art methods alongside the two variants of \textsc{FedAdaVR} under the IID-Non\_IID data partitioning method are presented. \textsc{FedAdaVR} and \textsc{FedAdaVR-Quant} consistently outperforms existing state-of-the-art approaches.
    
    \begin{figure}[h]
         \centering
        \includegraphics[scale=0.12]{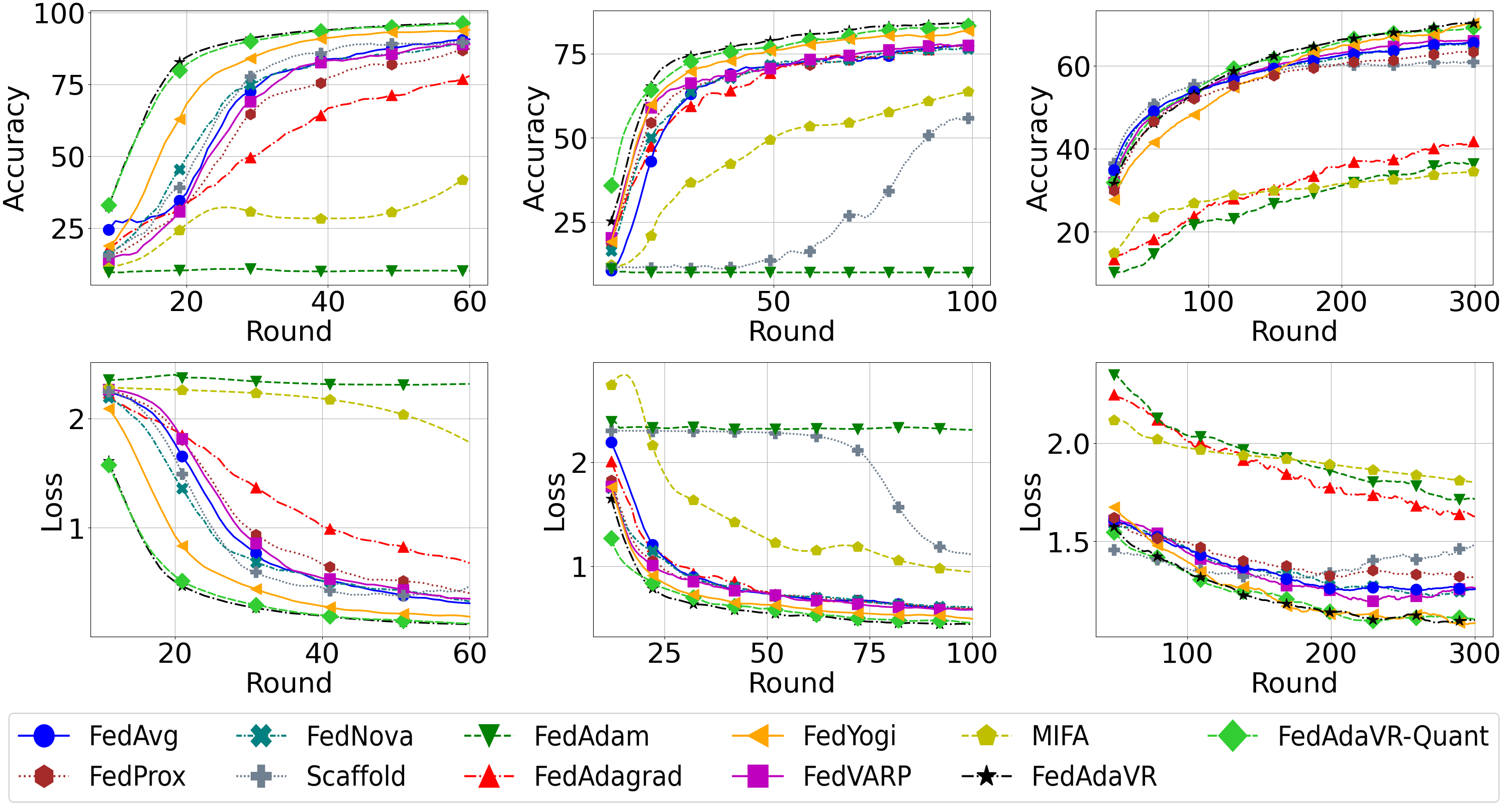}
        \caption{Accuracy and loss metrics under the IID-Non\_IID data partitioning setup are presented in the top and bottom rows, respectively, for the MNIST (left), FMNIST (center), and CIFAR-10 (right) datasets.}
        \label{fig:final_evaluation_all_iid_noniid}
    \end{figure}
    
    \subsubsection{Accuracy and loss comparison with the Dirichlet data partitioning method (Vision Datasets).}
    
    Figure~\ref{fig:final_evaluation_all_dirichlet} demonstrate that the two variants of \textsc{FedAdaVR} outperform existing state-of-the-art methods. While FedYogi performs close to \textsc{FedAdaVR} and \textsc{FedAdaVR-Quant} on the MNIST and FMNIST datasets, it falls behind on the CIFAR-10 dataset.
    
    \begin{figure}[H]
        \centering
        \includegraphics[scale=0.12]{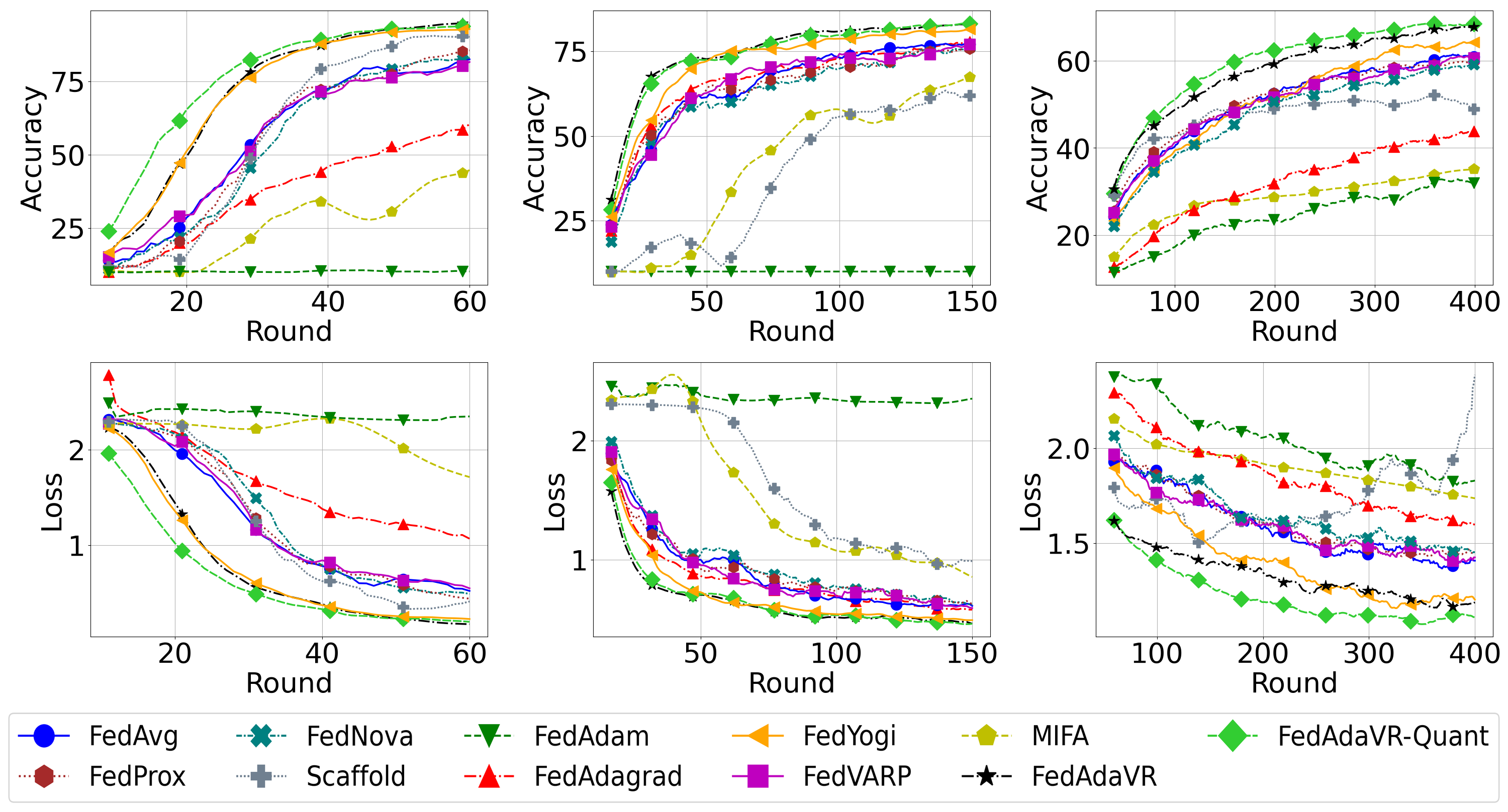}
        \caption{Accuracy and loss metrics under the Dirichlet data partitioning setup are presented in the top and bottom rows, respectively, for the MNIST (left), FMNIST (center), and CIFAR-10 (right) datasets.}
        \label{fig:final_evaluation_all_dirichlet}
    \end{figure}
    
    \subsubsection{Accuracy and loss comparison with LQ-1 data partitioning method (Vision Datasets).}
    
    Figure~\ref{fig:final_evaluation_all_lq1} illustrate that the two variants of \textsc{FedAdaVR} significantly outperform existing state-of-the-art methods on the CIFAR-10 dataset, which represents the most challenging scenario. Although FedVARP performs well, it remains notably behind both \textsc{FedAdaVR-Quant} and \textsc{FedAdaVR}. This clearly demonstrates the superiority of the adaptivity to variance reduction approach in extreme cases. Furthermore, \textsc{FedAdaVR} and \textsc{FedAdaVR-Quant} exhibit faster convergence than other state-of-the-art methods on the MNIST and FMNIST datasets.

    \begin{figure}[H]
        \centering
        \includegraphics[scale=0.12]{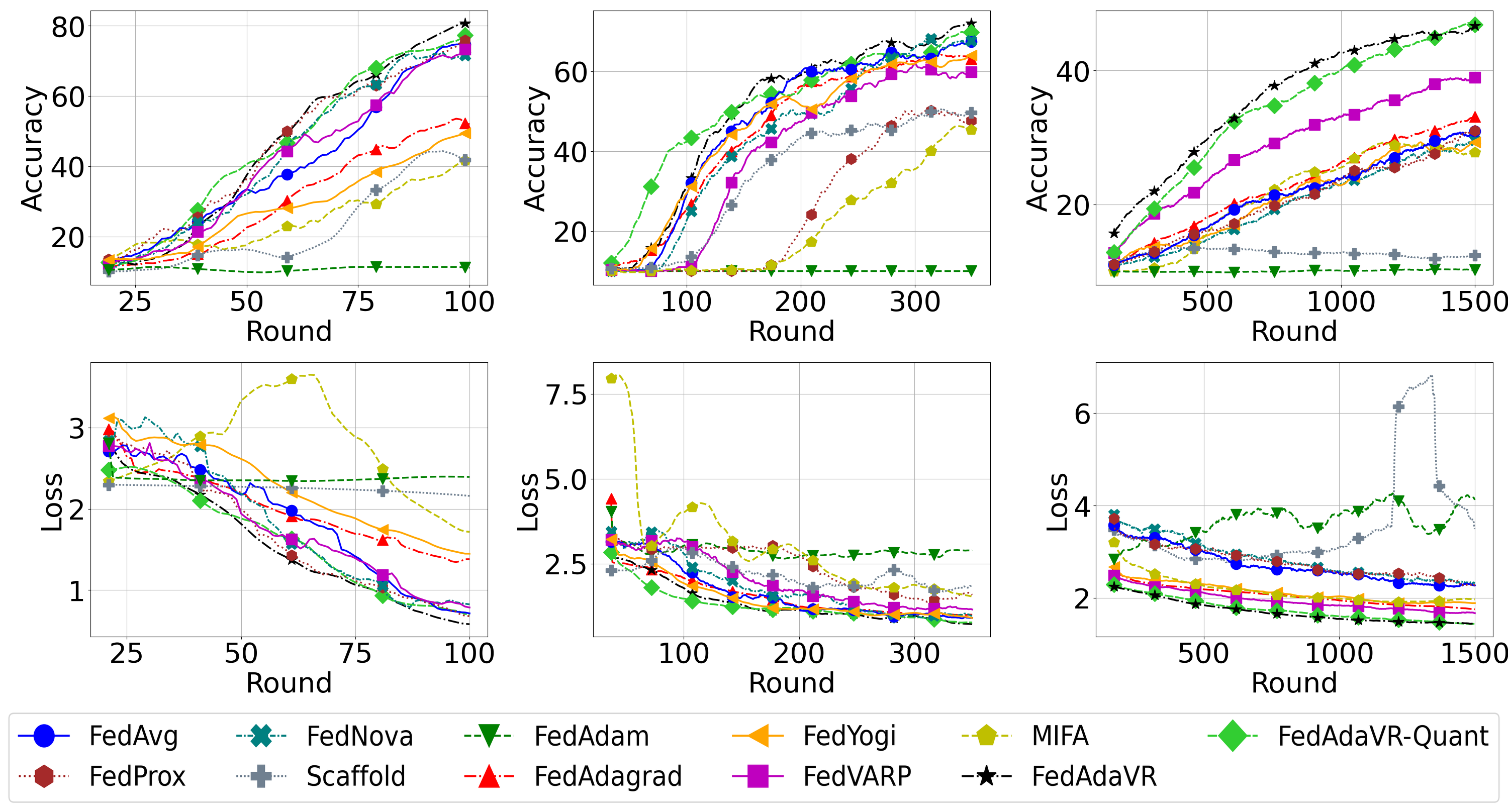}
        \caption{Accuracy and loss metrics under the LQ-1 data partitioning setup are presented in the top and bottom rows, respectively, for the MNIST (left), FMNIST (center), and CIFAR-10 (right) datasets.}
        \label{fig:final_evaluation_all_lq1}
    \end{figure}

    \subsubsection{Accuracy and loss comparison with LQ-2 data partitioning method (Vision Datasets).}
    
    Figure~\ref{fig:final_evaluation_all_lq2} demonstrate that the two variants of \textsc{FedAdaVR} outperform existing state-of-the-art methods. While FedYogi performs comparably to \textsc{FedAdaVR} and \textsc{FedAdaVR-Quant} on the MNIST dataset, all state-of-the-art models lag significantly behind on the CIFAR-10 dataset.

    \begin{figure}[h]
        \centering
        \includegraphics[scale=0.12]{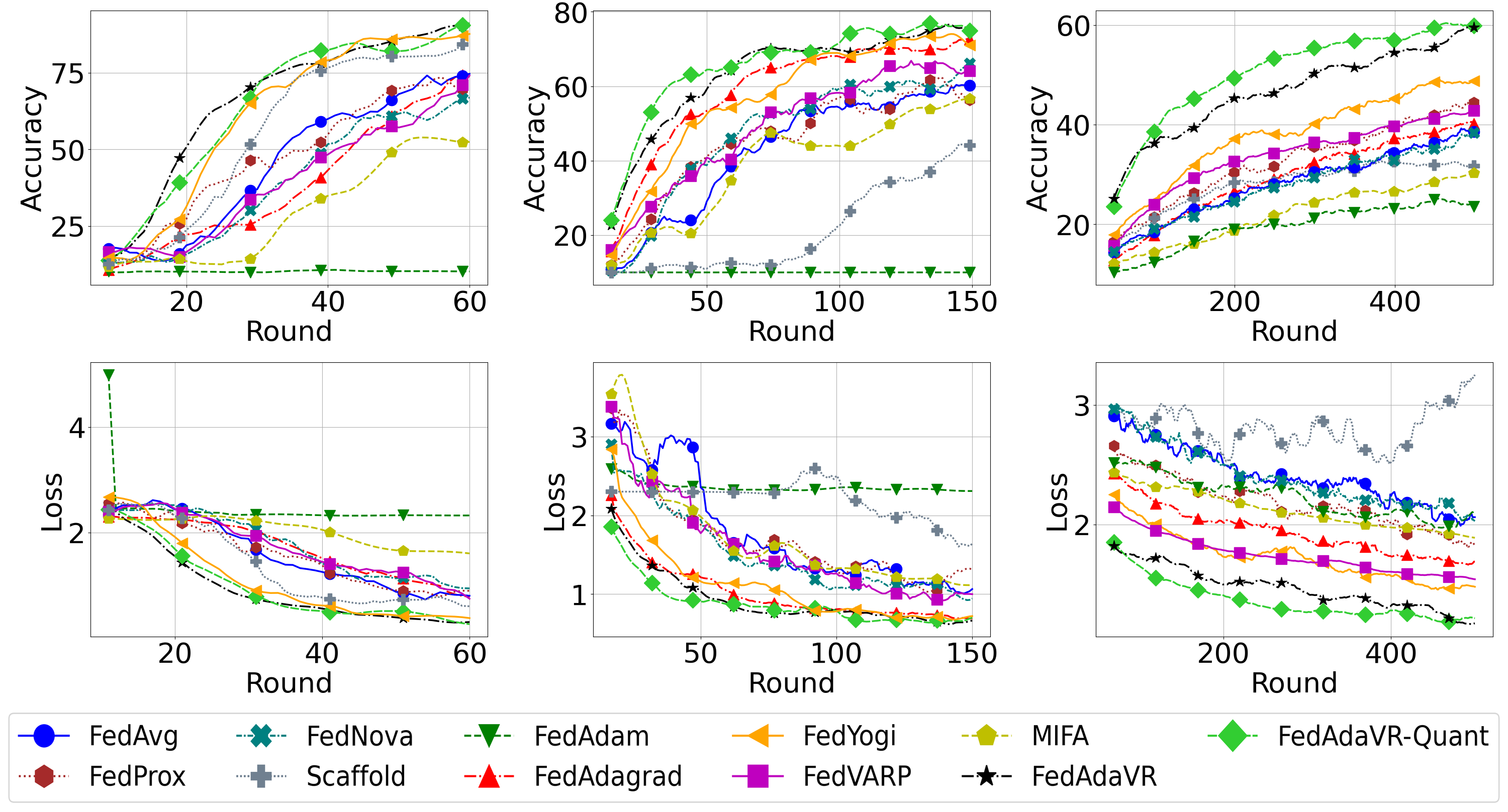}
        \caption{Accuracy and loss metrics under the LQ-2 data partitioning setup are presented in the top and bottom rows, respectively, for the MNIST (left), FMNIST (center), and CIFAR-10 (right) datasets.}
        \label{fig:final_evaluation_all_lq2}
    \end{figure}

    \subsubsection{Accuracy and loss comparison with LQ-3 data partitioning method (Vision Datasets).}
    
    Figures~\ref{fig:final_evaluation_all_lq3} show that \textsc{FedAdaVR} and \textsc{FedAdaVR-Quant} consistently outperform existing state-of-the-art methods. Although FedAdagrad and FedYogi performs comparably on the FMNIST dataset, a clear performance gap is evident in other experiments.

    \begin{figure}[H]
        \centering
        \includegraphics[scale=0.12]{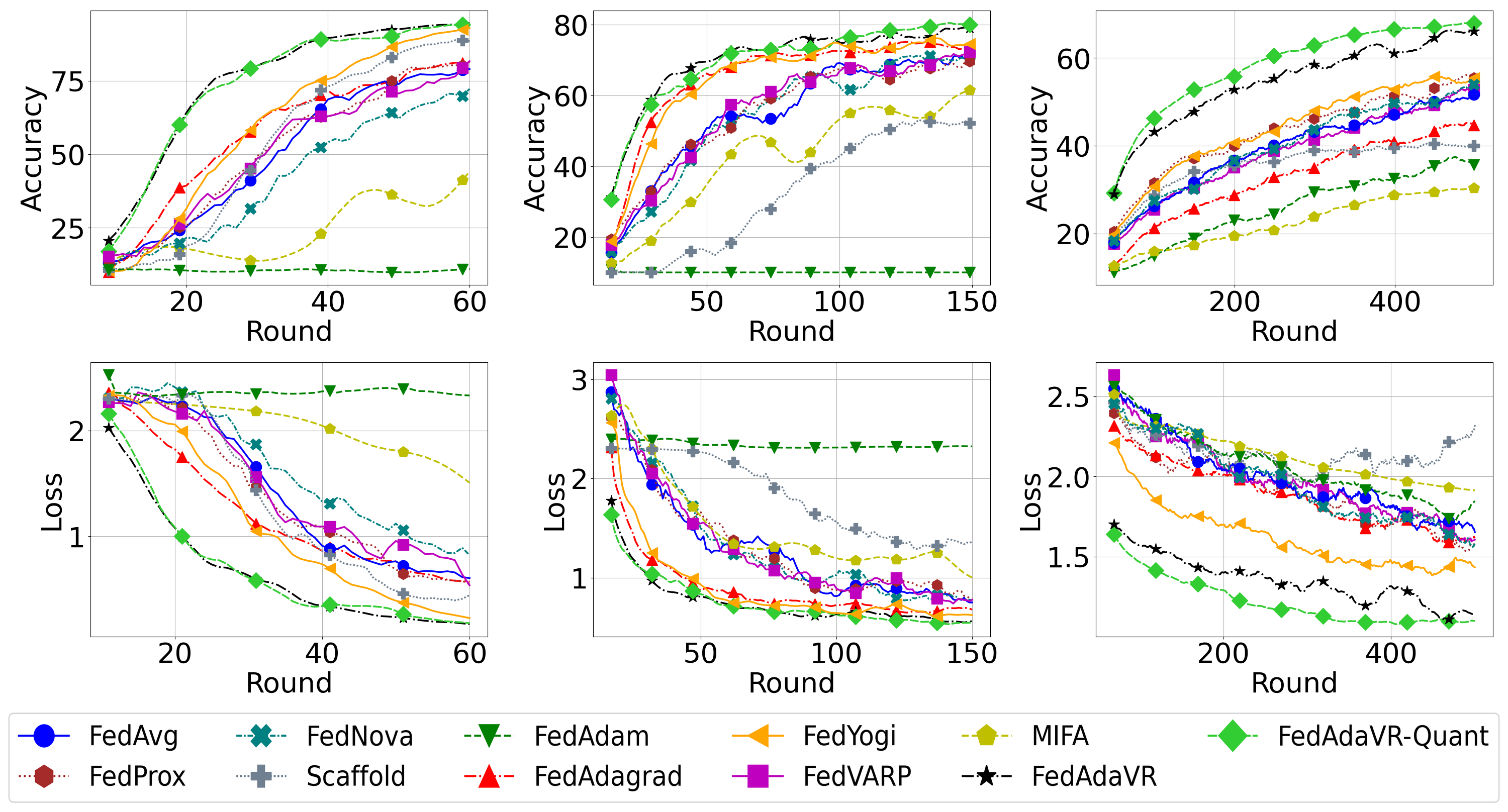}
        \caption{Accuracy and loss metrics under the LQ-3 data partitioning setup are presented in the top and bottom rows, respectively, for the MNIST (left), FMNIST (center), and CIFAR-10 (right) datasets.}
        \label{fig:final_evaluation_all_lq3}
    \end{figure}

    \subsubsection{Accuracy and loss comparison on Shakespeare dataset (Natural ID Partitioner).} \label{Apd:Shakespeare}
    
    As illustrated in Figure \ref{fig:Shakespeare}, both \textsc{FedAdaVR} and its quantised variant, \textsc{FedAdaVR-Quant}, outperform existing state-of-the-art methods on the Shakespeare dataset, with the exception of FedAdagrad. While the margin is narrower here compared to vision experiments, \textsc{FedAdaVR} remains robust; it matches FedAdagrad on this specific task while strictly outperforming it across all other benchmarks.
    
    \begin{figure*}[!h]
        \begin{center}
            \includegraphics[scale=0.18]{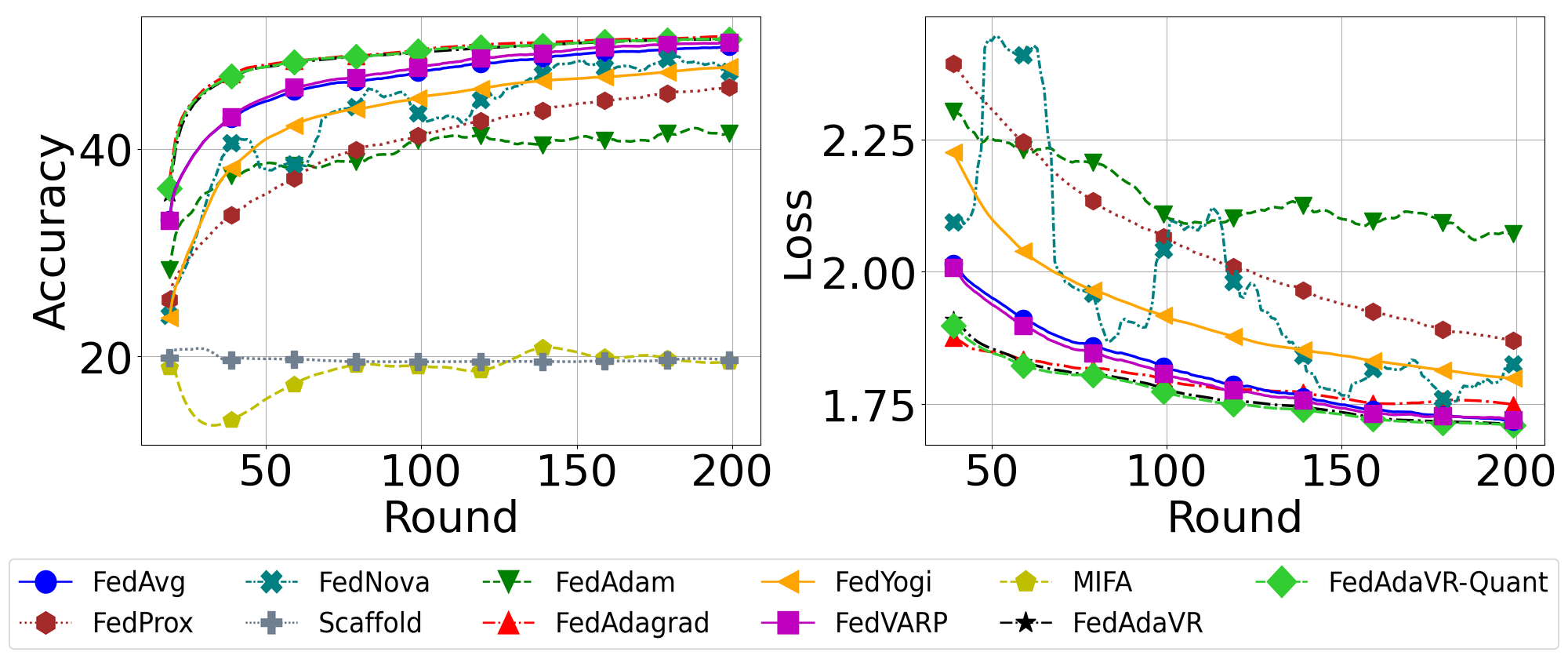}
            \caption{Accuracy and loss comparison on Shakespeare dataset (Natural ID Partitioner). Note that Scaffold and MIFA are omitted from the loss plots due to excessive out-of-range fluctuations.}
            \label{fig:Shakespeare}
        \end{center}
    \end{figure*}

    In summary, \textsc{FedAdaVR} and \textsc{FedAdaVR-Quant} consistently outperform existing state-of-the-art methods. While FedYogi demonstrates comparable performance on the MNIST and FMNIST datasets, and FedAdagrad achieves similar results on the Shakespeare dataset, a pronounced performance gap emerges under the LQ data partitioning scheme on the more complex CIFAR-10 dataset. Given the greater complexity of CIFAR-10 relative to MNIST, FMNIST, and Shakespeare, convergence rates are slower across all algorithms. Real-world datasets are typically even more complex, featuring highly heterogeneous data distributions among clients. Therefore, federated learning algorithms must be designed to address such challenges effectively. Crucially, \textsc{FedAdaVR} does not discard client updates due to staleness; instead, it adheres to a SAGA-like structure that assumes virtual full participation by retaining all historical updates. While implementing a mechanism to prune excessively stale updates could theoretically enhance efficiency, this investigation remains outside the scope of our current research. Both proposed algorithms, \textsc{FedAdaVR} and \textsc{FedAdaVR-Quant}, are deemed well-suited to real-world federated learning scenarios.

    \textbf{Numerical Instability of SCAFFOLD and MIFA.} During hyperparameter tuning, we observed that SCAFFOLD and MIFA were sensitive to learning-rate changes under the extremely low client participation rates (1--2\%) considered in our experiments. Attempts to adjust their learning rates did not consistently improve convergence and, in some cases, led to numerical instability, including exploding gradients and \texttt{NaN} training losses. This behaviour may be related to the infrequent client updates, which can make the control variates in SCAFFOLD and the stored historical gradients in MIFA increasingly outdated. Therefore, we report the best stable configurations obtained for these baselines under the evaluated settings.

    \section{Detailed Comparison of \textsc{FedAdaVR} and \textsc{FedAdaVR-Quant} 
    Across Adaptive Optimisers}\label{Apd:ComOwn}

    This appendix provides the complete experimental results summarised in Section~\ref{ExperimentalResultsComparisonOurs} of the main paper. Tables~\ref{tab:shakespeare_accuracy_com}, \ref{tab:MNIST_accuracy_com}, \ref{tab:FashiontMNIST_accuracy_com}, and \ref{tab:CIFAR10_accuracy_com} present a comparative analysis of \textsc{FedAdaVR} and \textsc{FedAdaVR-Quant} across various adaptive optimisers and datasets. FedVARP is also included for reference, as it similarly utilises a server learning rate $\eta_s$. Bold text denotes the optimiser achieving the highest accuracy for a given configuration. The best-performing configurations from these tables are used to generate the curves shown in Figures~\ref{fig:final_evaluation_accuracy_all}, and \ref{fig:final_evaluation_all_iid}--\ref{fig:Shakespeare}.
    
    Lamb consistently attains the best results on the CIFAR-10 dataset, whereas no single optimiser emerges as a definitive frontrunner across the other datasets. Consequently, Lamb appears to be the preferred choice for more complex tasks such as CIFAR-10. While Adagrad achieves the strongest performance on the Shakespeare dataset, it demonstrates competitive but largely sub-optimal performance across the vision datasets, with the exception of a single setting where it secured the highest accuracy.

    \subsection{Comparison Between Different Quantisation Methods}\label{Apd:ComQuantAccuracy}

    In our experiments, client--server communication retained the original FP32 format. However, for server-side storage of client updates, we evaluated four precision levels: FP32 (full precision), FP16 (half-precision), Int8, and Int4. While storing client states requires $\mathcal{O}(Nd)$ memory at full precision, employing FP16, Int8, or Int4 quantisation reduces this footprint by 50\%, 75\%, and 87.5\%, respectively.
    
    As detailed in the main text, the quantisation of stored updates has a negligible impact on convergence. Consequently, the quantised variants of \textsc{FedAdaVR} achieve performance comparable to the full-precision baseline, as illustrated in Figure~\ref{fig:quantised_comparison}. These findings were validated on the CIFAR-10 dataset under the LQ-1 partition, with both server and client learning rates fixed at $\eta_s = \eta_c = 0.001$.
        
    Given that cross-device settings may involve thousands of resource-constrained clients, the $\mathcal{O}(Nd)$ storage complexity is a critical practical consideration. Table~\ref{tab:memory_req_per_training} details the estimated server-side memory requirements for each precision variant across a range of representative model architectures, demonstrating the deployment feasibility of \textsc{FedAdaVR-Quant} in large-scale cross-device settings.

    \begin{figure*}[!h]
        \begin{center}
            \includegraphics[scale=0.18]{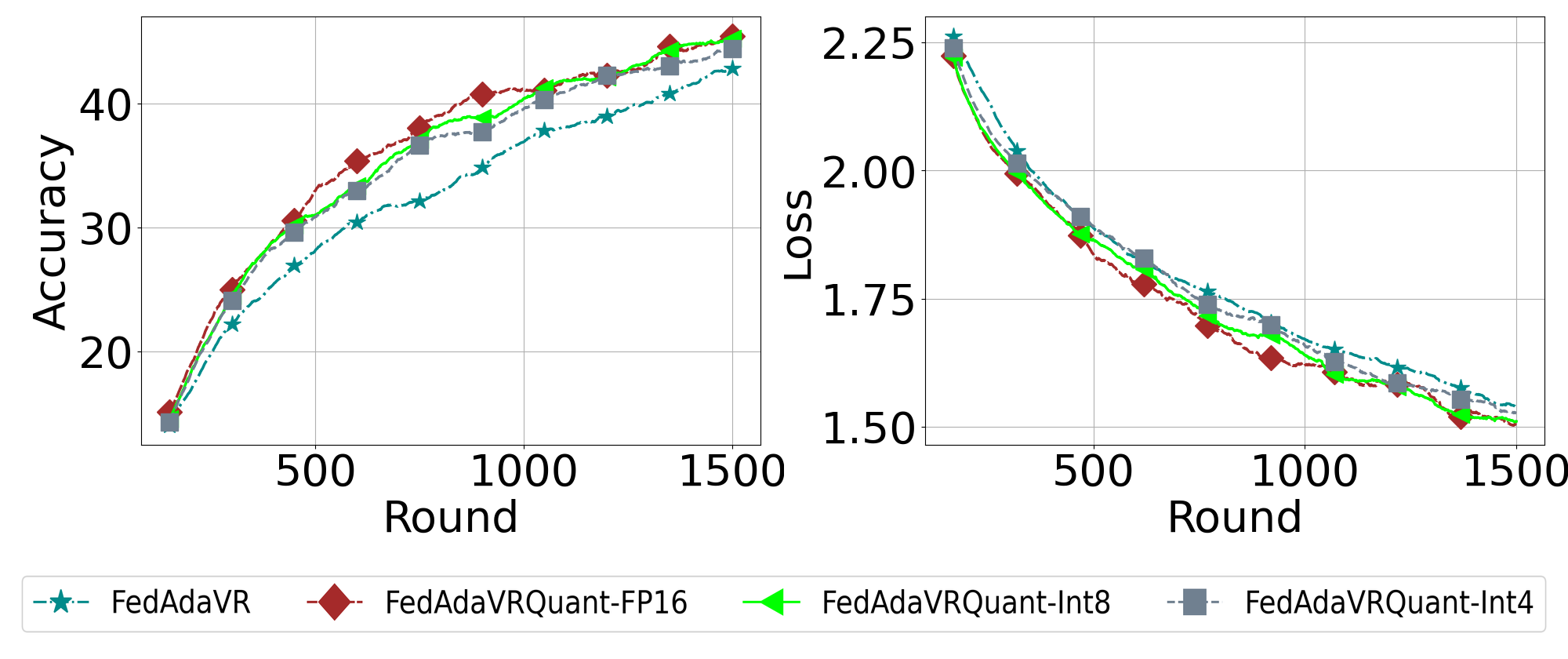}
            \caption{Comparison of model accuracy and loss across four quantisation precisions 
            (FP32, FP16, Int8, Int4) for \textsc{FedAdaVR} on the CIFAR-10 dataset under 
            LQ-1 partitioning.}
            \label{fig:quantised_comparison}
            \vskip -0.2in
        \end{center}
    \end{figure*}

    \begin{table}[ht]
        \caption{Estimated server-side memory requirements (GB) for storing client states in \textsc{FedAdaVR} across representative model architectures and client population sizes $N$ (in thousands).}
        \label{tab:memory_req_per_training}
        \centering
        \setlength{\tabcolsep}{4.5pt} 
        \begin{tabular}{lcccccc}
            \toprule
            \multirow{2}{*}{\textbf{Local Model}} 
            & \multirow{2}{*}{\textbf{\makecell{Parameters \\ (Millions)}}} 
            & \multirow{2}{*}{\textbf{\makecell{\# of Clients \\ (Thousands)}}}  
            & \multicolumn{4}{c}{\textbf{Approximate Memory (GB)}} \\
            \cmidrule(lr){4-7}
            & & 
            & \textbf{FP32} 
            & \textbf{FP16} 
            & \textbf{Int8} 
            & \textbf{Int4} \\
            \midrule
            \multirow{3}{*}{\textbf{\makecell[l]{SqueezeNet 1.1\\ \citep{Iandola2016}}}} 
            & \multirow{3}{*}{\textbf{1.2}} 
            & 1   & 4.6   & 2.3   & 1.15  & 0.58  \\
            & & 10  & 46.0  & 23.0  & 11.5  & 5.8   \\
            & & 100 & 460   & 230   & 115   & 58    \\
            \midrule
            \multirow{3}{*}{\textbf{\makecell[l]{ShuffleNet V2\\ \citep{Ma2018}}}} 
            & \multirow{3}{*}{\textbf{2.3}} 
            & 1   & 8.8   & 4.4   & 2.2   & 1.1   \\
            & & 10  & 88.0  & 44.0  & 22.0  & 11.0  \\
            & & 100 & 880   & 440   & 220   & 110   \\
            \midrule
            \multirow{3}{*}{\textbf{\makecell[l]{MobileNetV3-Small\\ \citep{Howard2019}}}} 
            & \multirow{3}{*}{\textbf{2.5}} 
            & 1   & 9.5   & 4.8   & 2.4   & 1.2   \\
            & & 10  & 95.0  & 48.0  & 24.0  & 12.0  \\
            & & 100 & 950   & 480   & 240   & 120   \\
            \midrule
            \multirow{3}{*}{\textbf{\makecell[l]{MobileNetV2\\ \citep{Sandler2018}}}} 
            & \multirow{3}{*}{\textbf{3.4}} 
            & 1   & 13.0  & 6.5   & 3.25  & 1.63  \\
            & & 10  & 130   & 65.0  & 32.5  & 16.3  \\
            & & 100 & 1300  & 650   & 325   & 163   \\
            \midrule
            \multirow{3}{*}{\textbf{\makecell[l]{MNASNet 1.0\\ \citep{Tan2018}}}} 
            & \multirow{3}{*}{\textbf{3.9}} 
            & 1   & 14.9  & 7.45  & 3.73  & 1.86  \\
            & & 10  & 149   & 74.5  & 37.3  & 18.6  \\
            & & 100 & 1490  & 745   & 373   & 186   \\
            \midrule
            \multirow{3}{*}{\textbf{\makecell[l]{MobileNetV1\\ \citep{Howard2017}}}} 
            & \multirow{3}{*}{\textbf{4.2}} 
            & 1   & 16.0  & 8.0   & 4.0   & 2.0   \\
            & & 10  & 160   & 80.0  & 40.0  & 20.0  \\
            & & 100 & 1600  & 800   & 400   & 200   \\
            \midrule
            \multirow{3}{*}{\textbf{\makecell[l]{EffNet-Lite0\\ \citep{Tan2019}}}} 
            & \multirow{3}{*}{\textbf{4.7}} 
            & 1   & 17.9  & 8.95  & 4.48  & 2.24  \\
            & & 10  & 179   & 89.5  & 44.8  & 22.4  \\
            & & 100 & 1790  & 895   & 448   & 224   \\
            \midrule
            \multirow{3}{*}{\textbf{\makecell[l]{GoogLeNet\\ \citep{Szegedy2015}}}} 
            & \multirow{3}{*}{\textbf{6.6}} 
            & 1   & 25.2  & 12.6  & 6.3   & 3.15  \\
            & & 10  & 252   & 126   & 63.0  & 31.5  \\
            & & 100 & 2520  & 1260  & 630   & 315   \\
            \midrule
            \multirow{3}{*}{\textbf{\makecell[l]{ResNet-18\\ \citep{resnet}}}} 
            & \multirow{3}{*}{\textbf{11.7}} 
            & 1   & 44.6  & 22.3  & 11.2  & 5.58  \\
            & & 10  & 446   & 223   & 112   & 55.8  \\
            & & 100 & 4460  & 2230  & 1115  & 558   \\
            \midrule
            \multirow{3}{*}{\textbf{\makecell[l]{ResNet-50\\ \citep{resnet}}}} 
            & \multirow{3}{*}{\textbf{25.6}} 
            & 1   & 97.7  & 48.9  & 24.4  & 12.2  \\
            & & 10  & 977   & 489   & 244   & 122   \\
            & & 100 & 9770  & 4890  & 2440  & 1220  \\
            \midrule
            \multirow{3}{*}{\textbf{\makecell[l]{ConvNeXt-Tiny\\ \citep{Liu2022}}}} 
            & \multirow{3}{*}{\textbf{28.6}} 
            & 1   & 109   & 54.5  & 27.3  & 13.6  \\
            & & 10  & 1090  & 545   & 273   & 136   \\
            & & 100 & 10900 & 5450  & 2730  & 1360  \\
            \bottomrule
        \end{tabular}
    \end{table}

    \subsection{Server Learning Rate and Optimiser Selection}\label{Apd:DisOwnOpt}
    
    Experimental findings indicate that for a given optimiser, the performance difference across server learning rates is generally marginal. On the MNIST and FMNIST datasets, the LeNet-5 model achieves superior results with a higher server learning rate in most scenarios. Conversely, for ResNet-18 on CIFAR-10, a lower server learning rate facilitates faster convergence, particularly under the challenging LQ-1 partition. These observations underscore the critical sensitivity of algorithm performance to the choice of server learning rate.
    
    All optimisers integrated into the proposed algorithms show competitive performance across tasks. The Lamb optimiser achieves the highest accuracy on the more complex CIFAR-10 dataset but converges slowly on MNIST and FMNIST. Given the black-box characteristics of federated learning, selecting an optimal optimiser a priori remains non-trivial. Nevertheless, all optimisers employed in \textsc{FedAdaVR} and \textsc{FedAdaVR-Quant} consistently outperform existing baselines across all evaluated settings. Although five commonly used optimisers were evaluated in this work, other contemporary or related variants could be readily incorporated into the proposed framework.

    \section{Computational, Latency, and Communication Comparison Against Baselines}
    \label{Apd:ComLatComSOTA}
    
    This appendix provides the full system-side analysis summarised in Section~\ref{ExperimentalResultsComparisonSOTA} of the main paper. Computational latency in federated learning is governed by diverse factors, including hardware constraints, algorithmic design, and client participation density. Although all primary experiments were conducted on a High-Performance Computing (HPC) cluster, a workstation equipped with a 13th Gen Intel Core i9-13900H CPU (2.60~GHz), 64~GB of RAM, and an NVIDIA GeForce RTX~4080 Laptop GPU (12~GB VRAM) was utilised to measure the average time per round presented in Table \ref{tab:system_wallclock}. This standardised approach was necessary because the HPC cluster consists of heterogeneous nodes with varying capabilities, and the algorithms may not be allocated identical hardware configurations during execution. We organise the analysis into three parts: algorithmic complexity (Table~\ref{tab:system_complexity}), empirical wall-clock times (Table~\ref{tab:system_wallclock}), and communication round efficiency (Table~\ref{tab:ComComSOTA}).
    
    \textsc{FedAdaVR} is designed to impose no additional computational burden on clients; consequently, the client-side execution path is identical to that of FedAvg. However, the server incurs additional overhead to mitigate variance: the calculation of the correction term $\mathbf{r}^{(t)}$ requires aggregation of historical updates from all clients. As our algorithm explicitly accommodates dynamic client participation by allowing clients to join or depart arbitrarily, managing these historical states inevitably increases per-round computation time. It is worth noting, however, that the per-round server computation is strictly $\mathcal{O}(|\mathcal{S}^{(t)}|d)$ rather than $\mathcal{O}(Nd)$, as the variance reduction correction is maintained as a running sum and only the active clients' contributions are updated each round.
    
    \subsection{Algorithmic Complexity Comparison}
    
    Table~\ref{tab:system_complexity} summarises the theoretical complexity of each algorithm in terms of communication cost per round, server-side memory requirements, and client-side computation overhead. $d$ denotes model size, $N$ total clients, $K$ local epochs, and $d_{\text{quant}}$ the quantised model size. \textsc{FedAdaVR} maintains the same lightweight client footprint as FedAvg, offloading all variance reduction and adaptive optimisation entirely to the server. Compared to SCAFFOLD, FedProx, and FedNova, which introduce additional client-side computation in the form of control variates, proximal terms, and gradient normalisation respectively, \textsc{FedAdaVR} imposes no extra burden on resource-constrained client devices. Among all evaluated methods, only SCAFFOLD incurs additional communication overhead by transmitting dual parameters (model weights and control variates) during both uplink and downlink.

    \begin{table}[ht]
        \caption{Theoretical complexity comparison of all evaluated algorithms in terms of communication cost per round, server memory, and client computation. $d$ denotes model size, $N$ total clients, $K$ local epochs, and $d_{\text{quant}}$ the quantised model size.}
        \label{tab:system_complexity}
        \centering
        \begin{tabular}{llcl}
            \toprule
            \textbf{Algorithm} 
            & \textbf{\makecell[l]{Communication\\Cost}} 
            & \textbf{\makecell[l]{Server\\Memory}} 
            & \textbf{\makecell[l]{Client\\Computation}} \\
            \midrule
            FedAvg                 & $\mathcal{O}(d)$  & $\mathcal{O}(d)$                 & $\mathcal{O}(K)$ \\
            FedAdam                & $\mathcal{O}(d)$  & $\mathcal{O}(d)$                 & $\mathcal{O}(K)$ \\
            FedAdagrad             & $\mathcal{O}(d)$  & $\mathcal{O}(d)$                 & $\mathcal{O}(K)$ \\
            FedYogi                & $\mathcal{O}(d)$  & $\mathcal{O}(d)$                 & $\mathcal{O}(K)$ \\
            FedProx                & $\mathcal{O}(d)$  & $\mathcal{O}(d)$                 & $\mathcal{O}(K)$ + Proximal Term \\
            SCAFFOLD               & $\mathcal{O}(2d)$ & $\mathcal{O}(d)$                 & $\mathcal{O}(K)$ + Control Variates \\
            FedNova                & $\mathcal{O}(d)$  & $\mathcal{O}(d)$                 & $\mathcal{O}(K)$ + Gradient Normalisation \\
            MIFA                   & $\mathcal{O}(d)$  & $\mathcal{O}(Nd)$                & $\mathcal{O}(K)$ \\
            FedVARP                & $\mathcal{O}(d)$  & $\mathcal{O}(Nd)$                & $\mathcal{O}(K)$ \\
            \textsc{FedAdaVR-Quant}& $\mathcal{O}(d)$  & $\mathcal{O}(Nd_{\text{quant}})$ & $\mathcal{O}(K)$ \\
            \textsc{FedAdaVR}      & $\mathcal{O}(d)$  & $\mathcal{O}(Nd)$                & $\mathcal{O}(K)$ \\
            \bottomrule
        \end{tabular}
    \end{table}

    \subsection{Wall-Clock Time Comparison}

    Table~\ref{tab:system_wallclock} reports empirical timing results across all evaluated algorithms. For \textsc{FedAdaVR}, we additionally break down per-round latency by server-side optimiser to illustrate the marginal overhead difference between optimiser choices. Whilst \textsc{FedAdaVR} incurs a minor per-round latency increase over FedAvg, this overhead is offset by a drastically reduced number of rounds required to reach target accuracy. \textsc{FedAdaVR} (Adagrad) reaches 45\% accuracy in 163,689 seconds compared to 232,336 seconds for FedAvg, a reduction of approximately 30\%. Hyphens indicate the algorithm did not reach the target accuracy within the allocated maximum rounds. While the server-side latency could be further optimised for scenarios with fixed client sets, such implementation details are beyond the scope of this work.

    \begin{table}[ht]
        \caption{Empirical wall-clock time comparison on CIFAR-10 under identical hardware. 
        Per-round time is measured on the IID partition. End-to-end time to reach accuracy 
        thresholds is measured on the LQ-1 partition at a 2\% client participation rate. 
        Times are reported in seconds. Hyphens indicate the algorithm did not reach the 
        target accuracy within the allocated maximum rounds.}
        \label{tab:system_wallclock}
        \centering
        \setlength{\tabcolsep}{8pt}
        \begin{tabular}{l c cccc}
            \toprule
            \multirow{2}{*}{\textbf{Algorithm}} 
            & \multirow{2}{*}{\textbf{\makecell{Time/Round\\(s)}}} 
            & \multicolumn{4}{c}{\textbf{Time to Reach Target Accuracy (s)}} \\
            \cmidrule(lr){3-6}
            & 
            & \textbf{30\%} 
            & \textbf{35\%} 
            & \textbf{40\%} 
            & \textbf{45\%} \\
            \midrule
            FedAvg                 & 166.43 & 103,020         & 136,972         & 220,353          & 232,336 \\
            FedAdam                & 168.32 & -               & -               & -                & -       \\
            FedAdagrad             & 167.89 & -               & 229,673         & -                & -       \\
            FedYogi                & 168.54 & -               & 249,945         & -                & -       \\
            FedProx                & 170.57 & 123,834         & 157,095         & 203,490          & -       \\
            SCAFFOLD               & 184.89 & -               & -               & -                & -       \\
            FedNova                & 168.65 & 123,789         & 182,479         & 228,015          & 251,457 \\
            MIFA                   & 172.58 & 174,651         & -               & -                & -       \\
            FedVARP                & 175.27 & 74,840          & 99,203          & 153,887          & 242,749 \\
            \midrule
            \textbf{\textsc{FedAdaVR}}      & 178.31 & \textbf{52,780} & \textbf{75,069} & \textbf{117,863} & \textbf{163,689} \\
            \textsc{FedAdaVR-Quant}& 178.31 & 61,160          & 98,784          & 143,183          & 199,351 \\
            \bottomrule
        \end{tabular}
    \end{table}

    \subsection{Communication Round Efficiency}\label{Apd:ComComSOTA}
    
    Due to their superior convergence rates, \textsc{FedAdaVR} and \textsc{FedAdaVR-Quant} require significantly fewer total communication rounds compared to all baselines. Table~\ref{tab:ComComSOTA} details the number of federated rounds required by each algorithm to achieve specific target accuracy thresholds on CIFAR-10 (LQ-1), along with the normalised communication cost (Norm.\ Rounds) relative to \textsc{FedAdaVR}. A value of 2.0 indicates the method required twice as many rounds as \textsc{FedAdaVR} to reach the same threshold. The symbol `x' indicates the target accuracy was not reached within the allocated maximum rounds.

    \begin{table}[ht]
        \caption{Communication efficiency analysis on CIFAR-10 (LQ-1 partition). The table 
        reports the total communication rounds required to reach specific accuracy thresholds, 
        and the normalised communication cost relative to \textsc{FedAdaVR} (Norm.\ Rounds). 
        A value of 2.0 indicates twice as many rounds as \textsc{FedAdaVR}. `x' indicates 
        the target accuracy was not reached within the allocated maximum rounds.}
        \label{tab:ComComSOTA}
        \centering
        \setlength{\tabcolsep}{4.5pt} 
        \begin{tabular}{l l cccccc}
            \toprule
            \multirow{2}{*}{\textbf{Algorithm}} 
            & \multirow{2}{*}{\textbf{Criteria}}  
            & \multicolumn{6}{c}{\textbf{Accuracy Threshold}} \\
            \cmidrule(lr){3-8}
            & & \textbf{20\%} & \textbf{25\%} & \textbf{30\%} 
              & \textbf{35\%} & \textbf{40\%} & \textbf{45\%} \\
            \midrule
            \multirow{2}{*}{FedAvg}
            & Rounds        & 335 & 488  & 619  & 823  & 1324 & 1396 \\
            & Norm.\ Rounds & 3.3 & 2.3  & 2.1  & 2.0  & 2.0  & 1.5  \\
            \midrule
            \multirow{2}{*}{FedAdagrad}
            & Rounds        & 355 & 646  & 988  & 1368 & x    & x    \\
            & Norm.\ Rounds & 3.5 & 3.1  & 3.3  & 3.2  & x    & x    \\
            \midrule
            \multirow{2}{*}{FedAdam}
            & Rounds        & x   & x    & x    & x    & x    & x    \\
            & Norm.\ Rounds & x   & x    & x    & x    & x    & x    \\
            \midrule
            \multirow{2}{*}{FedYogi}
            & Rounds        & 478 & 727  & 1066 & 1483 & x    & x    \\
            & Norm.\ Rounds & 4.7 & 3.5  & 3.6  & 3.5  & x    & x    \\
            \midrule
            \multirow{2}{*}{FedProx}
            & Rounds        & 338 & 484  & 726  & 921  & 1193 & x    \\
            & Norm.\ Rounds & 3.3 & 2.3  & 2.5  & 2.2  & 1.8  & x    \\
            \midrule
            \multirow{2}{*}{SCAFFOLD}
            & Rounds        & x   & x    & x    & x    & x    & x    \\
            & Norm.\ Rounds & x   & x    & x    & x    & x    & x    \\
            \midrule
            \multirow{2}{*}{FedNova}
            & Rounds        & 338 & 612  & 734  & 1082 & 1352 & 1491 \\
            & Norm.\ Rounds & 3.3 & 2.9  & 2.5  & 2.6  & 2.0  & 1.6  \\
            \midrule
            \multirow{2}{*}{MIFA}
            & Rounds        & 444 & 596  & 1012 & x    & x    & x    \\
            & Norm.\ Rounds & 4.4 & 2.9  & 3.4  & x    & x    & x    \\
            \midrule
            \multirow{2}{*}{FedVARP}
            & Rounds        & 116 & 204  & 427  & 566  & 878  & 1385 \\
            & Norm.\ Rounds & 1.1 & 1.0  & 1.4  & 1.3  & 1.3  & 1.5  \\
            \midrule
            \multirow{2}{*}{\textsc{FedAdaVR-Quant}}
            & Rounds        & 189 & 284  & 343  & 554  & 803  & 1118 \\
            & Norm.\ Rounds & 1.9 & 1.4  & 1.2  & 1.3  & 1.2  & 1.2  \\
            \midrule
            \multirow{2}{*}{\textbf{\textsc{FedAdaVR}}}
            & Rounds        & \textbf{101} & \textbf{209} & \textbf{296} 
                                           & \textbf{421} & \textbf{661} & \textbf{918} \\
            & Norm.\ Rounds & 1.0 & 1.0  & 1.0  & 1.0  & 1.0  & 1.0  \\
            \bottomrule
        \end{tabular}
    \end{table}

    \begin{table}[ht]
        \caption{Accuracy comparison of FedVARP, \textsc{FedAdaVR} and \textsc{FedAdaVR-Quant} algorithms on the Shakespeare dataset (Server learning rate: \(\eta_s\)).}
        \label{tab:shakespeare_accuracy_com}
        \centering
        \setlength{\tabcolsep}{8pt}
        \begin{tabular}{ll ccccc}
            \toprule
            \multirow{2}{*}{\textbf{Algorithm}} & \multirow{2}{*}{\textbf{Optimiser}}  
            & \multicolumn{5}{c}{\textbf{$\eta_s$}} \\
            \cmidrule(lr){3-7}
            &  
            & \textbf{0.001} 
            & \textbf{0.005} 
            & \textbf{0.01} 
            & \textbf{0.1} 
            & \textbf{1.0} \\
            \midrule
            \textbf{FedVARP}
            & x & x & x & 43.038 & 49.434 & \textbf{50.288} \\
            \midrule
            \multirow{5}{*}{\textbf{\textsc{FedAdaVR}}}
            & Adabelief & 45.271 & 49.453 & 49.514 & x & x \\
            & Adam      & 35.733 & 48.541 & 49.343 & x & x \\
            & Adagrad   & 38.379 & 50.048 & \textbf{50.577} & x & x \\
            & Yogi      & 27.921 & 48.440 & 49.549 & x & x \\
            & Lamb      & 42.057 & 41.034 & 31.383 & x & x \\
            \midrule
            \multirow{5}{*}{\textbf{\textsc{FedAdaVR-Quant}}}
            & Adabelief & 45.295 & 49.273 & 49.441 & x & x \\
            & Adam      & 28.589 & 48.604 & 49.468 & x & x \\
            & Adagrad   & 41.492 & 50.065 & \textbf{50.611} & x & x \\
            & Yogi      & 43.501 & 48.470 & 49.507 & x & x \\
            & Lamb      & 41.799 & 41.201 & 41.556 & x & x \\            
            \bottomrule
        \end{tabular}
    \end{table}

    \begin{table}[ht]
        \caption{Accuracy comparison of FedVARP, \textsc{FedAdaVR} and \textsc{FedAdaVR-Quant} algorithms on the MNIST dataset (Server learning rate: \(\eta_s\)).}
        \label{tab:MNIST_accuracy_com}
        \centering
        \small 
        \setlength{\tabcolsep}{4pt} 
        \begin{tabular}{ll l cccccc}
            \toprule
            & & & \multicolumn{6}{c}{\textbf{Partition}} \\
            \cmidrule(lr){4-9}
            \textbf{Algorithm} 
            & \textbf{Optimiser} 
            & $\eta_s$ 
            & \textbf{IID} 
            & \textbf{Non\_IID} 
            & \textbf{Dir.} 
            & \textbf{LQ-1} 
            & \textbf{LQ-2} 
            & \textbf{LQ-3} \\
            \midrule
            \textbf{FedAvg} & - & - & 91.06 & 90.44 & 81.75 & 75.26 & 74.04 & 78.78 \\
            \midrule
            \textbf{FedProx} & - & - & 90.34 & 86.74 & 85.24 & 78.81 & 69.61 & 80.32 \\
            \midrule
            \textbf{SCAFFOLD} & - & - & 89.63 & 89.59 & 90.31 & 38.45 & 84.40 & 88.79 \\
            \midrule
            \textbf{FedNova} & - & - & 90.85 & 89.15 & 82.44 & 70.85 & 66.69 & 69.79 \\
            \midrule
            \textbf{FedAdam} & - & - & 11.24 & 10.10 & 10.39 & 11.04 & 10.21 & 10.89 \\
            \midrule
            \textbf{FedAdagrad} & - & - & 90.85 & 76.79 & 58.5 & 55.38 & 74.36 & 81.15 \\
            \midrule
            \textbf{FedYogi} & - & - & 95.89 & 93.58 & 92.59 & 54.04 & 87.15 & 92.45 \\
            \midrule
            \textbf{MIFA} & - & - & 52.26 & 41.77 & 43.89 & 42.21 & 52.43 & 41.25 \\
            \midrule
            \multirow{3}{*}{\textbf{FedVARP}} 
            & \multirow{3}{*}{-} 
            & 0.01  & 11.94 & 12.37 & 10.89 & 7.11  & 13.15 & 8.67  \\
            &       & 0.1   & 33.71 & 10.09 & 11.20 & 18.55 & 11.38 & 19.19 \\
            &       & 1.0   & \textbf{90.60} & \textbf{89.33} & \textbf{80.28} & \textbf{73.24} & \textbf{70.98} & \textbf{78.18} \\
            \midrule
            \multirow{15}{*}{\textbf{\textsc{FedAdaVR}}} 
            & \multirow{3}{*}{Adabelief} 
            & 0.001 & 93.10 & 90.91 & 84.93 & 69.18 & 77.92 & 83.23 \\
            &       & 0.005 & 96.53 & \textbf{96.27} & 90.65 & 80.15 & \textbf{90.30} & 92.31 \\
            &       & 0.01  & 96.47 & 94.31 & 88.22 & 29.60 & 86.85 & 92.16 \\
            \cmidrule(r){2-9}
            & \multirow{3}{*}{Adam} 
            & 0.001 & 89.94 & 88.96 & 79.93 & 59.16 & 74.58 & 82.44 \\
            &       & 0.005 & 95.82 & 95.52 & \textbf{94.71} & 75.65 & 84.95 & 85.93 \\
            &       & 0.01  & \textbf{97.05} & 95.76 & 94.09 & 50.32 & 78.89 & \textbf{94.38} \\
            \cmidrule(r){2-9}
            & \multirow{3}{*}{Adagrad} 
            & 0.001 & 80.72 & 73.43 & 58.21 & 33.54 & 55.08 & 46.12 \\
            &       & 0.005 & 91.59 & 90.68 & 89.27 & 53.17 & 78.55 & 84.84 \\
            &       & 0.01  & 92.78 & 93.17 & 92.54 & 61.94 & 83.39 & 87.94 \\
            \cmidrule(r){2-9}
            & \multirow{3}{*}{Lamb} 
            & 0.001 & 38.89 & 16.90 & 14.17 & 19.57 & 19.27 & 12.74 \\
            &       & 0.005 & 77.12 & 65.15 & 62.86 & 57.07 & 50.16 & 64.27 \\
            &       & 0.01  & 90.10 & 85.64 & 82.00 & 81.88 & 77.99 & 75.29 \\
            \cmidrule(r){2-9}
            & \multirow{3}{*}{Yogi} 
            & 0.001 & 89.73 & 89.10 & 78.24 & 49.33 & 71.81 & 83.51 \\
            &       & 0.005 & 95.36 & 95.47 & 92.92 & \textbf{82.80} & 88.21 & 92.98 \\
            &       & 0.01  & 96.65 & 94.92 & 93.64 & 60.52 & 86.69 & 93.48 \\
            \midrule
            \multirow{15}{*}{\makecell[l]{\textbf{\textsc{FedAdaVR-}} \\ \textbf{\textsc{Quant}}}} 
            & \multirow{3}{*}{Adabelief} 
            & 0.001 & 93.65 & 91.21 & 85.57 & 70.89 & 67.06 & 84.32 \\
            &       & 0.005 & 96.28 & \textbf{96.25} & 92.16 & 68.55 & 87.63 & 92.68 \\
            &       & 0.01  & 96.65 & 95.65 & 93.69 & 52.37 & \textbf{90.60} & 92.99 \\
            \cmidrule(r){2-9}
            & \multirow{3}{*}{Adam} 
            & 0.001 & 89.49 & 88.65 & 80.97 & 71.77 & 68.05 & 82.38 \\
            &       & 0.005 & 95.79 & 96.12 & \textbf{93.77} & 67.32 & 81.98 & 93.72 \\
            &       & 0.01  & \textbf{96.90} & 95.74 & 85.74 & 23.23 & 89.08 & \textbf{94.13} \\
            \cmidrule(r){2-9}
            & \multirow{3}{*}{Adagrad} 
            & 0.001 & 80.63 & 65.57 & 49.60 & 37.27 & 39.02 & 48.22 \\
            &       & 0.005 & 91.78 & 91.74 & 83.91 & 52.49 & 75.97 & 86.97 \\
            &       & 0.01  & 94.54 & 93.18 & 90.08 & 69.31 & 81.98 & 78.66 \\
            \cmidrule(r){2-9}
            & \multirow{3}{*}{Lamb} 
            & 0.001 & 21.75 & 21.45 & 17.39 & 19.58 & 16.89 & 18.87 \\
            &       & 0.005 & 78.68 & 72.76 & 65.16 & 49.36 & 58.69 & 54.94 \\
            &       & 0.01  & 88.68 & 85.35 & 85.29 & \textbf{79.12} & 77.51 & 82.17 \\
            \cmidrule(r){2-9}
            & \multirow{3}{*}{Yogi} 
            & 0.001 & 91.65 & 89.87 & 74.87 & 45.56 & 74.13 & 84.27 \\
            &       & 0.005 & 96.28 & 95.49 & 91.64 & 73.52 & 83.34 & 91.72 \\
            &       & 0.01  & 96.02 & 83.08 & 90.47 & 52.47 & 84.42 & 92.48 \\
            \bottomrule
        \end{tabular}
    \end{table}

    \begin{table}[ht]
        \caption{Accuracy comparison of FedVARP, \textsc{FedAdaVR} and \textsc{FedAdaVR-Quant} algorithms on the FMNIST dataset (Server learning rate: \(\eta_s\)).}
        \label{tab:FashiontMNIST_accuracy_com}
        \centering
        \small 
        \setlength{\tabcolsep}{4pt} 
        \begin{tabular}{ll l cccccc}
            \toprule
            & & & \multicolumn{6}{c}{\textbf{Partition}} \\
            \cmidrule(lr){4-9}
            \textbf{Algorithm} 
            & \textbf{Optimiser} 
            & $\eta_s$ 
            & \textbf{IID} 
            & \textbf{Non\_IID} 
            & \textbf{Dir.} 
            & \textbf{LQ-1} 
            & \textbf{LQ-2} 
            & \textbf{LQ-3} \\
            \midrule
            \textbf{FedAvg} & - & - & 81.35 & 77.43 & 77.08 & 63.53 & 59.70 & 71.92 \\
            \midrule
            \textbf{FedProx} & - & - & 81.33 & 76.92 & 75.8 & 48.12 & 59.12 & 69.66 \\
            \midrule
            \textbf{SCAFFOLD} & - & - & 67.18 & 55.81 & 62.0 & 48.49 & 41.57 & 52.24 \\
            \midrule
            \textbf{FedNova} & - & - & 80.10 & 76.46 & 76.23 & 66.23 & 63.57 & 71.15 \\
            \midrule
            \textbf{FedAdam} & - & - & 10.00 & 10.00 & 10.00 & 10.00 & 10.00 & 10.00 \\
            \midrule
            \textbf{FedAdagrad} & - & - & 80.99 & 77.40 & 77.83 & 62.63 & 71.35 & 73.01 \\
            \midrule
            \textbf{FedYogi} & - & - & 83.68 & 81.89 & 81.50 & 62.14 & 71.91 & 74.56 \\
            \midrule
            \textbf{MIFA} & - & - & 67.84 & 63.74 & 67.42 & 35.81 & 55.36 & 61.51 \\
            \midrule
            \multirow{3}{*}{\textbf{FedVARP}} 
            & \multirow{3}{*}{-} 
            & 0.01  & 10.00 & 10.37 & 10.00 & 14.46 & 30.65 & 16.86 \\
            &       & 0.1   & 60.21 & 60.50 & 63.44 & 44.84 & 56.68 & 58.47 \\
            &       & 1.0   & \textbf{79.86} & \textbf{77.53} & \textbf{77.15} & \textbf{59.83} & \textbf{66.60} & \textbf{72.40} \\
            \midrule
            \multirow{15}{*}{\textbf{\textsc{FedAdaVR}}} 
            & \multirow{3}{*}{Adabelief} 
            & 0.001 & 77.20 & 75.72 & 75.47 & 60.53 & 70.33 & 71.19 \\
            &       & 0.005 & 83.06 & 80.80 & 82.50 & 69.56 & 73.31 & 77.43 \\
            &       & 0.01  & 84.08 & \textbf{84.06} & \textbf{83.20} & \textbf{71.97} & 74.12 & 77.96 \\
            \cmidrule(r){2-9}
            & \multirow{3}{*}{Adam} 
            & 0.001 & 75.10 & 74.04 & 83.09 & 63.02 & 68.74 & 66.98 \\
            &       & 0.005 & 82.96 & 81.79 & 79.05 & 70.45 & 71.88 & 77.60 \\
            &       & 0.01  & 84.13 & 83.31 & 73.93 & 69.60 & 73.89 & 78.18 \\
            \cmidrule(r){2-9}
            & \multirow{3}{*}{Adagrad} 
            & 0.001 & 61.77 & 55.18 & 51.17 & 32.83 & 48.90 & 56.65 \\
            &       & 0.005 & 75.53 & 74.30 & 73.28 & 52.83 & 69.27 & 71.13 \\
            &       & 0.01  & 80.63 & 77.40 & 77.08 & 66.62 & 72.21 & 73.21 \\
            \cmidrule(r){2-9}
            & \multirow{3}{*}{Lamb} 
            & 0.001 & 29.95 & 18.78 & 20.71 & 25.69 & 10.83 & 31.06 \\
            &       & 0.005 & 64.06 & 62.50 & 68.60 & 59.87 & 57.88 & 60.63 \\
            &       & 0.01  & 74.48 & 74.90 & 76.42 & 66.52 & 70.55 & 74.94 \\
            \cmidrule(r){2-9}
            & \multirow{3}{*}{Yogi} 
            & 0.001 & 73.73 & 72.90 & 73.57 & 58.43 & 69.22 & 71.97 \\
            &       & 0.005 & 82.27 & 80.92 & 81.84 & 68.52 & 75.17 & 77.55 \\
            &       & 0.01  & \textbf{84.31} & 83.48 & 82.66 & 70.35 & \textbf{75.262} & \textbf{78.97} \\
            \midrule
            \multirow{15}{*}{\makecell[l]{\textbf{\textsc{FedAdaVR-}} \\ \textbf{\textsc{Quant}}}} 
            & \multirow{3}{*}{Adabelief} 
            & 0.001 & 76.77 & 75.06 & 75.92 & 61.08 & 70.57 & 69.91 \\
            &       & 0.005 & 84.31 & 80.53 & 82.48 & 64.23 & 71.66 & 78.02 \\
            &       & 0.01  & \textbf{84.84} & 82.47 & \textbf{83.25} & 67.75 & \textbf{76.29} & 79.77 \\
            \cmidrule(r){2-9}
            & \multirow{3}{*}{Adam} 
            & 0.001 & 83.62 & 71.55 & 73.35 & 64.84 & 65.53 & 71.40 \\
            &       & 0.005 & 83.62 & 80.22 & 81.66 & 66.37 & 72.87 & 77.43 \\
            &       & 0.01  & 83.28 & \textbf{83.37} & 82.75 & 67.42 & 74.90 & \textbf{80.08} \\
            \cmidrule(r){2-9}
            & \multirow{3}{*}{Adagrad} 
            & 0.001 & 58.68 & 57.48 & 48.56 & 30.27 & 33.81 & 46.89 \\
            &       & 0.005 & 75.56 & 73.36 & 74.48 & 55.51 & 66.89 & 70.55 \\
            &       & 0.01  & 80.32 & 78.17 & 77.64 & 64.22 & 71.44 & 75.40 \\
            \cmidrule(r){2-9}
            & \multirow{3}{*}{Lamb} 
            & 0.001 & 33.00 & 18.69 & 22.77 & 27.41 & 15.89 & 24.67 \\
            &       & 0.005 & 62.42 & 65.23 & 69.66 & 60.41 & 59.47 & 69.68 \\
            &       & 0.01  & 74.99 & 72.67 & 76.82 & 66.03 & 70.27 & 73.93 \\
            \cmidrule(r){2-9}
            & \multirow{3}{*}{Yogi} 
            & 0.001 & 73.88 & 73.90 & 72.35 & 61.02 & 69.32 & 71.14 \\
            &       & 0.005 & 83.06 & 80.67 & 80.39 & \textbf{69.80} & 72.44 & 77.48 \\
            &       & 0.01  & 84.50 & 82.26 & 83.01 & 67.61 & 72.14 & 78.20 \\
            \bottomrule
        \end{tabular}
    \end{table}

    \begin{table}[ht]
        \caption{Accuracy comparison of FedVARP, \textsc{FedAdaVR} and \textsc{FedAdaVR-Quant} algorithms on the CIFAR-10 dataset (Server learning rate: \(\eta_s\)).}
        \label{tab:CIFAR10_accuracy_com}
        \centering
        \setlength{\tabcolsep}{1pt} 
        \begin{tabular}{ll l cccccc}
            \toprule
            & & & \multicolumn{6}{c}{\textbf{Partition}} \\
            \cmidrule(lr){4-9}
            \textbf{Algorithm} 
            & \textbf{Optimiser} 
            & $\eta_s$ 
            & \textbf{IID} 
            & \textbf{Non\_IID} 
            & \textbf{Dirichlet} 
            & \textbf{LQ-1} 
            & \textbf{LQ-2} 
            & \textbf{LQ-3} \\
            \midrule
            \textbf{FedAvg} & - & - & 68.4 & 65.6 & $35.4 \pm 25.4$ & $30.6 \pm 0.4$ & $38.6 \pm 2.6$ & $48.9 \pm 1.9$ \\
            \midrule
            \textbf{FedProx} & - & - & 66.9 & 63.4 & $60.5 \pm 0.6$ & $29.7 \pm 1.4$ & $43.0 \pm 2.5$ & $54.2 \pm 2.7$ \\
            \midrule
            \textbf{SCAFFOLD} & - & - & 62.8 & 60.9 & 48.89 & 13.09 & 32.35 & 39.64 \\
            \midrule
            \textbf{FedNova} & - & - & 68.2 & 65.1 & $59.2 \pm 0.0$ & $29.7 \pm 0.1$ & $38.0 \pm 3.1$ & $51.1 \pm 1.1$ \\
            \midrule
            \textbf{FedAdam} & - & - & 41.6 & 36.4 & $32.6 \pm 0.6$ & $10.3 \pm 0.0$ & $23.4 \pm 4.3$ & $29.8 \pm 3.0$ \\
            \midrule
            \textbf{FedAdagrad} & - & - & 39.1 & 41.8 & $43.1 \pm 0.7$ & $32.6 \pm 0.5$ & $39.4 \pm 2.4$ & $42.7 \pm 2.1$ \\
            \midrule
            \textbf{FedYogi} & - & - & 71.7 & 70.2 & $63.1 \pm 1.2$ & $30.5 \pm 1.2$ & $47.0 \pm 1.5$ & $56.4 \pm 3.4$ \\
            \midrule
            \textbf{MIFA} & - & - & 36.6 & 34.5 & 35.16 & 27.79 & 27.20 & 28.95 \\
            \midrule
            \multirow{3}{*}{\textbf{FedVARP}} 
            & \multirow{3}{*}{-} 
            & 0.01  & 33.3 & 32.3 & $26.5 \pm 1.2$ & $34.9 \pm 1.4$ & $27.6 \pm 1.6$ & $42.4 \pm 9.0$ \\
            &       & 0.1   & 51.9 & 51.6 & $48.1 \pm 0.9$ & $\boldsymbol{38.4 \pm 0.5}$ & $\boldsymbol{41.8 \pm 0.9}$ & $49.8 \pm 0.4$ \\
            &       & 1.0   & \textbf{69.0} & \textbf{65.4} & $\boldsymbol{61.0 \pm 0.3}$ & $32.8 \pm 1.3$ & $39.7 \pm 2.4$ & $\boldsymbol{53.3 \pm 0.5}$ \\
            \midrule
            \multirow{15}{*}{\textbf{\textsc{FedAdaVR}}} 
            & \multirow{3}{*}{Adabelief} 
            & 0.001 & 70.1 & \textbf{70.3} & $67.0 \pm 2.0$ & $41.0 \pm 1.8$ & $58.1 \pm 4.1$ & $66.0 \pm 2.5$ \\
            &       & 0.005 & 66.9 & 67.4 & $62.1 \pm 1.8$ & $33.8 \pm 0.6$ & $49.8 \pm 3.3$ & $59.1 \pm 5.2$ \\
            &       & 0.01  & 61.7 & 55.7 & $60.2 \pm 1.9$ & $21.6 \pm 8.9$ & $45.4 \pm 1.7$ & $57.1 \pm 3.0$ \\
            \cmidrule(r){2-9}
            & \multirow{3}{*}{Adam} 
            & 0.001 & 67.7 & 69.6 & $67.2 \pm 1.5$ & $40.0 \pm 2.0$ & $57.2 \pm 2.9$ & $66.4 \pm 1.4$ \\
            &       & 0.005 & 70.9 & 65.6 & $63.3 \pm 2.2$ & $34.4 \pm 2.3$ & $51.0 \pm 1.7$ & $59.9 \pm 3.7$ \\
            &       & 0.01  & 61.2 & 61.6 & $58.6 \pm 1.7$ & $35.3 \pm 0.3$ & $48.2 \pm 0.8$ & $56.9 \pm 2.2$ \\
            \cmidrule(r){2-9}
            & \multirow{3}{*}{Adagrad} 
            & 0.001 & 59.8 & 58.6 & $54.4 \pm 0.2$ & $46.1 \pm 3.4$ & $51.1 \pm 0.5$ & $59.8 \pm 3.1$ \\
            &       & 0.005 & 68.2 & 67.4 & $65.3 \pm 0.1$ & $40.7 \pm 4.4$ & $52.8 \pm 3.5$ & $59.8 \pm 3.1$ \\
            &       & 0.01  & 54.2 & 60.2 & $58.5 \pm 0.9$ & $39.8 \pm 4.1$ & $48.2 \pm 3.2$ & $55.0 \pm 2.0$ \\
            \cmidrule(r){2-9}
            & \multirow{3}{*}{Lamb} 
            & 0.001 & 54.7 & 58.5 & $55.5 \pm 0.3$ & $\boldsymbol{48.7 \pm 2.0}$ & $52.8 \pm 1.7$ & $57.2 \pm 2.0$ \\
            &       & 0.005 & 67.4 & 68.0 & $65.3 \pm 3.3$ & $41.0 \pm 1.2$ & $57.4 \pm 1.5$ & $66.9 \pm 1.1$ \\
            &       & 0.01  & 71.2 & 68.2 & $\boldsymbol{69.0 \pm 1.2}$ & $33.2 \pm 1.0$ & $55.5 \pm 2.4$ & $64.5 \pm 1.9$ \\
            \cmidrule(r){2-9}
            & \multirow{3}{*}{Yogi} 
            & 0.001 & 70.1 & 68.1 & $66.2 \pm 2.1$ & $41.3 \pm 4.1$ & $\boldsymbol{60.8 \pm 1.2}$ & $\boldsymbol{67.2 \pm 1.0}$ \\
            &       & 0.005 & 65.0 & 68.0 & $62.3 \pm 4.3$ & $33.9 \pm 3.0$ & $50.9 \pm 4.4$ & $61.1 \pm 3.0$ \\
            &       & 0.01  & \textbf{72.0} & 59.2 & $65.2 \pm 2.0$ & $40.0 \pm 3.2$ & $52.9 \pm 3.2$ & $61.5 \pm 3.7$ \\
            \midrule
            \multirow{15}{*}{\makecell[l]{\textbf{\textsc{FedAdaVR-}} \\ \textbf{\textsc{Quant}}}}  
            & \multirow{3}{*}{Adabelief} 
            & 0.001 & 69.8 & 68.0 & $68.1 \pm 1.2$ & $42.9 \pm 1.2$ & $58.7 \pm 1.4$ & $67.0 \pm 1.1$ \\
            &       & 0.005 & 63.3 & 64.4 & $61.4 \pm 2.1$ & $33.4 \pm 0.9$ & $48.4 \pm 2.6$ & $60.1 \pm 4.3$ \\
            &       & 0.01  & 61.7 & 59.6 & $58.9 \pm 0.8$ & $29.6 \pm 0.2$ & $45.3 \pm 3.0$ & $59.8 \pm 0.6$ \\
            \cmidrule(r){2-9}
            & \multirow{3}{*}{Adam} 
            & 0.001 & 68.0 & \textbf{69.6} & $66.7 \pm 1.8$ & $41.6 \pm 0.6$ & $59.8 \pm 0.4$ & $\boldsymbol{68.0 \pm 0.4}$ \\
            &       & 0.005 & 64.2 & 65.3 & $62.9 \pm 2.8$ & $34.2 \pm 1.4$ & $50.9 \pm 0.3$ & $63.6 \pm 2.1$ \\
            &       & 0.01  & 63.2 & 64.1 & $57.2 \pm 1.2$ & $34.4 \pm 0.7$ & $48.6 \pm 0.5$ & $55.2 \pm 3.6$ \\
            \cmidrule(r){2-9}
            & \multirow{3}{*}{Adagrad} 
            & 0.001 & 59.4 & 59.3 & $54.7 \pm 0.6$ & $46.4 \pm 1.0$ & $51.0 \pm 1.6$ & $55.4 \pm 1.8$ \\
            &       & 0.005 & 68.2 & 68.3 & $63.9 \pm 1.0$ & $44.8 \pm 2.1$ & $55.6 \pm 0.8$ & $63.8 \pm 0.9$ \\
            &       & 0.01  & 55.5 & 61.4 & $59.5 \pm 0.2$ & $42.1 \pm 2.4$ & $49.7 \pm 0.9$ & $56.0 \pm 0.6$ \\
            \cmidrule(r){2-9}
            & \multirow{3}{*}{Lamb} 
            & 0.001 & 55.0 & 57.8 & $55.4 \pm 0.5$ & $\boldsymbol{46.9 \pm 0.1}$ & $53.7 \pm 0.2$ & $58.7 \pm 0.2$ \\
            &       & 0.005 & 69.6 & 67.5 & $66.8 \pm 0.4$ & $39.8 \pm 2.0$ & $\boldsymbol{59.9 \pm 0.1}$ & $67.9 \pm 0.0$ \\
            &       & 0.01  & \textbf{73.4} & 69.2 & $\boldsymbol{69.4 \pm 0.8}$ & $35.3 \pm 0.9$ & $58.4 \pm 0.5$ & $65.1 \pm 1.4$ \\
            \cmidrule(r){2-9}
            & \multirow{3}{*}{Yogi} 
            & 0.001 & 70.3 & 67.6 & $66.9 \pm 1.1$ & $44.0 \pm 1.1$ & $59.7 \pm 1.1$ & $66.8 \pm 1.1$ \\
            &       & 0.005 & 65.7 & 63.4 & $64.3 \pm 2.5$ & $37.9 \pm 0.8$ & $56.2 \pm 0.2$ & $63.5 \pm 0.1$ \\
            &       & 0.01  & 60.0 & 63.3 & $65.2 \pm 2.0$ & $39.2 \pm 7.6$ & $53.5 \pm 5.1$ & $61.4 \pm 2.6$ \\
            \bottomrule
        \end{tabular}
        
    \end{table}

    \section{Extended Limitations and Future Directions}\label{Apd:Limitations}

	Below we expand on the theoretical and system-level limitations summarised in Section~\ref{conclusion} and outline directions for future work:

	\subsection*{Theoretical scope and assumptions}
    
    Our convergence analysis (Theorem~\ref{theo:con}) focuses on Adagrad as the server optimiser and relies on uniformly bounded stochastic gradients (Assumption~\ref{ass:bg}), which the proof invokes only at the visited global and local iterates. Empirically, our decoupled variance reduction setup works equally well with Adam, AdaBelief, Yogi, and Lamb, though extending the formal proof to these optimisers is non-trivial since one must jointly bound heavy-ball momentum and stale historical states. Relaxing Assumption~\ref{ass:bg} using recent affine-variance techniques for Adagrad is another natural next step.
	
    \subsection*{Convergence theory for \textsc{FedAdaVR-Quant}}
    
    We have not derived a formal convergence bound for \textsc{FedAdaVR-Quant}. If one uses unbiased stochastic quantisation, the compression noise has zero mean and can be folded directly into the variance bound $\sigma^2$ (Assumption~\ref{ass:boundvar}) without changing the asymptotic rate, much like recent analyses of low-precision adaptive optimisers \citep{Tang2026}. However, our experiments use deterministic linear quantisation, where rounding errors accumulate and correlate across historical client states $\mathbf{y}_j^{(t)}$ over time; tightly bounding this coupled error remains open, even though empirically these rounding perturbations do not induce material bias and frequently act as a mild regulariser under extreme label skew.
	
    \subsection*{Server memory scaling at extreme model and population sizes}
    
    \textsc{FedAdaVR-Quant} only compresses states on the server without modifying client-to-server communication. Its $\mathcal{O}(Nd_{\text{quant}})$ footprint is manageable for medium-scale cross-device settings (e.g., 1,360~GB for ConvNeXt-Tiny with $N=10^5$ clients in Int4; Table~\ref{tab:memory_req_per_training}). At foundation-model scales ($d \ge 10^8$ parameters and $N = 10^6$ clients), however, even an 87.5\% savings from Int4 still leaves roughly $25$~TB of server state. Cluster-based alternatives like ClusterFedVARP \citep{Jhunjhunwala2022} cut this cost to $\mathcal{O}(\text{number of clusters x d})$, but they introduce their own failure modes when clients are misgrouped, clusters become imbalanced, or client representations drift during training. Scaling stateful variance reduction to billion-parameter models will likely require combining low-bit quantisation with partial client caching or low-rank state approximations.
	
    \subsection*{Temporal distribution shifts}
    
    SAGA-style variance reduction assumes that a client's stored update $\mathbf{y}_j^{(t)}$ remains a reasonable proxy until that client is sampled again. If a client's local data distribution shifts rapidly over time (concept drift), those cached updates can become outdated and pull the aggregate estimate $\mathbf{r}^{(t)}$ in stale directions. Adding staleness-based decay or an exponential forgetting factor to downweight older $\mathbf{y}_j^{(t)}$ entries would help adapt the method to non-stationary settings.


\end{document}